\documentclass{article}

\usepackage[final]{neurips_2025}

\usepackage[utf8]{inputenc} 
\usepackage[T1]{fontenc}    
\usepackage{hyperref}       
\usepackage{url}            
\usepackage{booktabs}       
\usepackage{amsfonts}       
\usepackage{nicefrac}       
\usepackage{microtype}      
\usepackage{lipsum}
\usepackage{fancyhdr}       
\usepackage{graphicx}       
\graphicspath{{media/}}     
\usepackage{natbib}

\usepackage{amssymb,amsmath}
\usepackage{amsthm}
\usepackage{mathtools}
\usepackage{bbm}
\usepackage{natbib}
\usepackage{enumitem}
\usepackage{xcolor}

\usepackage{algorithm}
\usepackage{algpseudocode}
\usepackage{adjustbox}

\usepackage[utf8]{inputenc} 
\usepackage[T1]{fontenc}    
\usepackage{hyperref}       
\usepackage{url}            
\usepackage{booktabs}       
\usepackage{amsfonts}  
\usepackage{amsmath}
\usepackage{nicefrac}       
\usepackage{microtype}      
\usepackage{xcolor}         
\usepackage{subcaption}
\usepackage{graphicx}
\usepackage{float}
\newcommand{\expit}{\operatorname{expit}}

\newtheorem{theorem}{Theorem}
\newtheorem{definition}{Definition}
\newtheorem{corollary}[theorem]{Corollary}
\newtheorem{proposition}[theorem]{Proposition}
\newtheorem{lemma}{Lemma}
\newtheorem{assumption}{Assumption}

\newtheorem{remark}{Remark}

\usepackage{xcolor}

\DeclareMathOperator{\argmin}{argmin}

\title{Concentration and excess risk bounds for imbalanced
classification with synthetic oversampling}

\author{Touqeer Ahmad\textsuperscript{1}, 
        Mohammadreza M. Kalan\textsuperscript{1}, 
        François Portier\textsuperscript{1}, 
        Gilles Stupfler\textsuperscript{2} \\
        \\
\textsuperscript{1}Univ Rennes, Ensai, CNRS, CREST—UMR 9194, F-35000 Rennes, France \\
\textsuperscript{2}Univ Angers, CNRS, LAREMA, SFR MATHSTIC, F-49000 Angers, France
}

\begin{document}
\maketitle
\begin{abstract} 
Synthetic oversampling of minority examples using \textsc{Smote} and its variants is a leading strategy for addressing imbalanced classification problems. 
Despite the success of this approach in practice, its theoretical foundations remain underexplored. 
We develop a theoretical framework to analyze the behavior of \textsc{Smote} and related methods when classifiers are trained on synthetic data. 
We first derive a uniform concentration bound on the discrepancy between the empirical risk over synthetic minority samples and the population risk on the true minority distribution. 
We then provide a nonparametric excess risk guarantee for kernel-based classifiers trained using such synthetic data.
These results lead to practical guidelines for better parameter tuning of both \textsc{Smote} and the downstream learning algorithm. 
Numerical experiments are provided to illustrate and support the theoretical findings.
\end{abstract}


\section{Introduction}
The problem of imbalanced classification arises when one of the classes of interest, often referred to as the \emph{minority class}, is significantly underrepresented relative to the other classes. In such settings, standard classification algorithms tend to produce trivial decision rules that favor the majority class, often ignoring the minority class altogether. 
Designing algorithms that perform well under class imbalance has become a key challenge in modern machine learning \citep{chawla2004special,5128907,lemaavztre2017imbalanced,spelmen2018review,feng2021imbalanced}.

There are two main families of methods to tackle this problem. 
{\it Model-level approaches} modify the training objective, often by reweighting the loss function to penalize misclassification of the minority class more {heavily. 
Several} theoretical developments support this direction, including work on cost-sensitive learning \citep{menon2013statistical,xu2020class,aghbalou2024sharp}, Neyman-Pearson classification \citep{JMLR:v12:rigollet11a,tong2013plug,kalan2024distribution} and applications to deep learning models \citep{cao2019learning,pmlr-v97-byrd19a}.

By contrast, {\it data-level approaches} modify the training data 
by applying resampling techniques \citep{kubat1997addressing,chawla2002smote,mani2003knn,Barandela,lemaavztre2017imbalanced}, either by reducing the size of the majority class (undersampling) or by augmenting the minority class with synthetic data (oversampling). 
Their key advantage is compatibility with off-the-shelf machine learning algorithms and standard validation strategies, such as cross-validation. While undersampling can reduce computational complexity in training well-specified parametric models \citep{pmlr-v119-wang20a,chen2025optimal}, it risks discarding valuable information. 
Oversampling, 
however, seeks to mitigate imbalance by generating new synthetic examples for the minority class.

The seminal oversampling algorithm, known as \textsc{Smote} (Synthetic Minority Oversampling Technique) \citep{chawla2002smote}, generates new samples by interpolating between a minority instance and its nearest minority class neighbors. A different proposal called Kernel Density Estimation-based Oversampling (\textsc{Kdeo}) was further introduced in \citet{gao2012probability,kamalov2020kernel}. The present paper studies, from a theoretical perspective, the use of \textsc{Smote} and \textsc{Kdeo} for imbalanced binary classification problems. Both methods are nonparametric and iteratively generate synthetic data points from the minority class to achieve a more balanced dataset. Similar methods such as \textsc{Gsmote}~\citep{Scornet2024} can be handled using our approach, but significantly different extensions of \textsc{Smote} such as \textsc{Borderline-Smote}~\citep{han2005borderline} and \textsc{Adasyn}~\citep{he2008adasyn} 
cannot be. 

We adopt a transfer learning viewpoint, interpreting oversampling as a means to generate synthetic samples for the estimation of a modified, weighted probability measure that gives more weight to the minority class. 
This perspective transforms the original imbalanced learning task into a balanced one and allows us to unify reweighting and resampling approaches under a common objective: approximating the same reweighted distribution. By contrast with previous studies on \textsc{Smote}-based algorithms relying on the ROC curve \citep{chawla2002smote}, AUC \citep{Scornet2024} or F1-score \citep{bej2021loras}, the metric of interest here is associated with the AM-risk to account for the fact that oversampling facilitates a shift toward a more balanced risk. 
Our main contributions are threefold:

\noindent\textbullet \ \textbf{ Uniform concentration inequalities for \textsc{Smote} and 
\textsc{Kdeo}.}  
We derive uniform concentration bounds for the discrepancy between the empirical risk on synthetic data and the population risk on the minority class, over a class of uniformly bounded functions. For \textsc{Smote}, we obtain a bound of order \( n_1^{-1/2} + (k/n_1)^{1/d} \), where \( n_1 \) is the number of minority class samples and \( k \) is the number of neighbors used in \textsc{Smote}. For 
\textsc{Kdeo}, a similar bound of the form \( n_1^{-1/2} + h \) is established, where \( h \) is the kernel bandwidth. 
From these 
inequalities, we then derive, for each oversampling method, excess risk bounds for 
classifiers trained on synthetically generated data, implying consistency results for the balanced risk, provided the oversampling parameters ($k$ and $h$) are sufficiently small and the hypothesis class has controlled Rademacher complexity. This supports the commonly recommended choice of \( k = 5 \) when using \textsc{Smote} \citep{chawla2002smote}. For 
\textsc{Kdeo}, this choice of $k = 5$ corresponds to choosing $h$ of order $n_1^{-1/d}$.


\noindent\textbullet \ \textbf{ Excess risk bounds for nonparametric classifiers with 
\textsc{Kdeo}.}  
We analyze the performance of a kernel smoothing plug-in classifier \citep{tsyb,devroye2013probabilistic} with 
\textsc{Kdeo}. We establish an upper bound on the excess risk of this classifier, revealing a variance term of order \( (n_1 h^d)^{-1/2} \) and a bias term of order \( h \) or \( h^2 \), depending on the regularity of the class-specific distributions. 
This highlights a trade-off involved in choosing the kernel bandwidth \( h \), whose optimal choice 
is different from the \textsc{Smote}-default analogous choice $n_1^{-1/d}$ and depends on the regularity of the distribution. This analysis provides novel theoretical guidance for bandwidth selection in 
\textsc{Kdeo}, an issue that has remained largely heuristic in practice. 


\noindent\textbullet \ \textbf{ Empirical validation.}  
Numerical experiments 
illustrate and support the 
theory. 
They demonstrate that following the well-known Scott's rule $h\propto n_1 ^ {-1/(d+4) } $ might be highly beneficial when using nonparametric classification rules such as \textit{kernel smoothing} or \textit{nearest neighbors}. This improvement is less significant when using a parametric classification rule such as \textit{logistic regression}.



\noindent\textbf{Related Work.} While the empirical effectiveness of \textsc{Smote} is widely acknowledged, its theoretical justification remains limited. 
The density, mean, and variance of the sampling distribution have recently been derived~\citep{elreedy2019comprehensive,elreedy2024theoretical,Scornet2024}. 
{As far as we are aware, our work provides the first 
concentration} inequalities for \textsc{Smote}, thereby offering a principled explanation for its success. 
We rely on advanced tools from 
nearest neighbors theory, notably Vapnik-style inequalities \citep{xue2018achieving,jiang2019non,portier2025nearest}.

Regarding 
\textsc{Kdeo}, 
we go beyond earlier studies by deriving not only concentration results but also excess risk bounds for the kernel smoothing plug-in algorithm of \cite{tsyb,devroye2013probabilistic}. Our 
theory builds upon concentration bounds for the $L^1$-error \citep{Devroye1991} and bias-variance decompositions from 
\cite{devroye1985,HOLMSTROM1992245}. 
We note that \cite{tong2013plug} deals with the different problem of Neyman-Pearson classification by controlling the $\|\cdot \|_\infty $-error, rather than the $L^1$-error, of KDE. 
We extend these results to the case of Lipschitz densities that may exhibit discontinuities at domain boundaries -- a common scenario in real-world applications. 

\noindent\textbf{Outline.} Section~\ref{problem-setting} introduces the mathematical background 
and a formal characterization of the problem. Section~\ref{sec:main} presents the main theoretical results and Section~\ref{numerical} provides supporting numerical experiments. The Appendix contains all mathematical proofs and additional numerical results.

\section{Problem setting}\label{problem-setting}

\subsection{Background}

Let $\mathbb{P}$ be a probability distribution on $\mathbb{R}^d \times \{0,1\}$, and let $\mathcal{D}_n = \{(X,Y), (X_i, Y_i)_{1\leq i\leq n}\}$ be independent and identically distributed (i.i.d.)~random variables with common distribution $\mathbb{P}$.  
Denote by $n_1 = \sum_{i=1}^n \mathrm{1}_{\{Y_i = 1\}}$ and $n_0 = \sum_{i=1}^n \mathrm{1}_{\{Y_i = 0\}}$
the number of samples from the conditional distributions $\mathbb{P}_1 := \mathbb{P}(\,\cdot\,|\,Y=1)$ and $\mathbb{P}_0 := \mathbb{P}(\,\cdot\,|\,Y=0)$, respectively.  
Under class imbalance, we assume $n_1 \ll n_0$, or equivalently, $p_1:= \mathbb{P}(Y=1) \ll {1-p_1}= \mathbb{P}(Y=0)$. Denote by $\{X_{1i}\}_{1\leq i\leq n_1}$ and $\{X_{0i}\}_{1\leq i \leq n_0}$ the features corresponding respectively to the minority and majority classes.


Let $\ell$ be a 
loss function, e.g.~the $0$-$1$ loss $\ell (\alpha) = \mathrm{1}_{\{\alpha \leq 0\}}$ or a convex surrogate 
such as $\ell (\alpha) = e^{-\alpha}$, used to evaluate a discriminant function $g:\mathbb R^d \to \mathbb R$.
For each pair $(X,Y)$, 
$\ell( g(X) (2Y-1)) $ 
quantifies the quality of the prediction and can be used to define the (imbalanced) risk
\[ 
{R(g) := \mathbb{E}[\ell (g(X) (2 Y - 1) )]} = p_1\, \mathbb{E}_{1}[\ell ( g (X))   ] + {(1-p_1)} \, \mathbb{E}_{0}[\ell  (-g (X)) ].
\]
A key challenge 
here is that trivial classifiers that always predict the majority class ($Y=0$) 
will achieve a low such risk when $p_1$ is small, even though 
they perform 
poorly on the minority class ($Y = 1$). A common approach to mitigate this issue is to consider the $\beta$-reweighted (balanced) risk 
 \[
R_{\beta}(g) := \beta \, \mathbb{E}_{1}[\ell ( g (X))   ] + (1-\beta) \, \mathbb{E}_{0}[\ell  (-g (X)) ],
\]
where a commonly used choice is $\beta = 1/2$, corresponding to the so-called AM-risk as studied in \cite{menon2013statistical,xu2020class,aghbalou2024sharp}. While $R_{\beta}(g)$ can in principle be minimized by reweighting the original samples, a widely used and successful alternative is
\emph{synthetic oversampling}. 
We focus on \textsc{Smote} and \textsc{Kdeo}, that expand the minority class until the empirical class prior of the \emph{augmented} sample matches the evaluation weight~$\beta$.

\subsection{Oversampling techniques}\label{subsec:smote-detail}

\paragraph{\textsc{Smote}.} This method generates synthetic minority samples by interpolating between randomly selected minority examples and their nearest neighbors (NN), thereby filling in sparse regions of the minority class in the feature space. At each iteration $i$, a {minority} class ($Y=1$) sample \( \tilde{X}_{1i} \) is drawn uniformly at random from the minority class points \( \{X_{1i}\}_{1\leq i\leq n_1} \), and then another point $\overline X_{1i}  $ is generated uniformly at random among the \( k \)NN of \( \tilde{X}_{1i} \) in the minority class {deprived of \( \tilde{X}_{1i} \), that is, $\{X_{1j}\}_{1\leq j\leq n_1} \setminus \{ \tilde{X}_{1i} \}$ (if $n_1=1$, we set $\overline X_{1i} = \tilde{X}_{1i}$)}. A synthetic sample is generated as 
\begin{equation}\label{smote-equation}
X_{1i}^{*} = (1-\lambda) \tilde{X}_{1i} + \lambda \overline X_{1i} ,
\end{equation}
where \( \lambda \sim \mathcal{U}[0,1] \) is independently drawn. 
The distribution of synthetic \textsc{Smote} points 
is then the mixture $
(1/ n_1) \sum_{i=1}^{n_1} \mathcal{U}_k(X_{1i})
$, where \( \mathcal{U}_k(X_{1i}) \) denotes the (conditional) uniform distribution over the union of segments \( \bigcup_{j=1}^k (X_{1i}, \hat X_{1j}^{(i)}) \) with $\{\hat{X}_{11}^{(i)},\ldots,\hat{X}_{1k}^{(i)} \}$ being the set of $k$ nearest neighbors of $X_{1i}$ within the minority class deprived of $X_{1i}$.

\paragraph{
\textsc{Kdeo}.}
This method generates synthetic minority samples by perturbing randomly selected minority points using kernel-based noise. There are again two steps. First, a point \( \tilde X_{1i} \) is drawn uniformly at random from the minority class \( \{X_{1i}\}_{1\leq i\leq n_1} \); then, a synthetic sample is generated as
\begin{equation}\label{smoteKDE-equation}
X_{1i}^{*} = \tilde X_{1i} + hW_i,
\end{equation}
where \( W_i \sim K \) is {independently drawn} from a fixed kernel distribution (e.g., standard Gaussian), and \( h > 0 \) is a bandwidth parameter controlling the magnitude of the perturbation. This procedure can be interpreted as, given a kernel function \( K \) and bandwidth \( h > 0 \), drawing synthetic samples \( X_{1i}^* \) from the KDE $\hat{f}_{1h}  (x) = ({1}/{n_1}) \sum_{i=1}^{n_1} K_h(x - X_{1i} )$ (with $K_h(x) = h^{-d} K(x/h)$) in the minority class. The key difference between \textsc{Smote} and 
\textsc{Kdeo} lies in the support and shape of the component distributions: \textsc{Smote} uses uniform distributions over segments joining neighbors, while 
\textsc{Kdeo} employs symmetric kernels that may perturb points in any direction. 

The number \( m \) of 
synthetic replications is chosen so that the 
proportion between the synthetic minority samples and the majority samples is such that ${m}/{(m+n_0)} \approx \beta$. By considering $ m = n_0$, we focus on the minimization of the AM-risk $R_{1/2}$, but 
our results also hold for arbitrary $\beta\in(0,1)$.

\section{Main results}
\label{sec:main}

\subsection{Concentration inequalities and excess risk bound for empirical risk minimization with oversampling}
\label{sec:main:concentration}

We show first that the empirical mean based on 
\textsc{Smote} or \textsc{Kdeo}-generated samples concentrates uniformly around the population mean over a class of uniformly bounded functions. 
For \textsc{Smote}, we 
put a mild condition on the minority class distribution $\mathbb{P}_1$ to ensure that 
it puts enough mass everywhere on its support. Let $B(x,r)$ be the Euclidean ball with center $x$ and radius $r>0$ in $\mathbb{R}^d$.
%
\begin{assumption}\label{assump_regularity}
There exists a constant $C_d>0$ such that for all $x\in \text{supp}(\mathbb{P}_1)$ and all $r>0$, we have $\mathbb{P}_1(B(x,r))\geq \min\{C_d r^d,1\}$.  [In particular, $\text{supp}(\mathbb{P}_1)$ is bounded.] 
\end{assumption}

The above assumption is standard to obtain error bounds for nearest neighbors estimators \citep{Gadat,jiang2019non,portier2025nearest} or other local averaging methods such as local polynomial estimators \citep{tsyb}. The next assumption, not actually specific to \textsc{Smote}, introduces the class of functions over which the uniform convergence bound is derived.

\begin{assumption}
\label{function_class}
Let \(\mathcal{F}\) be a separable class of functions \(G: \mathbb{R}^d \to \mathbb{R}\) which is uniformly bounded, i.e.~there is $B>0$ such that $\sup_{G \in \mathcal{F}} \|G\|_{\infty} \le B$, and each \(G \in \mathcal{F}\) is \(L\)-Lipschitz.
\end{assumption}

One last assumption is needed to control the complexity of the class $\mathcal F$. This assumption relies on the well-known notion of Rademacher complexity \citep{bartlett2002rademacher} recalled below. 

\begin{definition}[Rademacher complexity]\label{rademacher_compl}
Let \(X_1,\dots,X_n\) be 
i.i.d.~from a distribution \(P\) on \(\mathcal{X}\), and 
let \(\sigma_1,\dots,\sigma_n\) be i.i.d.~Rademacher random variables (i.e.~\(\mathbb{P}[\sigma_i=1]=\mathbb{P}[\sigma_i=-1]=1/2\)) that are independent of the $X_i$. The empirical Rademacher complexity of \(\mathcal{F}\) on \(X_{1:n}=(X_1,\dots,X_n)\) is 
\[
\widehat{\mathcal{R}}_{n}(\mathcal{F};X_{1:n})
:=
\mathbb{E}_{\sigma}\Biggl[\sup_{G\in\mathcal{F}}
\frac{1}{n}\sum_{i=1}^n \sigma_i\,G(X_i)\Biggr].
\]
\end{definition}
\begin{assumption}[Rademacher Bound]
\label{assump:rad-free}
There exists a sequence \(\mathcal{R}_n(\mathcal{F})\) such that, for every distribution \(P\) on \(\mathcal{X}\) with supp($P$) $\subset \text{supp}(\mathbb{P}_1)\cup\text{supp}(\mathbb{P}_0)$ and $n\geq 1$, the following common upper bound holds: 
\[
\mathbb{E}_{X_{1:n}\sim P^n} \left[\widehat{\mathcal{R}}_{n}(\mathcal{F};X_{1:n})\right]
\le
\mathcal{R}_n(\mathcal{F}).
\]  
\end{assumption}

Note that the \textsc{Smote} algorithm only applies when $n_1>0$. Define 
\[
\mu_{\mathrm{Smote}}^*(G) := \left( \frac{1}{m} \sum_{i=1}^{m} G(X_{1i}^{*}) \right) \mathrm{1}_{ \{ n_1>0 \} } + 0 \times \mathrm{1}_{ \{ n_1=0 \} }.
\]
While this specification is necessary to make our statement mathematically correct, we note that, in the case \( n_1 = 0 \), the corresponding behavior could have been specified in an arbitrary manner. We further introduce $ \sigma_1^2 (\mathcal{F}) =\sup_{G\in\mathcal{F}} \operatorname{Var} (G(X)|Y=1) $, where $X$ denotes the original data, and, likewise, we let $\hat \sigma_1^2 (\mathcal{F})=\sup_{G\in\mathcal{F}} \mathrm{Var} (G (\tilde X_{11}) \mid \mathcal{D}_n)$ given that $n_1>0$, where $\tilde X_{11}$ is drawn uniformly at random from the minority class $\{X_{1i}\}_{1\leq i\leq n_1}$. We make below the convention $\mathcal{R}_0(\mathcal{F})=1/0=+\infty$ and we set \( k_{\delta}=\max(k, (d+1)\log(2n_1) + \log(8/\delta)  )\) on the event $\{ n_1 > 1 \}$ and $0$ otherwise.
\begin{theorem}\label{theorem-smote-concentration_uniform}
Suppose that Assumptions~\ref{assump_regularity}, \ref{function_class} and \ref{assump:rad-free} hold. Let \( \delta \in (0,1/5) \) and $m\geq 1$. 
{If $n_1>1$, let \( k \in \{1,\ldots, n_1-1\} \)} and let \( \{X_{1i}^{*}\}_{1\leq i\leq m} \) be \( m \) i.i.d.\ samples generated by the \textsc{Smote} algorithm \eqref{smote-equation}. Then, with probability at least \( 1 - 5\delta \), 
\[
\sup_{G\in \mathcal{F}}\left|
\mu_{\mathrm{Smote}}^*(G) 
- \mathbb{E}_{1}[G(X)]
\right| 
\leq 4\mathcal{R}_m(\mathcal{F})+4\mathcal{R}_{n_1}(\mathcal{F})+ L \left( \frac{6}{C_d} \right)^{1/d} \left(\frac{k_\delta }{n_1}\right)^{1/d} + R,
\]
where 
\[
R = \left(  \sqrt{ \frac{\hat{\sigma}_1^2(\mathcal{F})}{m}  } +  \sqrt{ \frac{\sigma_1^2(\mathcal{F})}{n_1}  }\right) \sqrt{2 \log\left(\frac{1}{\delta}\right)} + \frac{8}{3}B\left(\frac{1}{m}  + \frac{1}{n_1}\right) \log\left(\frac{1}{\delta}\right).
\]
\end{theorem}

We apply this to the proof of an excess risk bound for empirical risk minimization algorithms under synthetic \textsc{Smote}-based oversampling. 
Let $\mathcal G$ be a class of real-valued discriminative functions defined on $\mathbb R^d$ and let $\ell: \mathbb R\to [0,\infty)$ be the loss function. Consider the following classifier:
\begin{align}\label{erm}
    \hat g ^*_{\mathcal G} \in \argmin_{g\in \mathcal G}\, \left\{ \mathrm 1_{ \{ n_1>0 \} } \sum_{i=1}^m \ell(g(X_{1i}^*))  +\mathrm 1_{ \{ n_0>0 \} } \sum_{i=1}^{n_0} \ell(-g(X_{0i})) \right\}  .
\end{align}
This is a standard classifier, 
using \textsc{Smote}-generated data. By convention, and similarly to what was done with $\mu_{\mathrm{Smote}}^*(G)$ above, the first term (resp.~second term) in the above minimization is set as $0$ when $n_1=0$ (resp.~$n_0=0$). Let $\ell(\mathcal G)=\{ \ell \circ g, \, g\in \mathcal{G} \}$ and \( \sigma_0^2(\mathcal{F}) = \sup_{G\in \mathcal{F}}\mathrm{Var}[G (X)|Y=0] \).

\begin{corollary}\label{cor:smote}
Let $m=n_0\mathrm 1{\{n_1>0\}}$ and $\delta \in (0,1/7)$. 
Under the assumptions of Theorem \ref{theorem-smote-concentration_uniform} with $\mathcal F=  \ell(\mathcal G) \cup \ell( - \mathcal G)$, we have, with probability at least $1-7\delta$, 
\[
R_{1/2} (\hat g ^*_{\mathcal G}) - \inf_{g\in \mathcal G} R_{1/2} (g) \leq   4\mathcal R_{n_1}(\mathcal{F})+8 \mathcal R_{n_0}(\mathcal{F}) + L \left( \frac{6}{C_d} \right)^{1/d} \left(\frac{k_\delta }{n_1}\right)^{1/d} + R, 
\]
where 
\[
R= \left(  \sqrt{ \frac{\hat{\sigma}_1^2(\mathcal{F})}{{n_0}}  } +  \sqrt{ \frac{\sigma_1^2(\mathcal{F})}{n_1}  }+\sqrt{ \frac{\sigma_0^2(\mathcal{F})}{n_0}  }\right) \sqrt{2 \log\left(\frac{1}{\delta}\right)} + \frac{8}{3}B\left( \frac{2}{n_0} +  \frac{1}{n_1} \right) \log\left(\frac{1}{\delta}\right).
\]
\end{corollary}

We turn to 
\textsc{Kdeo}. For $q>0$ and any measurable $f:\mathbb{R}^d\to [0,\infty)$, let $M_q(f): = \int \|z\|_2 ^{q} f(z) dz$. Define $\mu_{\mathrm{KDEO}}^*(G)$ as $\mu_{\mathrm{Smote}}^*(G)$, only with synthetic data generated by \textsc{Kdeo} rather than \textsc{Smote}.

\begin{theorem}\label{th:smote_expectation}
Suppose that Assumptions \ref{function_class} and \ref{assump:rad-free} hold. Let \(K\) be a density function with \(M_1(K) < \infty\). 
Let $\delta\in (0,1/5)$ and $m\geq 1$. 
Moreover, whenever $n_1>0$, let \( \{X_{1i}^{*}\}_{1\leq i\leq m} \)  be \( m \) i.i.d.~samples generated by the \textsc{Kdeo} algorithm \eqref{smoteKDE-equation}. Then, with probability at least $1-5\delta$, 
\[
\sup_{G\in \mathcal{F}}\left| \mu_{\mathrm{KDEO}}^*(G) - \mathbb{E}_1 [ G (X) ] \right| \leq 4\mathcal{R}_{n_1}(\mathcal{F}) + 4\mathcal{R}_m(\mathcal{F}) + 5Lh M_1(K)  +R, 
\]
where 
\[
R= \left(\sqrt{\frac{\sigma_1^2(\mathcal{F}) }{n_1}}+ \sqrt{\frac{\hat{\sigma}_1^2(\mathcal{F}) }{m}}\right)  \sqrt{2\log\left(\frac{1}{\delta}\right)}  
 + \frac{8}{3} B \left( \frac{27}{8m} + \frac{1}{n_1} \right) \log\left(\frac{1}{\delta}\right).
\]
\end{theorem}

%
%
%

For the classifier defined in \eqref{erm} with \textsc{Kdeo} and not \textsc{Smote}, we obtain the excess risk bound below.

\begin{corollary}\label{cor:kde}
Let $m=n_0\mathrm 1{\{n_1>0\}}$ and $\delta \in (0,1/7)$. 
Under the assumptions of Theorem \ref{th:smote_expectation} with $\mathcal F=  \ell(\mathcal G) \cup \ell( - \mathcal G)$, we have, with probability at least $1-7\delta$, 
\[
R_{1/2} (\hat g ^*_{\mathcal G}) - \inf_{g\in \mathcal G} R_{1/2} (g) \leq   4\mathcal{R}_{n_1}(\mathcal{F})+ 8\mathcal{R}_{n_0}(\mathcal{F}) +5LhM_1(K) +  R,
\]
where 
\[
R = \left(  \sqrt{ \frac{\hat{\sigma}_1^2(\mathcal{F})}{n_0}  } +  \sqrt{ \frac{\sigma_1^2(\mathcal{F})}{n_1}  }+\sqrt{ \frac{\sigma_0^2(\mathcal{F})}{n_0}  }\right) \sqrt{2 \log\left(\frac{1}{\delta}\right)} + \frac{8}{3} B \left( \frac{35}{8n_0} + \frac{1}{n_1} \right) \log\left(\frac{1}{\delta}\right).
\]
\end{corollary}

\begin{remark} The Lipschitz property required in Corollaries \ref{cor:smote} and \ref{cor:kde} as part of Assumption \ref{function_class} holds for standard loss functions such as hinge and logistic, and for classifiers like neural networks under mild conditions such as bounded spectral norms of the weight matrices.
\end{remark}

\begin{remark}
Many practical function classes 
satisfy the 
Rademacher complexity bound in Assumption~\ref{assump:rad-free} under natural constraints, with
$\mathcal{R}_n(\mathcal{F}) \le {B_{\mathcal{F}}} / {\sqrt{n}}$, where $B_{\mathcal{F}}>0$ is a constant. For example, if \(\mathcal{F}\) has finite VC–subgraph dimension \(v\), then the previous bound is valid with $B_{\mathcal F}$ depending on $v$, $B$ and $\sigma_1^2 (\mathcal{F})$~\citep[Proposition 2.1]{gineConsistencyKernelDensity2001}; for linear predictors based on 
\(\mathcal{F}=\{x\mapsto w^\top x:\|w\|_2\le W\}\) with \(\|x\|_2\le R\), one has that 
$ B_{\mathcal F} $ depends on $ R$ and $W$ and the ambient dimension~\citep[][Lemma~19]{bartlett2002rademacher}; and for fully‑connected ReLU networks of depth $L$ whose inputs satisfy $\|x\|_{2}\le R$ and whose weight matrices $A_{\ell}$ obey the Frobenius‑norm bounds $\|A_{\ell}\|_F\le s_{\ell}$, 
$B_{\mathcal F}$ scales as $R\sqrt L \prod_{\ell=1}^{L}s_{\ell}$ \citep[][Theorem~1]{golowich2018size}.
\end{remark}

\begin{remark} The bounds of Corollaries \ref{cor:smote} and \ref{cor:kde}  depend on the sample sizes $n_0$ and $n_1$ which, in our framework, are random variables. Considering $n_1$, a multiplicative Chernoff bound gives that with probability at least $1 - \delta$, $n_1 \geq n p_1 (1 - \epsilon)$, where $\epsilon^2 = 2\log(1/\delta)/ (n p_1 )$. Doing so reveals a frontier which is, in the asymptotic regime, $np_1 \to \infty$. When $p_1$ is below this frontier, it is not clear that consistent estimation is possible as the bound does not vanish asymptotically.
\end{remark}

\subsection{Nonparametric excess risk bound for the \textsc{Kdeo}-based kernel smoothing plug-in classifier 
}
\label{sec:main:excessrisk}


The balanced Bayes classifier minimizing the AM-risk is $g(x) = \mathrm{1}_{ \{ \eta(x) > p_1  \} }$, where $\eta (x) =  \mathbb P (Y = 1 |X=x) $. Besides, if $\mathbb{P}_y$ has a density function $f_y$ with respect to the Lebesgue measure then, by the law of total probability, the distribution of $X$ has density $ f = p_1 f_1 + (1-p_1)f_0$ with respect to this same measure. It follows from the Bayes formula that the balanced Bayes classifier is then given by 
\begin{equation}\label{eq:bayes}    
\forall x\in \mathbb R^d,\quad g(x) = \mathrm{1}_{ \{   f_1(x)  > f(x)  \} } = \mathrm{1}_{ \{   f_1(x)  > f_0(x)  \} } .
\end{equation}
Our ultimate goal is to study the counterpart of $g$ 
using a \textsc{Kdeo}-based kernel smoothing classifier. 
Before that, we give a general result, with respect to the AM-risk, on the performance of any discrimination rule $\hat g$ of the form $\hat g(x) = \mathrm{1}_{ \{  \hat f_1(x)  >  \hat f_0(x)   \} }$, where each $\hat f_y$ is an estimator of $f_y$. 
%
\begin{theorem}
\label{th:risk_bound_1}
For each $y\in \{0,1\}$, suppose that $\mathbb P_y$ has a density $f_y$ with respect to the Lebesgue measure. Let $S_y $ denote the support of $\mathbb P_y$ and $S=S_0\cup S_1$. We have 
\[  
    R_{1/2}(\hat g)   - R_{1/2}(g)  \leq   \frac{1 } { 2} \int_{S} | \hat f_1 (x) - f_1(x) |dx + \frac{1}{2} \int_{S}  | \hat f_0 (x) - f_0(x) | dx.
\]
\end{theorem}

This is reminiscent of 
known results in 
\cite[Chapter 6]{devroye2013probabilistic} or \cite[Chapter 10]{devroye1985}, showing that the excess classification risk is bounded by the $L^1$-error of the conditional probability estimators or the density estimators in the different classes. 
Theorem \ref{th:risk_bound_1}, though, is concerned with the AM-risk while the above references deal with the standard risk measure.

Let us now focus on the kernel smoothing classifier of \cite{tsyb,devroye2013probabilistic} when using KDE-generated data. 
Consider, as in \eqref{smoteKDE-equation}, independent samples $\{ X_{1i}^* \}_{1\leq i\leq m}$ with common density $\hat f_{1h}$ (given the initial sample), i.e.~the KDE of the minority class covariates. 
Define 
\[
\hat \eta^* (x) := \frac{\sum_{i=1}^{m}  K_{s} (x- X_{1i}^{*}  )}{\sum_{i=1} ^{m} K_{s} (x- X_{1i}^{*}  )  + \sum_{i=1} ^{n_0} K_{s} (x- X_{0i}  )  } = \frac{m \hat f_{1s}^{*} (x)}{m \hat f_{1s}^{*} (x)  + n_0 \hat f_{0s} (x) },
\]
with the convention $0/0 = 0$ and where $\hat f^*_{1s} (x) = \mathrm 1_{\{ n_1>0\}} ({1}/ {m}) \sum_{i=1}^m K_{s}  (x - X_{1i}^*)$ is the kernel density estimate based on synthetic data from class $1$ and $\hat f_{0s} $ is the kernel density estimate based on initial data from class $0$ (with $\hat f_{0s} =0$ when $n_0=0$). 
{One could, of course, choose a different kernel for $\hat f_{0s}$, $\hat f^*_{1s}$ and $\hat f_{1h}$; we take the same kernel $K$ in each estimate for the sake of simplicity.}

By setting $m=n_0\mathrm 1_{\{n_1>0\}}$, the discrimination rule $\hat g^*(x) = \mathrm{1}_{ \{ \hat \eta^*(x) >1/2 \} } =   \mathrm{1}_{ \{ \hat f^*_{1s} (x)  > \hat f_{0s} (x)  \} }$ is tailored to minimizing the AM-risk. Next we obtain an upper bound on the excess risk of $\hat g^*$ in this setting. Regularity assumptions are required; let, for any open subset $U$ of $\mathbb{R}^d$, $W^{s,1}(U)$ be the Sobolev space of functions $G: U \to \mathbb R$ whose (weak) partial derivatives of order $s$ are integrable.

\begin{assumption}\label{cond:kernel} 
    $K: \mathbb R ^d \to \mathbb R $ is a square-integrable symmetric density function such that $M_{d+\varepsilon}(K + K^2)< \infty $ for some $\varepsilon >0$ and $M_2(K) <\infty$.
\end{assumption}

\begin{assumption}\label{cond:P_X_density2} 
    For each $y\in \{ 0,1 \}$, $\mathbb P_y$ has a density $f_y\in W^{2,1}(\mathbb{R}^d)$ with respect to the Lebesgue measure and $M_{d+\varepsilon}(f_y) < \infty $ for some $\varepsilon>0$. 
\end{assumption}

To deal with situations violating Assumption~\ref{cond:P_X_density2}, where the densities may, for example, have compact supports and be smooth and bounded away from zero on the interior of their supports, we introduce alternative assumptions. For two sets $S_1$ and $S_2$, $S_1+S_2$ is the set $\{ s_1+s_2, \ (s_1,s_2) \in S_1 \times S_2 \}$. The Lebesgue measure is denoted by $\lambda$ and the topological boundary of a set $E$ is denoted by $\partial E$. 

\begin{assumption}\label{cond:kernel_compact} 
    $K: \mathbb R ^d \to \mathbb R $ is a square-integrable symmetric density function supported on $B(0,1)$.
\end{assumption}

\begin{assumption}\label{cond:P_X_density1_compact} 
    For each $y\in \{ 0,1 \}$, the support $S_y$ of $\mathbb{P}_y$ is smooth, in the sense that there are $\kappa,r_0>0$ such that $ \lambda ( \partial S_y  + B(0,r) ) \leq \kappa r $ for $0<r<r_0$. Moreover, $\mathbb{P}_y$ has a density $f_y\in W^{1,1}(U)$ with respect to the Lebesgue measure, where $U$ is the interior of $S$, which is bounded on $\partial S + B(0,2r_0)$, and $M_{d+\varepsilon}(f_y) < \infty $ for some $\varepsilon>0$. 
\end{assumption}

Assumption \ref{cond:P_X_density1_compact} on $S_y$ is, for example, satisfied if $\partial S_y$ is a finite union of compact, closed and smooth submanifolds of dimension $d-1$ in $\mathbb{R}^d$ whose pairwise distances to one another are nonzero. This is a consequence of Weyl's tube formula in the Euclidean space, see Equation~(14) in~\cite{weyl1939} applied to manifolds of codimension $m=1$ with the notation therein; a self-contained statement of this result is Theorem~9.3.11 in~\cite{nic2007}.

These assumptions make it possible to prove $L^1-$bounds on the KDEs of the $f_y$. 
The proofs 
rest upon the McDiarmid inequality~\citep{mcdiarmid1989method} and a bias-variance decomposition, where it is shown and used that, if $f_y\ast K_h(x)=\int K_h(x-z) f_y(z) dz$ denotes the convolution of $f_y$ and $K_h$, 
\begin{align*}
    \int \sqrt { \mathbb E | \hat f_{yh}(x) - f_y\ast K_h(x) |^2 }dx &\leq c_{y,d,\varepsilon} \sqrt { \frac{{M_0(K^2)}}{n_y h^{d}} } ( 1 +  h^{(d+\varepsilon)/2} ) \\
    \mbox{and } \ \int |f_y\ast K_h (x) - f_y(x) | dx &\leq \left\{ \! \! \begin{array}{l} \phi_y h^2 \ \mbox{ under Assumptions } \ref{cond:kernel} \mbox{ and } \ref{cond:P_X_density2}, \\ \psi_y h \ \mbox{ if } h\leq r_0 \mbox{ under Assumptions } \ref{cond:kernel_compact} \mbox{ and } \ref{cond:P_X_density1_compact}. \end{array} \right.
\end{align*}
%
The constant $c_{y,d,\varepsilon}$ (resp.~$\phi_y$, $\psi_y$), defined right below Lemma~\ref{lemma:conv_kernel} (resp.~Lemma~\ref{lemma:bias_regular_density}), depends on $d$, $\varepsilon$, $f_y$ and $K$ (resp.~$f_y$ and $K$) only, see the Appendix. Let us introduce $\hat c_y$, obtained from $c_{y,d,\varepsilon}$ by plugging in $\hat f_{yh}$ instead of $f_y$. 
One may likewise obtain a (conditional) $L^1-$bound on the KDE $\hat f^*_{1s} $ based upon the oversampled covariates $\{ X_{1i}^* \}_{1\leq i\leq m}$ in the minority class.
Combining these bounds with Theorem~\ref{th:risk_bound_1} results in the following bound on the excess risk of the classification rule $\hat g^*$. 
\begin{theorem}\label{th:risk_bound_smote_1}
Let $m=n_0\mathrm 1_{\{ n_1 >0\}}$ and $\delta\in (0,1/3)$. 
Then, with probability at least $1-3\delta$,
\begin{align*}
    & R_{1/2} (\hat g^*) - R_{1/2} (g) \\
    &\leq \frac{\sqrt{M_0(K^2)}}{2} \left( c_{0,d,\varepsilon} \sqrt{ \frac{1}{n_0 s^d} } + c_{1,d,\varepsilon}\sqrt{ \frac{1}{n_1 h^d} } + \hat{c}_{1,d,\varepsilon}  \sqrt{ \frac{1}{n_0 s^d} } \right) (1 + \max(h,s)^{(d+\varepsilon)/2})  \\
    &+ \sqrt{2 \log\left(\frac 1 {\delta}\right)} \left( \frac{1}{\sqrt{n_0}} + \frac{1}{2\sqrt{n_1}} \right) \\
    &+ \frac{1}{2}\left\{ \! \! \begin{array}{l} \phi_0 s^2 + \phi_1 (h^2 + s^2) \ \mbox{ under Assumptions } \ref{cond:kernel} \mbox{ and } \ref{cond:P_X_density2}, \\ \psi_0 s + \psi_1 (h + s) \ \mbox{ if } h,s\leq r_0 \mbox{ under Assumptions } \ref{cond:kernel_compact} \mbox{ and } \ref{cond:P_X_density1_compact}. \end{array} \right. 
\end{align*}
%
%
\end{theorem}

In Theorem \ref{th:risk_bound_smote_1} the leading term 
(in case of second order regularity) is scaling as $ (n_0 s^d)^{-1/2}  + s^2 + (n_1h^d)^{-1/2}    + h^2$.
The optimal value of the second bandwidth $s$, involved in the error of $\hat f_{0s}$ and $\hat f_{1s}^{*}$ (based on $n_0$ observations), is then $s = n_0 ^{-1/(d+4)}$, while the optimal value for the first bandwidth $h$, involved in the data generation step (based on $n_1$ observations) is $h = n_1 ^{-1/(d+4)}$. In highly imbalanced scenarios where $ n_1 / n_0 \to 0$, the two bandwidths should thus be set differently to optimize the upper bound and reach the optimal convergence rate $n_0^{-2/(d+4)} + n_1^{-2/(d+4)}$. 

\begin{remark}[Comparison with \textsc{Smote}] 
The default value in \textsc{Smote} is $k = 5$ \citep{chawla2002smote} which, viewing $k$NN as a kernel estimator with data-driven bandwidth (see {\it e.g.}~\cite{portier2024nearest2}, Lemma~3), would correspond to a choice of $h \simeq (1/n_1)^{1/d}$ (omitting constants). This is different from the optimal choice above, and it would not (based on our upper bound) guarantee consistency, as the upper bound in Theorem~\ref{th:risk_bound_smote_1} would not converge to $0$.
\end{remark}

\begin{remark}[Comparison with the kernel smoothing plug-in rule]
\label{rmk:comparison}
Our method of proof allows to establish an excess AM-risk bound for the kernel smoothing rule $\hat g_h(x) = \mathrm{1}_{ \{  \hat f_{1h}(x)  >  \hat f_{0h}(x)   \} }$ based on the initial data only (see Proposition~\ref{th:kbc_L1} in the Appendix). The bound scales as 
$ (n_1h^d)^{-1/2} + h^2 + (n_0h^d)^{-1/2}  + h^2 $, the optimal value for $h$ being $h = (1 / n_1 + 1/n_2 ) ^{1/(d+2) }$. The bound obtained in Theorem \ref{th:risk_bound_smote_1} is similar but different, as the two bandwidth parameters, $h$ and $s$, might be set differently. 
\end{remark}

These remarks suggest that the synthetic oversampling rule $\hat g^*$ (i) performs no worse than the rule $\hat g_h$ based on the initial data only, (ii) should be expected to perform better than a rule using \textsc{Smote} synthetic generation with default parameters, and (iii) improvements may be observed in practice by choosing carefully $s$. In the numerical experiments, by considering the $K$-NN classifier, which can be seen as a practical modification of the kernel smoothing classifier, we investigate a cross-validation procedure to choose $K$ (having similar role as $s$), while $h$ is chosen following Scott's rule. 

Investigating the $K$NN plug-in rule (instead of kernel smoothing) in Theorem \ref{th:risk_bound_smote_1} as well as fast convergence rates under the noise condition of \cite{tsyb,tong2013plug}  remain open problems. Finally, Theorem \ref{th:risk_bound_smote_1} is only valid for KDE-based sampling and might be extended to \textsc{Smote}. None of these problems are direct consequences of this work.

\section{Numerical experiments}\label{numerical}


\subsection{Methods in competition}\label{Methods+variants}

\noindent\textbf{Oversampling techniques.}
The 
\textsc{Smote} and \textsc{Kdeo} 
methods are applied to imbalanced data with $n_1 < n_0$. After oversampling, the synthetic data will contain $ m = n_0$ observations with label $1$ and $ n_0 $ observations with label $0$. For \textsc{Smote}, we consider the default choice $k = 5$ neighbors.
For \textsc{Kdeo}, we consider 
the matrix-valued bandwidth $H_1$ such that
\(
H_1^2=  n_1^{-2/(d+4)} C_1,
\)
following from Scott's rule, where \( C_1 \) is the covariance matrix computed from the minority class samples. 

\noindent\textbf{Classification methods.}
We consider 
the kernel smoothing (KS) discrimination rule studied in Section \ref{sec:main:excessrisk} and the $K$NN classification rule as follows. First, apply the concerned oversampling technique, either \textsc{Smote} or \textsc{Kdeo}, and then employ the $K$NN (resp.~KS) algorithm with parameter $K=\sqrt{n}$ (resp.~$s= S_j$ obtained from 
Scott's rule $S_j^2 =  {n_0}^{-2/(d+4)} C_j $, $j = 0,1$, where $ C_j $ is the covariance matrix of class $j$). This choice of $S_j$ (as well as $H_1$ for \textsc{Kdeo}) corresponds to the optimal scaling recommended by Theorem \ref{th:risk_bound_smote_1}. Note that both $K$NN and KS are based on the local averaging principle. 
This is compared to the logistic regression (LR) method, employed to incorporate a parametric classification approach and to illustrate the theory developed in Section \ref{sec:main:concentration}. Finally, we also consider the $K$-NN balanced Bayes classifier (BBC), which does not involve oversampling but rather reweighting. It is defined as the classifier $ \mathrm{1}_{\{ \hat \eta_{NN}(x) > \hat p \}} $ where $\hat \eta_{NN}(x)$ is the $K$NN estimator of $\eta(x) $ with hyperparameter $K = \sqrt{n}$ and \( 
\hat p = n_1/n
\). Under class imbalance, the threshold adjustment allows to minimize the AM-risk \citep{aghbalou2024sharp}. Similar results with a Random Forest classifier applied after \textsc{Smote} and \textsc{Kdeo} are given in the Appendix. 







\subsection{Simulated data}


Four models are considered. In all cases, $\{(X_i, Y_i)\}_{1\leq i\leq n}$ are $n=1000$ i.i.d.~training samples from the distribution of $(X, Y)$. 
Let $\boldsymbol{e}_i$ be the $i$th vector in the canonical basis of $\mathbb{R}^d$.  
\begin{itemize}[leftmargin=*]
    \item \textbf{Example 1:} Let $X\sim \mathcal N ( 0,I_d)$ and $Y \sim \text{Bernoulli}(\expit(X^\top \boldsymbol{e}_1 + \alpha))$, with $\alpha$ tuning class imbalance and $\expit(u) = \exp(u)/(1+\exp(u))$.
    \item\textbf{Example 2:} Let $X\sim \mathcal N ( 0,I_d)$ and $Z \sim \text{GPD}(\sigma(X), \xi = 0.5)$ have a Generalized Pareto distribution, where $\sigma(X) = \exp(X^\top \boldsymbol{e}_1)$. Define $Y = \mathrm{1} _{\{ Z> t \} }$ with $t$ tuning class imbalance.
    \item\textbf{Example 3:} Let $Z = \mathcal{B} \sin(X^\top \boldsymbol{e}_1 /2) Y_1 + (1 - \mathcal{B}) \sin(X^\top \boldsymbol{e_2} /2) Y_2$, with $\mathcal{B} \sim \text{Bernoulli}(0.5)$, $X\sim \mathcal N (0,I_d)$, $Y_1 \sim \text{GPD}(1, 0.5)$, $Y_2 \sim \text{Exp}(10)$. Set $Y = \mathrm{1}_{\{ Z> t \} }$ with $t$ tuning class imbalance.
    \item\textbf{Example 4:} Let $d=2$ and $\mu_1 = (0,0)^\top$, $\mu_2 = (10,10)^\top $, $\mu_3 = (10,0)^\top$, $\mu_4 = (0,10)^\top$. Let $Z$ be $\{1,2,3,4 \}$-valued with $\mathbb P (Z = c)= w_c$, then $X \sim \sum_{c=1} ^4 \mathcal{N}(\mu_c, 6I_p) \mathrm{1} _ {\{Z=c\}}$ and $Y = \mathrm{1} _ {\{Z\geq 3\} }$.
\end{itemize}

Take $d=4$ in Examples $1$, $2$ and $3$. 
In each case, the validation sample is created by generating $10000$ observations and then undersampling the majority class so as to obtain a 
balanced data set.

\noindent\textbf{Results when varying $k$ in \textsc{Smote} and $H$ in \textsc{Kdeo}.} In Figure~\ref{fig:bandwidth-check1}, different values of $k$ and $H$ are considered when dealing with Examples 2 and 3 for which the binary response \( Y \) was constructed by thresholding at the probability $1-p_1=0.90$ (for Examples 1 and 4, see the Appendix). We consider 
\textsc{Smote}($k$) with varying $k \in (7, 65)$, 
and \textsc{Kdeo}($H$) with varying \( H=cH_1 \), for $c$ ranging in \( (1/20, 3)\), where $H_1$ follows from Scott’s rule. The KS, $K$NN and LR classification methods are considered. 
Figure~\ref{fig:bandwidth-check1} displays the average (over $50$ replications) AM-risk over the validation set. 

In Figure~\ref{fig:bandwidth-check1}, we see that varying $k$ in \textsc{Smote}($k$) or the bandwidth $H$ in \textsc{Kdeo}($H$) leads to noticeable changes in the AM-risk values. In contrast, when using LR, changing $H$ or $k$ has almost no effect. 
This indicates that LR is considerably less sensitive to the choice of oversampling parameters $k$ or $H$ compared to nonparametric classifiers. 
This 
was already suggested by Corollaries \ref{cor:smote} and \ref{cor:kde}, while the impact of a real-valued bandwidth 
was formally analyzed in Theorem \ref{th:risk_bound_smote_1} when using KS. 
Observe also that, in Figure~\ref{fig:bandwidth-check1} (left panel), the optimal value of $k$ for \textsc{Smote}(k) when using $K$NN lies between 45 and 60. 
This differs substantially from the 
default $k=5$. 
Figure~\ref{fig:bandwidth-check1} (right panel) shows that varying the bandwidth $H$ does not yield improvements over \textsc{Kdeo}, which relies on Scott’s rule of thumb, that in fact turns out to be optimal. 
Finally, note that \textsc{Kdeo}($H$) with small $H$ 
produces results that closely resemble those of \textsc{Smote}, which is consistent with the intuition that $k$ and $H$ play a similar role. 




\noindent\textbf{Results under different degrees of class imbalance.} 
The methods are evaluated on Examples 2 and 3 (for Examples 1 and 4, see the Appendix), where the parameters \( t \), $\alpha$ and $w_c$ are adjusted to achieve different probability levels (\( 1-p_1 = 0.60, 0.80, 0.85, 0.90, 0.92, 0.95 \)). 
The AM-risk 
of the $K$NN and KS classifiers for Examples 2 and 3 is reported in Figure~\ref{fig:sim-models1} under \textsc{Kdeo} and \textsc{Smote}. 

Across all examples, both classifiers exhibit similar patterns. A clear performance gain is observed for \textsc{Kdeo} and BBC compared with \textsc{Smote}.
In particular, the AM risk of the KS classifier under \textsc{Kdeo} sampling is consistently smaller than that of $K$NN when trained under \textsc{Smote}. 
These results suggest that fine-tuning the bandwidth $H_1$ in \textsc{Kdeo}, together with the use of the KS classifier based on well-chosen matrices $S_j$, provides clear benefits. This, in turn, opens the perspective of tuning $K$NN using cross-validation, as investigated below. By contrast, the LR classifier performs uniformly across all resampling methods (see the Appendix). This is in line with the results given in Theorems~\ref{theorem-smote-concentration_uniform} and \ref{th:smote_expectation} as they suggest that estimating the risk with \textsc{Smote} or \textsc{Kdeo} gives better results when $K$ and $H$ are small, underlining that oversampling may not be critical for such a parametric classifier.


\noindent\textbf{Cross validation (CV) selection of \( K \) in $K$NN: \textsc{Smote-CV},
\textsc{Kdeo}-CV, BBC-CV.}
To improve the \( K \)NN classifier's performance, 
we consider selecting 
\( K \) via cross-validation. Specifically, we use 5-fold cross-validation, where in each fold, the data is first balanced using the chosen oversampling method (\textsc{Smote}/\textsc{Kdeo}), and then \( K \) is selected (from a 
grid) to minimize the validation error. For the BBC, the procedure differs: we apply the BBC classifier obtained from training folds and choose the value of \(K\) that minimizes an AM-risk estimate obtained via under-sampling of the testing fold. 

The AM-risk across BBC, \textsc{Smote}, and \textsc{Kdeo}, and the CV choice of $K$
are reported for Examples 2 and 3 in Figure~\ref{fig:sim-models} (resp.~Examples 1 and 4 given in the Appendix). 
A noticeable 
improvement can be observed in the results for \textsc{Kdeo}-CV and \textsc{Smote}-CV. The 
CV on $K$ yields poorer performance for the BBC, suggesting that the BBC and CV may not interact effectively in this context. 
Moreover, \textsc{Kdeo} and \textsc{Kdeo}-CV consistently outperform \textsc{Smote}, \textsc{Smote}-CV, and BBC, even in scenarios where the probability 
$p_1$ is smaller. 
The intuition is that \textsc{Kdeo}, which generates synthetic samples using kernel density estimates, may better approximate the true minority distribution than \textsc{Smote}.



\subsection{Real data analysis: Abalone, California, MagicTel, Phoneme, and House\_16H datasets}
Each dataset is split into training and validation sets using a $70:30$ ratio. California and MagicTel are 
balanced datasets; to make the training set imbalanced, we subsample the minority class in these two datasets, adjusting the imbalance ratio to $20\%$, $10\%$, and $5\%$. 
The validation sets are balanced 
by undersampling the majority class to ensure fair assessment aligned with the evaluation metric.

The AM-risk results for the $K$NN classifier are presented in Figure~\ref{fig:realdata} 
(for alternative classifiers, see the Appendix). \textsc{Smote} and \textsc{Kdeo} generally perform similarly 
except for Abalone, where \textsc{Smote} consistently achieves a smaller risk. This suggests that the Abalone features might have a specific structure that is better synthesized by nearest neighbors 
rather than by the (more arbitrary) Gaussian kernel. One particularly relevant conclusion is that using CV for choosing $K$ compared to the choice $K= \sqrt n$ (almost) continuously improves the performance across all the different cases. A similar pattern is observed for Random Forest, LR, and LR-Lasso, see the Appendix.

\begin{figure}[h]
	\begin{subfigure}{.5\textwidth}
		\centering
\includegraphics[width=1\linewidth, height=0.20\textheight]{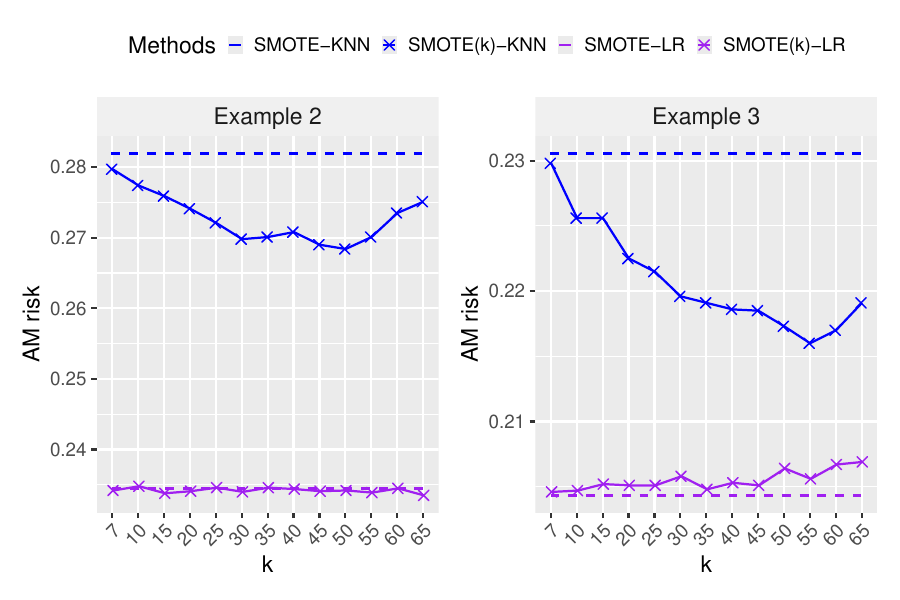}  
	\end{subfigure}
	\begin{subfigure}{.5\textwidth}
		\centering
		\includegraphics[width=1\linewidth, height=0.20\textheight]{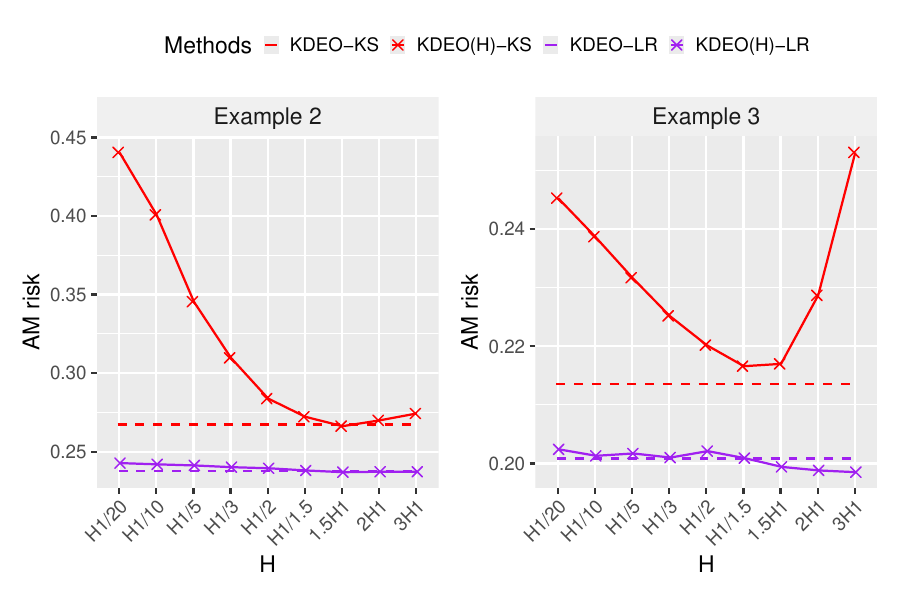}  
	\end{subfigure}
	\caption{Average AM-risk of KNN, KS, and LR classifiers on balanced data over 50 replications. 
    \textit{Left:} using \textsc{Smote} and \textsc{Smote($k$)} with $k \in (7, 65)$. 
    \textit{Right:} using \textsc{Kdeo} and  \textsc{Kdeo}\((H)\), with \( H=cH_1 \) and $c$ ranging in \( (1/20, 3) \) where $H_1$ follows from Scott’s rule.
    }
	\label{fig:bandwidth-check1}
\end{figure}

\begin{figure}[h]
	\begin{subfigure}{.5\textwidth}
		\centering
		\includegraphics[width=1\linewidth, height=0.20\textheight]{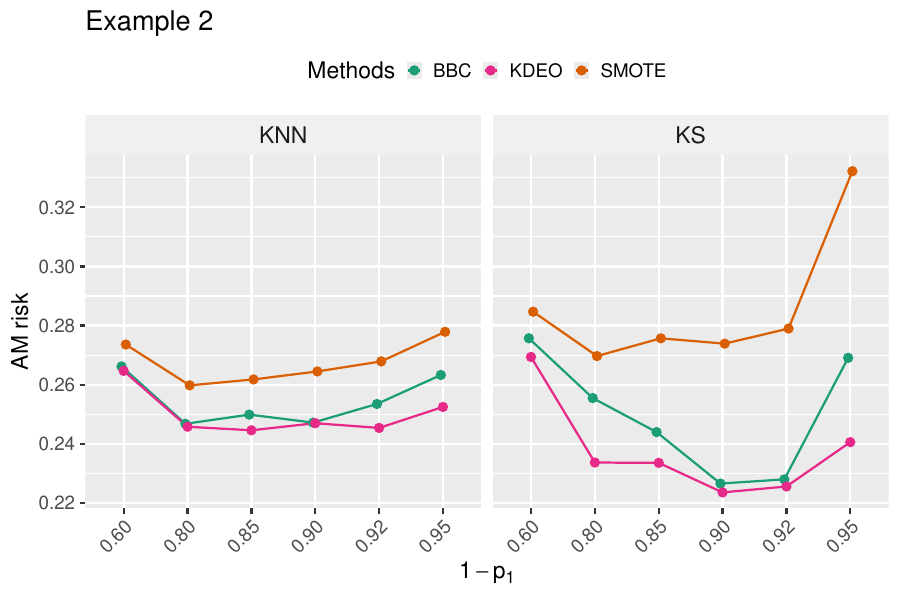}  
	\end{subfigure}
       \begin{subfigure}{.5\textwidth}
		\centering
\includegraphics[width=1\linewidth, height=0.20\textheight]{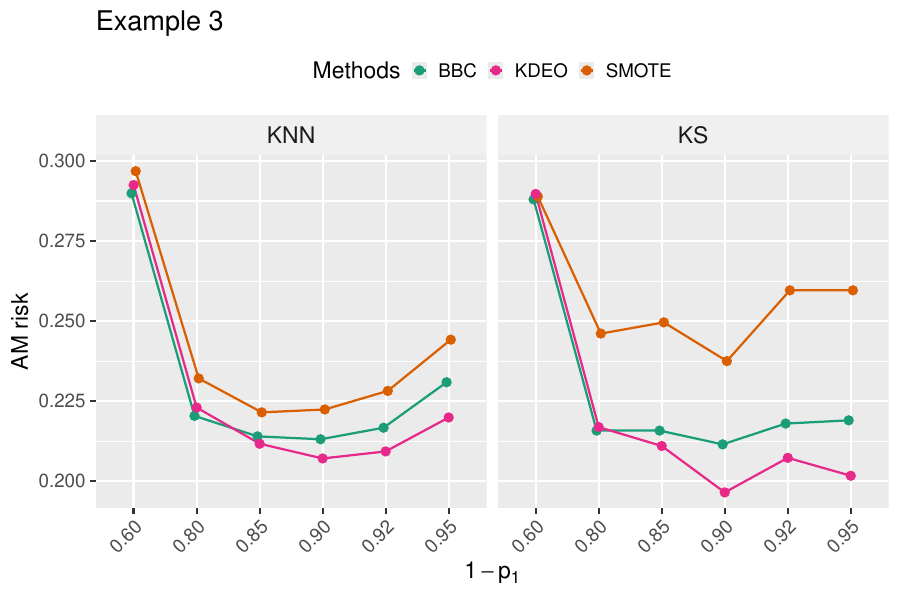}  
	\end{subfigure}
	\caption{Average AM-risk across different data imbalance regimes for the KS~(described in Section~\ref{sec:main:excessrisk}) and $K$NN classification rules computed over 
    50 replications.
    }
	\label{fig:sim-models1}
\end{figure}

\begin{figure}[h]
	\begin{subfigure}{.5\textwidth}
		\centering
		\includegraphics[width=1\linewidth, height=0.25\textheight]{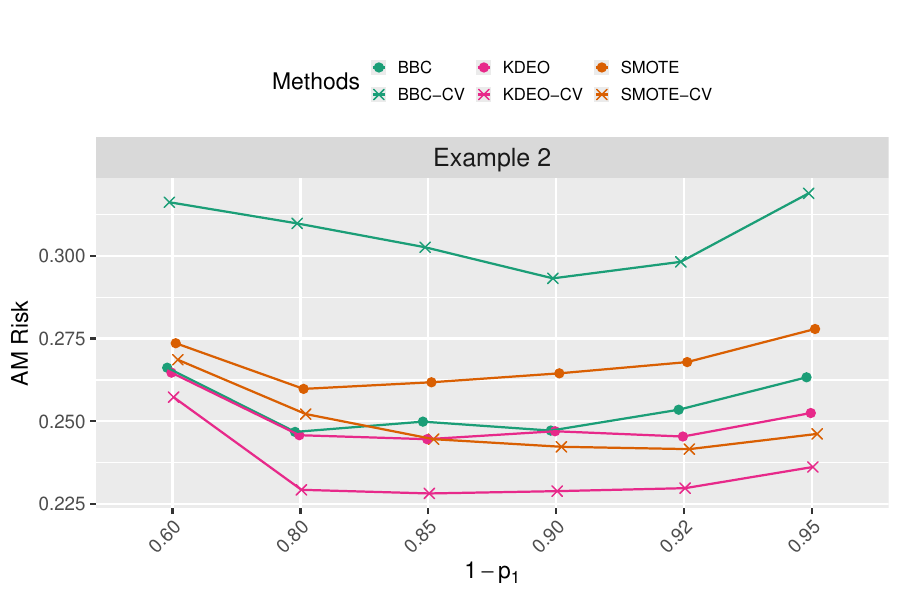}  
	\end{subfigure}
       \begin{subfigure}{.5\textwidth}
		\centering
\includegraphics[width=1\linewidth, height=0.25\textheight]{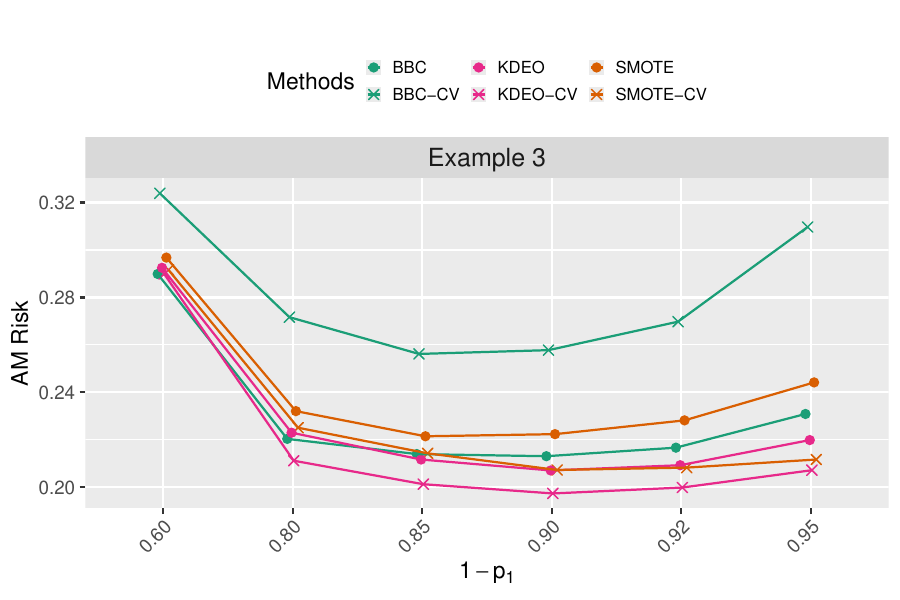}  
	\end{subfigure}
	\caption{Average AM-risk across different data imbalance regimes for the $K$NN methods described in Section~\ref{Methods+variants}, computed over
    50 replications.
    }
	\label{fig:sim-models}
\end{figure}

\begin{figure}[h]
		\centering
\includegraphics[width=1\linewidth, height=0.28\textheight]{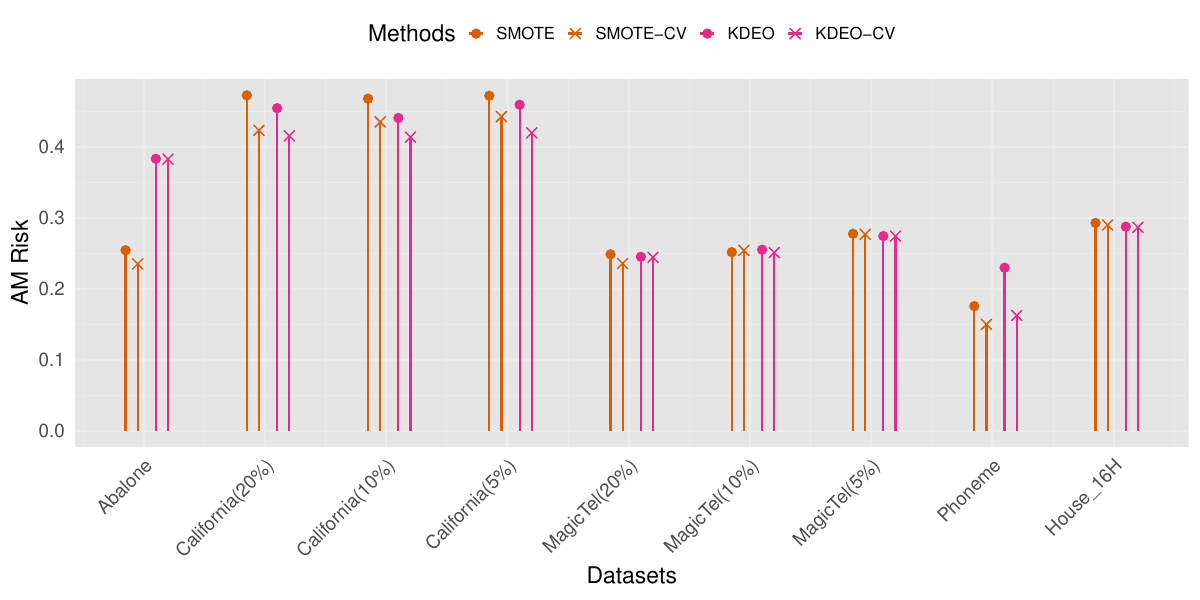}  
	\caption{AM-risk corresponding to different rebalancing methods 
    and datasets when using the $K$NN classifier.}
	\label{fig:realdata}
\end{figure}

\clearpage

\section*{Acknowledgments}
{T.~Ahmad acknowledges support from the R\'egion Bretagne through project SAD-2021-MaEVa. G.~Stupfler acknowledges support from grants ANR-19-CE40-0013 (ExtremReg project), ANR-23-CE40-0009 (EXSTA project) and ANR-11-LABX-0020-01 (Centre Henri Lebesgue), the TSE-HEC ACPR Chair ``Regulation and systemic risks'', and the Chair Stress Test, RISK Management and Financial Steering of the Foundation Ecole Polytechnique.}

\bibliographystyle{plainnat}
\bibliography{references.bib}


\clearpage

\appendix


\clearpage

\appendix

\begin{center}
{\Large Appendix to ``Concentration and excess risk bounds for imbalanced classification with synthetic oversampling''} \\[2ex]
{\large Touqeer Ahmad, Mohammadreza M. Kalan, François Portier, Gilles Stupfler}
\end{center}

This appendix is organized as follows. We provide auxiliary theoretical results in Appendix~\ref{app:aux}, among which classical concentration inequalities and inequalities for convolution of density functions. We then prove Theorem \ref{theorem-smote-concentration_uniform} in Appendix~\ref{app:proofsmote} and Theorem~\ref{th:smote_expectation} in Appendix~\ref{app:proofKDEO}. A unified proof of Corollaries \ref{cor:smote} and \ref{cor:kde} is given in Appendix~\ref{app:proofcoro}. Appendix~\ref{app:riskbound} and Appendix~\ref{app:riskboundsmote} are dedicated to the proofs of Theorems~\ref{th:risk_bound_1} and~\ref{th:risk_bound_smote_1}, respectively. Appendix~\ref{app:remarksmoothing} gives the statement of Proposition~\ref{th:kbc_L1} which is used in Remark~\ref{rmk:comparison} of the paper. Appendix~\ref{app:numerical} gives a further set of numerical results complementing those of Section~\ref{numerical}.

\section{Auxiliary results}
\label{app:aux}

\subsection{Concentration inequalities} 
\label{app:aux:concentration}

The following lemma provides a uniform bound on the distance between a point and its $k$-th nearest neighbor.
\begin{lemma}[\cite{xue2018achieving}, Lemma~1]
\label{lemma_uniform_bound}
    Fix \( {n_1>0} \) and let $\{X_{1i}\}_{1\leq i\leq n_1}$ be i.i.d.~samples from a distribution $\mathbb{P}_1$ 
    satisfying Assumption~\ref{assump_regularity}. Define $r_k(x)$ as the Euclidean distance from a point $x\in \text{supp}(\mathbb{P}_1)$ to its $k$-th nearest neighbor {in $\{X_{1i}\}_{1\leq i\leq n_1}$.} Then, for all $\delta\in (0,1)$, with probability at least $1-\delta$, it holds that
    \begin{align*}
    \forall k\in \{1,\ldots,n_1\}, \quad \sup_{x\in \mathrm{supp}(\mathbb{P}_1)} r_k(x)\leq \left(\frac{3}{C_d}\right)^{1/d} \left(\frac{\max(k,(d+1)\log(2n_1) + \log(8/\delta))}{n_1}\right)^{1/d}.
    \end{align*}
\end{lemma}
We next recall, without proof, the classical 
McDiarmid and Talagrand–Bousquet inequalities. 
\begin{lemma}[McDiarmid inequality, see~\cite{mcdiarmid1989method}] \label{lemma:mcdiarmid}
Let \( Z_1, Z_2, \dots, Z_n \) be independent random variables taking values in some set \( \mathcal{X} \). Let 
\(T: \mathcal{X}^n \rightarrow \mathbb{R}\)
be a function satisfying the \textbf{bounded differences condition}: for each \( i \in \{1, 2, \dots, n\} \), there exists a constant \( C_i \geq 0 \) such that for all \( z_1, \dots, z_n, z_i' \in \mathcal{X} \),
\[
\left| T(z_1, \dots, z_{i-1}, z_i, z_{i+1}, \dots, z_n) - T(z_1, \dots, z_{i-1}, z_i', z_{i+1}, \dots, z_n) \right| \leq C_i.
\]
Then, for all \( t > 0 \),
\begin{align*}
\mathbb{P} ( T(Z_1, \dots, Z_n) - \mathbb{E}[T(Z_1, \dots,Z_n)] \geq t )
&\leq \exp\left( -\frac{2t^2}{\sum_{i=1}^n C_i^2} \right) \\
\mbox{and\quad } \mathbb{P} ( T(Z_1, \dots, Z_n) - \mathbb{E}[T(Z_1, \dots, Z_n)] \leq -t )
&\leq \exp\left( -\frac{2t^2}{\sum_{i=1}^n C_i^2} \right).
\end{align*}
In particular, for any $\delta\in (0,1/2)$, it holds with probability at least $1-2\delta$ that 
\[
|T(Z_1, \dots, Z_n) - \mathbb{E}[T(Z_1, \dots,Z_n)]| \leq \sqrt{\frac{\log(1/\delta)}{2} \sum_{i=1}^n C_i^2}.
\]
\end{lemma}

\begin{lemma}[\cite{talagrand1996new}, Theorem~2.3 in~\cite{bousquet2002bennett}; with separability assumptions,~\cite{10.1093/acprof:oso/9780199535255.001.0001} p.315]\label{lem:bousquet}
Let $Z_1,\dots,Z_n$ be independent and identically distributed random variables with values in a
measurable space $\mathcal X$. Let $\mathcal F$ be a separable class of measurable
functions $f\colon\mathcal X\to\mathbb R$ that satisfy, for some $U>0$, 
\[
      \mathbb E f(Z_i)=0
      \quad\text{and}\quad
      |f(x)|\le U
      \quad\text{for all }f\in\mathcal F,\;x\in\mathcal X .
\]
Let \(
   \sigma^2 = 
\sup_{f\in\mathcal F}        \operatorname{Var}\bigl(f(Z_1)\bigr)\) . Then for every $t\ge 0$,
\begin{align}\label{uniform_probability}
   \mathbb{P}\left(
        S_n \ge
        \mathbb E(S_n)
        + \sqrt{2(n\sigma^2 + 2U \mathbb E(S_n))t}
        + \frac{U}{3} t
       \right)
      \le e^{-t},
\end{align}
where  $S_n$ can be either equal to $ \sup_{f\in\mathcal F}\;\sum_{i=1}^{n}f(Z_i)$ or $    \sup_{f\in\mathcal F}\; |\sum_{i=1}^{n}f(Z_i)|$. 
%
%
\end{lemma}

The next lemma simplifies the bound in the conclusion \eqref{uniform_probability} of Lemma \ref{lem:bousquet} to a form that is more convenient for our purposes. 
\begin{lemma}\label{simplified_term_uniform_probability}
    Assume the setting of Lemma~\ref{lem:bousquet} and retain the notation therein. Then, for all $\delta\in (0,1)$, 
    \begin{align*}
        \mathbb{P}\left(S_n\leq 2\mathbb{E}(S_n)+\sqrt{2n\sigma^2\log(1/\delta)}+\frac{4U}{3}\log(1/\delta) \right)\geq 1-\delta,
    \end{align*}
where  $S_n$ can be either equal to $ \sup_{f\in\mathcal F}\;\sum_{i=1}^{n}f(Z_i)$ or $    \sup_{f\in\mathcal F}\; |\sum_{i=1}^{n}f(Z_i)|$.
\end{lemma}
\begin{proof}
We take $t=\log(1/\delta)$ and bound the square-root term in \eqref{uniform_probability} as follows:
\begin{align*}
    \sqrt{2  (n\sigma^2 + 2U \mathbb{E}(S_n)) \log(1/\delta)} 
    &\leq \sqrt{2n\sigma^2 \log(1/\delta)} 
      + \sqrt{4U \, \mathbb{E}(S_n) \log(1/\delta)}\\
    &\leq \sqrt{2n\sigma^2 \log(1/\delta)} 
      + \mathbb{E}(S_n) + U\log(1/\delta).
\end{align*}
Apply \eqref{uniform_probability} in Lemma~\ref{lem:bousquet} with $t=\log(1/\delta)$ to complete the proof.
\end{proof}
\subsection{Inequalities for the convolution of density functions}
\label{app:aux:convolution}

The first lemma is the key ingredient in order to control variance terms in integrated $L^1-$deviations of kernel density estimators. It holds under simple conditions on the tails of $f$ and $K$; in particular, it holds if $f$ and $K$ have compact support, and it holds under either Assumptions~\ref{cond:kernel}-\ref{cond:P_X_density2} or under Assumptions \ref{cond:kernel_compact}-\ref{cond:P_X_density1_compact}. Recall the notation $M_{q}(g): = \int \|z\|_2 ^{q} g(z) dz$.

\begin{lemma}\label{lemma:conv_kernel}
    Let $f,K: \mathbb R^d \to \mathbb R$ be two density functions such that $M_{d+\varepsilon}(f+K) < \infty$ for some $\varepsilon>0$. Then it holds that 
    \begin{align*}
    \int \sqrt{ f\ast K_h (x) } dx &\leq \sqrt{V_d} \left( 1 + \sqrt{\frac{d}{\varepsilon} 2^{d+\varepsilon-1} ( M_{d+\varepsilon}(f)+h^{d+\varepsilon} M_{d+\varepsilon}(K) ) } \right) \\
    &\leq c( 1 + h^{(d+\varepsilon)/2})   
    \end{align*}
    where $c= C_{d,\varepsilon}  ( 1 +\sqrt{M_{d+\varepsilon}(f+ K)})$ and $C_{d,\varepsilon}$ is a constant that depends on $d$ and $\varepsilon$ only.
\end{lemma}
With this lemma at our disposal and the notation of our paper, we define $c_{y,d,\varepsilon}=C_{d,\varepsilon} (1 +\sqrt{M_{d+\varepsilon}(f_y+K)})$, which is a constant that appears in Theorem~\ref{th:risk_bound_smote_1} and that depends on $d$, $\varepsilon$, $f_y$ and $K$ only. Extensions of this lemma may be found in \citet[Lemma 7 and Proposition 8]{HOLMSTROM1992245}, where an $L^1$-rate of convergence is established for the kernel density estimator; the above version is sufficient for our purposes.
\begin{proof} Let $g$ be a density function on $\mathbb{R}^d$. By the Cauchy-Schwarz inequality, it holds
\[
\int_{\|x\|_2>1} \sqrt{g(x)} dx \leq \sqrt{ \int_{\|x\|_2>1} \|x\|_2 ^{d+\varepsilon}  g(x) dx  } \sqrt{ \int_{\|x\|_2>1} \|x\|_2 ^{-(d+\varepsilon)}  dx   }.
\]
The second integral in the right-hand side can be calculated using polar coordinates:
\[
\int _{\|x\|_2>1} \|x\|_2 ^{-(d+\varepsilon)}  dx   = dV_d  \int_{\rho>1} \rho ^{-(d+\varepsilon)}  \rho^{d-1} d\rho =  \frac{dV_d}{\varepsilon}
\]
where $V_d$ is the volume of the unit ball in $\mathbb{R}^d$. Besides, by the Jensen inequality, 
\[
\int_{\|x\|_2\leq 1} \sqrt{g(x)} \frac{dx}{V_d} \leq  \sqrt{ \int_{\|x\|_2\leq 1} g(x) \frac{dx}{V_d}} \leq \sqrt{\frac{1}{V_d}}.
\]
As a consequence
\[
\int  \sqrt{g(x)} dx = \int_{\|x\|_2>1} \sqrt{g(x)} dx + \int_{\|x\|_2\leq 1} \sqrt{g(x)} dx \leq  \sqrt{V_d} \left( 1 + \sqrt{ \frac{d}{\varepsilon} \int \|x\|_2 ^{d+\varepsilon} g(x) dx  }  \right)
\]
It suffices to apply the above inequality to the density $g:x\mapsto f\ast K_h (x) = \int f(x-z) K_h(z) dz$ and to notice that
\begin{align*}
\int \|x\|_2 ^{d+\varepsilon} f\ast K_h (x) dx &= \iint \|x\|_2 ^{d+\varepsilon}  f(x-z) K_h(z) dx \, dz \\
&= \iint \|s+z\|_2 ^{d+\varepsilon}  f(s) K_h(z) ds \, dz \\
&\leq {2^{d+\varepsilon-1} \left( \iint \|s\|_2^{d+\varepsilon}  f(s) K_h(z) ds \, dz + \iint \| z\|_2 ^{d+\varepsilon}  f(s) K_h(z) ds \, dz \right)} \\
&\leq 2^{d+\varepsilon-1}  \left( {M_{d+\varepsilon}(f) +  h^{d+\varepsilon} M_{d+\varepsilon}(K)} \right)
\end{align*}
to complete the proof.
\end{proof}

We turn to two lemmas dedicated to the control of bias terms in integrated $L^1-$deviations of kernel density estimators. These lemmas are in the same spirit as Lemma 3 in \cite{gine2008uniform} as well as Lemma 6 in \cite{delyon2016integral}, where similar quantities (these articles work with respect to a probability measure instead of the Lebesgue measure) are analyzed using high-order kernels to benefit from the smoothness of $f$. Our approach is somewhat different as we rely on standard kernels {under twice differentiability at most}. The rate of convergence obtained is therefore slower but the results have wider scope. 
Note also that Proposition 3 and Corollary 1 in \cite{gine2008uniform} provide (without rate) the convergence to $0$ of similar quantities. One of the two statements below, which shall be used under Assumptions~\ref{cond:kernel} and~\ref{cond:P_X_density2}, is taken from~\citet[Proposition 4]{HOLMSTROM1992245}. 
%
%
\begin{lemma}[Proposition 4 in~\cite{HOLMSTROM1992245}] 
\label{lemma:bias_regular_density}
Suppose that $K$ is a symmetric density function on $\mathbb R^d$, that $M_2(K) < \infty$ and $f\in W^{2,1}(\mathbb{R}^d)$. Then
\[
\int |f\ast K_h (x) - f(x) | dx  \leq h^2 \phi 
\]
where 
\[
\phi = \frac{1}{2} \left( \sum_{i=1}^d \int |\partial_{ii}^2 f(x)| dx \int x_i^2 K(x) dx + 2 \sum_{1\leq i<j\leq d} \int |\partial_{ij}^2 f(x)| dx \int |x_i x_j| K(x) dx \right)
\]
is a finite constant that depends on the $L^1$-norm of the second weak partial derivatives $\partial_{ij}^2 f$ ($1\leq i,j\leq d$) of $f$ and on $K$ only.
\end{lemma}
With this lemma at our disposal and the notation of our paper, we define $\phi_y$ as being the constant $\phi$ appearing in (ii) with $f=f_y$, which is a constant that appears in Theorem~\ref{th:risk_bound_smote_1} and that depends on $f_y$ and $K$ only.
%

Our next and final lemma is analogous to Lemma~\ref{lemma:bias_regular_density} but only requires the existence and integrability of the weak gradient of $f$, so that it can be applied under Assumptions \ref{cond:kernel_compact} and \ref{cond:P_X_density1_compact}. It should be clear that this result is not a straightforward consequence of the results of~\cite{HOLMSTROM1992245}, because it is not assumed below that $f$ is smooth on the whole of $\mathbb{R}^d$.
\begin{lemma}\label{lemma:technical_bias_general}
Suppose that 
\begin{itemize}
\item $K$ is a symmetric density function supported on $B(0,1)$. 
\item $f$ has a support $S$ such that there are $\kappa,r_0>0$ with $\lambda( \partial S  + B(0,r) ) \leq \kappa r $ for all $r\in (0,r_0)$. 
\item $f\in W^{1,1}(U)$, where $U$ denotes the interior of $S$. 
\item $f$ is bounded on $\partial S + B(0,2r_0)$.
\end{itemize}
If $h\leq r_0$, then it holds that
\[
\int |f\ast K_h (x) - f(x) | dx \leq h \psi,
\]
where, letting $\nabla f$ denote the weak gradient of $f$,
\[
\psi =   M_1(K) \int_S \|\nabla f(x)\|_2 dx + 2 \kappa \sup_{z\in \partial S +  B(0,2r_0)} f(z).
\]
\end{lemma}
{With this lemma at our disposal and the notation of our paper, we define $\psi_y$ as being the constant $\psi$ with $f=f_y$, which is a constant that appears in Theorem~\ref{th:risk_bound_smote_1} and that depends on $f_y$ and $K$ only.}
\begin{proof} Fix $h\leq r_0$. Write 
\[
f\ast K_h (x) - f(x) = \int (f(x-y) - f(x)) K_h(y) dy = \int (f(x+y) - f(x)) K_h(y) dy
\]
by the symmetry assumption on $K$. Let 
$T=\partial S+B(0,h)$. Since $K$ is supported on $B(0,1)$, for $f\ast K_h (x) - f(x)$ to be nonzero it is necessary that either $x\in S\cap T^c$ or $x\in T$. Moreover, 
the assumption on $S$ ensures that $\partial S = S \setminus U$ has Lebesgue measure zero. Hence we have
\begin{align}
\nonumber
\int |f\ast K_h(x) - f(x)| dx &\leq \int_{U\cap T^c} \left| \int_{B(0,h)} (f(x+y) - f(x)) K_h(y) dy \right| dx \\
\label{eqn:technical_bias_decomposition}
    &+ \int_T \left| \int_{B(0,h)} (f(x+y) - f(y)) K_h(y) dy \right| dx .
\end{align}
We start by dealing with the first integral in the right-hand side of~\eqref{eqn:technical_bias_decomposition}. 
By the Meyers-Serrin theorem~\citep{meyser1964}, there exists a sequence $(f_n)$ of functions that are infinitely differentiable on $U$ and such that $f_n \to f$ in $W^{1,1}(U)$, that is, $\int_U (|f_n(x)-f(x)| + \| \nabla f_n(x) - \nabla f(x) \|_2) dx\to 0$. Note that when $x\in U\cap T^c$ and $y\in B(0,h)$, one has $x+y\in U$, and then obviously 
\begin{align*}
&\int_{U\cap T^c} \left| \int_{B(0,h)} (f_n(x+y) - f_n(x)) K_h(y) dy \right| dx \\
&= \int_{U\cap T^c} \left| \int_{B(0,h)} \int_0^1 y^{\top} \nabla f_n(x+ty) K_h(y) dt \, dy \right| dx.
\end{align*}
This leads, for any $n$, to the inequality 
\begin{align}
\nonumber
     &\int_{U\cap T^c} \left| \int_{B(0,h)} (f(x+y) - f(x)) K_h(y) dy \right| dx \\
\nonumber
     &\leq \int_{U\cap T^c} \left| \int_{B(0,h)} \int_0^1 y^{\top} \nabla f(x+ty) K_h(y) dt \, dy \right| dx \\
\nonumber
     &+ \int_{U\cap T^c} \left| \int_{B(0,h)} ((f_n-f)(x+y) - (f_n-f)(x)) K_h(y) dy \right| dx \\
\label{eqn:technical_bias_density}
     &+ \int_{U\cap T^c} \left| \int_{B(0,h)} \int_0^1 y^{\top} \nabla (f_n-f)(x+ty) K_h(y) dt \, dy \right| dx. 
\end{align}
Clearly 
\[
\int_{U\cap T^c} \left| \int_{B(0,h)} ((f_n-f)(x+y) - (f_n-f)(x)) K_h(y) dy \right| dx \leq 2 \int_U |f_n(x)-f(x)| dx
\]
and 
\[
\int_{U\cap T^c} \left| \int_{B(0,h)} \int_0^1 y^{\top} \nabla (f_n-f)(x+ty) K_h(y) dt \, dy \right| dx \leq h M_1(K) \int_U \| \nabla f_n(x) - \nabla f(x) \|_2 dx.
\]
Both of these upper bounds converge to 0 as $n\to\infty$. Then, letting $n\to\infty$ in~\eqref{eqn:technical_bias_density}, we get 
\begin{align*}
&\int_{U\cap T^c} \left| \int_{B(0,h)} (f(x+y) - f(x)) K_h(y) dy \right| dx \\
&\leq \int_{U\cap T^c} \left| \int_{B(0,h)} \int_0^1 y^{\top} \nabla f(x+ty) K_h(y) dt \, dy \right| dx.     
\end{align*}
Now by a change of variables, symmetry of $K$, and the Cauchy-Schwarz inequality,
\begin{align*}
&\left| \int_{B(0,h)} \int_0^1 y^{\top} \nabla f(x+ty) K_h(y) dt \, dy \right| \\
&= \left| \int_{B(0,th)} \nabla f(x - u)^{\top} \left( \int_0^1 t^{-1-d} K_h(u/t) dt \right) u du \right| \\
    &\leq \int_{B(0,th)} \| \nabla f(x - u)\|_2 \left( \int_0^1 t^{-1-d} K_h(u/t) dt \right) \| u \|_2 du \\
    &\leq (\| \nabla f \|_2 \ast \| L \|_2) (x) 
\end{align*}
where $L (u) = ( \int_0^1 t^{-1-d} K_h(u/t) dt ) u$. 
Note that $\int \| L(u)\|_2 du  = \int \|v\|_2 K_h(v) dv = h M_1(K)$. As a consequence
\begin{align}
\nonumber
\int_{U\cap T^c} \left| \int_{B(0,h)} (f(x+y) - f(x)) K_h(y) dy \right| dx &\leq \int_{U\cap T^c} (\| \nabla f \|_2 \ast \| L \|_2) (x) dx \\
\nonumber
    &\leq \int_S \| \nabla f(x) \|_2 dx \int \| L (u) \|_2 du \\
\label{eqn:technical_bias_decomposition_1}
    &= h M_1(K) \int_S \| \nabla f(x) \|_2 dx. 
\end{align} 
Concerning the second integral in~\eqref{eqn:technical_bias_decomposition}, since there $x\in T$ and $y\in B(0,h)$, it holds that the distance of $x+y$ to $\partial S$ is at most $2h$ and then $|f(x+y) - f(x)| \leq 2 \sup_{z\in \partial S +  B(0,2h)} f(z)$. Plugging this along with~\eqref{eqn:technical_bias_decomposition_1} into~\eqref{eqn:technical_bias_decomposition}, we find, for $h\leq r_0$, that 
\begin{align*}
\int |f\ast K_h(x) - f(x)| dx &\leq h M_1(K) \int_S \| \nabla f (x)\|_2 dx + 2 \lambda (T) \sup_{z\in \partial S +  B(0,2 r_0)} f(z).
\end{align*}
The assumption on $S$ yields $\lambda (T) = \lambda(\partial S+B(0,h)) \leq \kappa h$ and then 
\[
\int |f\ast K_h (x) - f(x) | dx \leq h \left( M_1(K)  \int_S \| \nabla f (x)\|_2 dx + 2\kappa \sup_{z\in \partial S +  B(0,2r_0)} f(z)  \right) 
\]
as announced.
\end{proof}
\section{Proof of Theorem \ref{theorem-smote-concentration_uniform}}
\label{app:proofsmote}
Let 
\[
Z^*(\mathcal F) = Z_{\mathrm{Smote}}^*(\mathcal F) := \sup_{G\in \mathcal{F}}\left|\mu_{\mathrm{Smote}}^*(G) 
- \mathbb{E}_{1}[G(X)]
\right| .
\]
Given that $n_1>0$, we decompose the supremum as $Z^*(\mathcal F) \leq Z^*_1(\mathcal F)+ Z^*_2(\mathcal F)$, with 
\begin{align*}  
Z^*_1(\mathcal F)  & = \sup_{G\in \mathcal{F}}\left|\frac{1}{m} \sum_{i=1} ^m  G (X^{*}_{1i}) - \frac{1}{n_1} \sum_{i=1} ^{n_1}  G(X_{1i})\right|\\
\mbox{and } \ Z^*_2(\mathcal F)  & = \sup_{G\in \mathcal{F}}\left| \frac{1}{n_1} \sum_{i=1}^{n_1} G(X_{1i}) - \mathbb{E}_1[G(X)]\right|.
\end{align*}
We further bound each term on the right-hand side separately, beginning with the first one. Let us recall 
\[
\hat{\sigma}_1^2(\mathcal{F}) = \sup_{G \in \mathcal{F}} \left[ \frac{1}{n_1} \sum_{i=1}^{n_1} (G(X_{1i}))^2 - \left( \frac{1}{n_1} \sum_{i=1}^{n_1} G(X_{1i}) \right)^2 \right] .
\]

\begin{lemma}\label{uniform_lemma_synthetic_SMOTE_first_part_new}
    Let $\mathcal{D}_n = \{(X_i, Y_i)\}_{i=1}^{n}$ be a set of $n$ i.i.d. samples drawn from a distribution $\mathbb{P}$, and let $\{X^{*}_{1i}\}_{i=1}^{m}$ be $m$ i.i.d.~samples generated according to the \textsc{Smote} algorithm~\eqref{smote-equation}. Suppose that the minority class distribution $\mathbb{P}_1$ satisfies Assumption~\ref{assump_regularity}. Furthermore, let $\mathcal{F}$ be a function class satisfying Assumptions~\ref{function_class} and~\ref{assump:rad-free}. Fix $\delta\in (0,1/3)$. 
    Then,
    on the event $\{ n_1>0 \}$, we have
\begin{align*}
\mathbb P \left( Z_1^*(\mathcal F)    \leq  4\mathcal{R}_m(\mathcal{F})+\sqrt{\frac{2\hat{\sigma}_1^2(\mathcal{F})\log(1/\delta)}{m}} + \frac{8B}{3m} \log(1/\delta) +  L\left(\frac{6}{C_d}\right)^{1/d} \left(\frac{k_{\delta}}{n_1}\right)^{1/d} \,\middle|\, Y_{1:n} \right)& \\
\geq 1- 3\delta.&
\end{align*}
\end{lemma}
\begin{proof}
    First, recall that {when $n_1>0$ then} at each iteration \( i \), $X^{*}_{1i}$ is drawn uniformly at random on the line linking \( \tilde{X}_{1i} \) to \( \overline{X}_{1i} \), where \( \tilde{X}_{1i} \) is drawn uniformly at random from the minority class samples \( \{X_{1i}\}_{1\leq i\leq n_1} \). {On the event $\{ n_1>0 \}$, we} then consider the following decomposition: 
%
\begin{equation}\label{decomposition_proof_theorem1}
Z^*_1(\mathcal F) \leq \sup_{G\in \mathcal{F}}\left|\frac{1}{m}  \sum_{i=1}^m \left\{  G(X_{1i}^{*}) - G(\tilde{X}_{1i}) \right\} \right| + \sup_{G\in \mathcal{F}}\left|\frac{1}{m} \sum_{i=1}^m  G(\tilde{X}_{1i})  - \frac{1}{n_1} \sum_{i=1}^{n_1} G(X_{1i}) \right|.
\end{equation}
As for the first term in \eqref{decomposition_proof_theorem1}, using the Lipschitz property of the function class \(\mathcal{F}\) stated in Assumption~\ref{function_class}, we have
\begin{align*}
\sup_{G \in \mathcal{F}}
\left|
\frac{1}{m} \sum_{i=1}^{m} \left\{ G(X_{1i}^{*}) - G(\tilde{X}_{1i}) \right\}
\right|
\leq 
\frac{1}{m} \sum_{i=1}^{m} \sup_{G \in \mathcal{F}} 
\left| G(X^*_{1i}) - G(\tilde{X}_{1i}) \right|
&\leq 
\frac{1}{m} \sum_{i=1}^{m} 
L \| X^{*}_{1i} - \tilde{X}_{1i} \|_2 .
\end{align*}
According to the \textsc{Smote} procedure described in Section~\ref{subsec:smote-detail}, we have 
$
\|X^*_{1i} - \tilde{X}_{1i}\|_2 \leq r_k^{(i)}(\tilde{X}_{1i})
$ 
where \(r_k^{(i)}(\tilde{X}_{1i})\) denotes the distance to the \(k\)-th nearest neighbor of \(\tilde{X}_{1i}\) among the sample $S^{(i)} = \{ X_{11},\ldots, X_{1n_1} \} \backslash \{ \tilde{X}_{1i} \}$. Note that, by construction, $r_k^{(i)}(\tilde{X}_{1i}) = r_{k+1}(\tilde{X}_{1i})$, with $r_{k+1}$ introduced in Lemma~\ref{lemma_uniform_bound} and representing the distance to the $(k+1)$-th nearest neighbor within the full sample $\{ X_{11},\ldots, X_{1n_1} \}$. It follows, when $n_1>1$, that
\begin{equation}
\label{eqn:bound_F1_SMOTE}
\sup_{G \in \mathcal{F}}
\left|
\frac{1}{m} \sum_{i=1}^{m} \left\{ G(X_{1i}^{*}) - G(\tilde{X}_{1i}) \right\}
\right|
\leq L \max_{1\leq i\leq m} r_k^{(i)}(\tilde{X}_{1i})\leq  L \sup_{x\in \mathrm{supp}(\mathbb P_1)} r_{k+1} (x).
\end{equation}
Conditionally on $Y_{1:n}$ and given that $n_1>1$, $\{ X_{11},\ldots, X_{1n_1} \}$ is an i.i.d. sample of $n_1$ elements with common distribution $\mathbb P_1$. Then applying Lemma~\ref{lemma_uniform_bound}, we obtain that, whenever $n_1>1$, 
\[
\mathbb P\left( \sup_{x\in \mathrm{supp}(\mathbb P_1)} r_{k+1} (x) > 
\left( \frac{3}{C_d} \right)^{1/d} 
\left( \frac{\max(k+1,(d+1)\log(2n_1) + \log(8/\delta))}{n_1} \right)^{1/d} \, \middle| \, Y_{1:n} \right) \! \leq \delta.
\]
Clearly $\max(k+1,(d+1)\log(2n_1) + \log(8/\delta)) \leq \max(2k,(d+1)\log(2n_1) + \log(8/\delta))\leq 2k_{\delta}$ so that, when $n_1>1$, 
\[
\mathbb P\left( \sup_{x\in \mathrm{supp}(\mathbb P_1)} r_{k+1} (x) > 
\left( \frac{6}{C_d} \right)^{1/d} 
\left( \frac{k_{\delta}}{n_1} \right)^{1/d} \, \middle| \, Y_{1:n} \right) \leq \delta.
\]
Dealing with the case $n_1=1$ separately, for which the inequality below is trivial, and using~\eqref{eqn:bound_F1_SMOTE}, we obtain that whenever $n_1>0$,
\[
\mathbb P \left(  \sup_{G \in \mathcal{F}}
\left|
\frac{1}{m} \sum_{i=1}^{m} \left\{ G(X_{1i}^{*}) - G(\tilde{X}_{1i}) \right\}
\right| \leq L \left( \frac{6}{C_d} \right)^{1/d} 
\left( \frac{k_{\delta}}{n_1} \right)^{1/d} \, \middle| \, Y_{1:n} \right) \geq 1-\delta.
\]
For the second term in \eqref{decomposition_proof_theorem1}, we further work conditionally on $\mathcal{D}_n$ while assuming that $n_1>0$. For any \(G \in \mathcal{F}\), define 
\(\psi_G(x) = G(x) - \mu(G)\), where 
\(\mu(G) = \frac{1}{n_1} \sum_{i=1}^{n_1} G(X_{1i})\).
Note that since \(\tilde{X}_{11}\) is drawn uniformly at random from 
\(\{X_{1i}\}_{1\leq i\leq n_1}\), we have $
\mathbb{E}( G(\tilde{X}_{11}) \mid \mathcal{D}_n ) = \mu(G)$, so $\psi_{G}(\tilde{X}_{11})$ is centered given $\mathcal{D}_n$. Let 
\[
S_m=\sup_{G\in \mathcal{F}} \frac{1}{m}\sum_{i=1}^{m}\psi_{G}(\tilde{X}_{1i})=\sup_{G\in \mathcal{F}}\left(\frac{1}{m} \sum_{i=1}^{m}  G(\tilde{X}_{1i})  - \frac{1}{n_1}\sum_{i=1}^{n_1} G(X_{1i}) \right).
\]
The class of functions $\{ \psi_G, G\in \mathcal{F} \}$ satisfies the assumptions of Lemma~\ref{lem:bousquet} (with the centering assumption understood conditionally on $\mathcal{D}_n$ and $n_1>0$), and 
\begin{align*}
\sup_{G\in \mathcal{F}} \operatorname{Var}\bigl( \psi_{G}(\tilde{X}_{11}) \mid \mathcal{D}_n \bigr) &= \sup_{G \in \mathcal{F}} \mathbb{E}\left( (G(\tilde{X}_{11}) - \mu(G))^2  \right) \\
    &=  \sup_{G \in \mathcal{F}} \left[ \frac{1}{n_1} \sum_{i=1}^{n_1} (G(X_{1i}))^2 - \left( \frac{1}{n_1} \sum_{i=1}^{n_1} G(X_{1i}) \right)^2 \right] = \hat{\sigma}_1^2(\mathcal{F}).
\end{align*}
Using the symmetrization technique, we derive a bound for \( \mathbb{E}[S_m | \mathcal{D}_n] \) in order to apply Lemma~\ref{simplified_term_uniform_probability}. Let \(\{\tilde{X}_{1i}^{\prime}\}_{i=1}^{m}\) be an independent copy of \(\{\tilde{X}_{1i}\}_{i=1}^{m}\), given \(\mathcal{D}_n\), and let $ \{\varepsilon_i\}_{i=1}^{m}$ be independent Rademacher random variables that are independent of the \(\{\tilde{X}_{1i}^{\prime}\}_{i=1}^{m}\), \(\{\tilde{X}_{1i}\}_{i=1}^{m}\), and of $\mathcal{D}_n$. Then: 
\begin{align}
\nonumber
\mathbb{E}[S_m \mid \mathcal{D}_n] 
   &\le \mathbb{E} \left[ \sup_{G \in \mathcal{F}} \frac{1}{m} \sum_{i=1}^{m} \left( G(\tilde{X}_{1i}) - G(\tilde{X}_{1i}^{\prime}) \right) \,\middle|\, \mathcal{D}_n \right] \tag{ghost sample} \\
   &= \mathbb{E} \left[ \sup_{G \in \mathcal{F}} \frac{1}{m} \sum_{i=1}^{m} \varepsilon_i \left( G(\tilde{X}_{1i}) - G(\tilde{X}_{1i}^{\prime}) \right) \,\middle|\, \mathcal{D}_n \right] \tag{symmetrization} \\
\nonumber
   &\le 2\, \mathbb{E} \left[ \sup_{G \in \mathcal{F}} \frac{1}{m} \sum_{i=1}^{m} \varepsilon_i\, G(\tilde{X}_{1i}) \,\middle|\, \mathcal{D}_n \right] \\
\nonumber
   &\leq 2\, \mathcal{R}_m(\mathcal{F}),
\end{align}
where the last inequality follows from Assumption~\ref{assump:rad-free}. Then, by Lemma \ref{simplified_term_uniform_probability} applied to $S_m$, we obtain
%
\begin{align*}
\mathbb{P}\left(
\sup_{G\in\mathcal{F}}\left(
\frac{1}{m}\sum_{i=1}^{m} G(\tilde X_{1i})
- \frac{1}{n_1}\sum_{i=1}^{n_1} G(X_{1i})
\right) 
\le 4\mathcal{R}_m(\mathcal{F}) \right.
&+ \sqrt{\frac{2\hat\sigma_1^2(\mathcal{F})\log(1/\delta)}{m}} \\
&+ \left. \frac{8B}{3m}\log(1/\delta)
\,\middle|\, \mathcal{D}_n
\right) \ge 1-\delta.
\end{align*}
By combining this with the same argument applied to \(-\psi_G = \psi_{-G}\), and using the symmetry of the $\varepsilon_i$, we obtain
\begin{align*}
\mathbb{P}\left(
\sup_{G\in \mathcal{F}}\left|
\frac{1}{m}\sum_{i=1}^{m} G(\tilde{X}_{1i})
- \frac{1}{n_1}\sum_{i=1}^{n_1} G(X_{1i})
\right|
\le 4\mathcal{R}_m(\mathcal{F}) \right.
&+ \sqrt{\frac{2\hat{\sigma}_1^2(\mathcal{F})\log(1/\delta)}{m}} \\
&+ \left. \frac{8B}{3m}\log(1/\delta)
\,\middle|\, \mathcal{D}_n
\right)
\ge 1 - 2\delta,
\end{align*}
%
which implies that when $n_1>0$, we have
\begin{align*}
\mathbb{P}\left(
\sup_{G\in \mathcal{F}}\left|
\frac{1}{m}\sum_{i=1}^{m} G(\tilde{X}_{1i})
- \frac{1}{n_1}\sum_{i=1}^{n_1} G(X_{1i})
\right|
\le 4\mathcal{R}_m(\mathcal{F}) \right.
&+ \sqrt{\frac{2\hat{\sigma}_1^2(\mathcal{F})\log(1/\delta)}{m}} \\
&+ \left. \frac{8B}{3m}\log(1/\delta)
\,\middle|\, Y_{1:n}
\right)
\ge 1 - 2\delta
\end{align*}
%
by integrating out the conditional probability with respect to $X_1,\ldots,X_n$. The proof is complete.
\end{proof}
\begin{lemma}\label{uniform_lemma_synthetic_SMOTE_second_part}
   Let $\mathcal{D}_n = \{(X_i, Y_i)\}_{i=1}^{n}$ be $n$ i.i.d.~samples drawn from a distribution $\mathbb{P}$, and let $\mathcal{F}$ be a function class satisfying Assumptions \ref{function_class} and \ref{assump:rad-free}. Fix $\delta\in (0,1/2)$. Then, whenever $n_1>0$, and if $\sigma_1^2(\mathcal{F}) = \sup_{G\in \mathcal{F}}\mathrm{Var}[G (X)|Y=1]$, we have
  \[
\mathbb{P}\left(Z_2^*(\mathcal F) 
\leq 4\mathcal{R}_{n_1}(\mathcal{F}) + \sqrt{\frac{2\sigma_1^2(\mathcal{F}) \log(1/\delta)}{n_1}} + \frac{8B}{3n_1} \log(1/\delta)\,\middle|\, Y_{1:n} \right)\geq 1-2\delta.
\]
\end{lemma}
\begin{proof} Conditioning on the label sequence $Y_{1:n}=(Y_1,\ldots, Y_n)$ makes $n_1$ deterministic while leaving the $\{X_{1i}\}_{1\leq i\leq n_1}$ i.i.d.. Following the second part of the proof of Lemma \ref{uniform_lemma_synthetic_SMOTE_first_part_new}, define $\psi_{G}(x) = G(x) - \mathbb{E}_1[G(X)]$ and, when $n_1>0$, 
\[
S_{n_1} = \sup_{G \in \mathcal{F}} \frac{1}{n_1} \sum_{i=1}^{n_1} \psi_{G}(X_{1i}).
\]
Then, on $\{ n_1>0 \}$, we obtain by the same argument that $\mathbb{E}[S_{n_1}|Y_{1:n}] \leq 2\mathcal{R}_{n_1}(\mathcal{F})$ and, by Lemma \ref{simplified_term_uniform_probability}, we get
\[
\mathbb{P}\left(S_{n_1} \leq 4\mathcal{R}_{n_1}(\mathcal{F}) + \sqrt{\frac{2\sigma_1^2(\mathcal{F}) \log(1/\delta)}{n_1}} + \frac{8B}{3n_1} \log(1/\delta) \,\middle|\, Y_{1:n}\right)\geq 1-\delta.
\]
Repeating the argument for $-\psi_G$, we obtain the bound
\begin{align*}
\mathbb{P}\left(
\sup_{G \in \mathcal{F}}\left|
\frac{1}{n_1}\sum_{i=1}^{n_1} G(X_{1i})
- \mathbb{E}_1[G(X)]
\right|
\le 4\mathcal{R}_{n_1}(\mathcal{F}) \right.
&+ \sqrt{\frac{2\sigma_1^2(\mathcal{F})\log(1/\delta)}{n_1}} \\
&+ \left. \frac{8B}{3n_1}\log(1/\delta)
\,\middle|\, Y_{1:n}
\right)
\ge 1 - 2\delta
\end{align*}
as required.
\end{proof}
\noindent 
\textit{End of the proof of Theorem~\ref{theorem-smote-concentration_uniform}.} 
%
Let $t>0$. We have
\[
\mathbb P \left(\{ Z^*(\mathcal F) > t \} \,|\, Y_{1:n} \right) =  \mathbb P \left(\{ Z^*(\mathcal F) >t\} \, |\,Y_{1:n} \right) \mathrm{1}_{\{n_1>0\}} +  \mathbb P \left(\{ Z^*(\mathcal F) > t \} \,|\, Y_{1:n} \right) \mathrm{1}_{\{n_1=0\}}.
\]
However, from $Z^*(\mathcal F)\leq Z_1^*(\mathcal F) + Z^*_2(\mathcal F)$, the union bound ensures that, if $t_1+t_2=t$, 
\begin{align*}
    &\mathbb P \left(\{ Z^*(\mathcal F) > t\} \, |\,Y_{1:n} \right) \mathrm{1}_{\{n_1>0\}} \\
&\leq \mathbb P \left(\{ Z_1^*(\mathcal F) > t_1\} \, |\,Y_{1:n} \right) \mathrm{1}_{\{n_1>0\}} + \mathbb P \left(\{ Z^*_2(\mathcal F) > t_2\} \, |\,Y_{1:n} \right) \mathrm{1}_{\{n_1>0\}}.
\end{align*}
By Lemmas \ref{uniform_lemma_synthetic_SMOTE_first_part_new} and \ref{uniform_lemma_synthetic_SMOTE_second_part}, it follows that, whenever $n_1>0$,
\begin{align}\label{important_step}
     \mathbb P \left(\{ Z^*(\mathcal F) > t\} \, |\,Y_{1:n} \right)  \leq 5\delta,
\end{align}
     where $t_1$ and $t_2$ are respectively chosen as the upper bounds in Lemmas \ref{uniform_lemma_synthetic_SMOTE_first_part_new} and \ref{uniform_lemma_synthetic_SMOTE_second_part}, resulting in
\begin{align*}
     t &=4\mathcal{R}_m(\mathcal{F})+4\mathcal{R}_{n_1}(\mathcal{F})+ L\left(\frac{6}{C_d}\right)^{1/d} \left(\frac{k_{\delta}}{n_1}\right)^{1/d}\\
     &+\sqrt{\frac{2\hat{\sigma}_1^2(\mathcal{F}) \log(1/\delta)}{m}} +\sqrt{\frac{2\sigma^2_1(\mathcal{F}) \log(1/\delta)}{n_1}}+ \left(\frac{8B}{3m}+\frac{8B}{3n_1}\right) \log(1/\delta).
\end{align*}
When $n_1=0$, we set $t=+\infty$ and thus $P \left(\{ Z^*(\mathcal F) > t\} \, |\,Y_{1:n} \right) = 0$. As a consequence, we obtain
$
\mathbb P( Z^*(\mathcal F) > t ) = \mathbb{E}( \mathbb P \left(\{ Z^*(\mathcal F) >t\} \, |\,Y_{1:n} \right) \mathrm{1}_{\{n_1>0\}} )\leq  5\delta  \mathbb P ( n_1 >0 )    \leq  5\delta,
$
which completes the proof. \qed

\section{Proof of Theorem \ref{th:smote_expectation}}
\label{app:proofKDEO}

We start exactly as in the proof of Theorem \ref{theorem-smote-concentration_uniform}. On the event $n_1>0$, let 
\[
Z^*(\mathcal F) = Z_{\mathrm{KDEO}}^*(\mathcal F) := \sup_{G\in \mathcal{F}}\left|\mu_{\mathrm{KDEO}}^*(G) - \mathbb{E}_{1}[G(X)]\right|
\]
and write $Z^*(\mathcal F) \leq Z^*_1(\mathcal F)+ Z^*_2(\mathcal F)$, with 
\begin{align*}  
Z^*_1(\mathcal F)  & = \sup_{G\in \mathcal{F}}\left|\frac{1}{m} \sum_{i=1} ^m  G (X^{*}_{1i}) - \frac{1}{n_1} \sum_{i=1} ^{n_1}  G(X_{1i})\right|\\
\mbox{and } \ Z^*_2(\mathcal F)  & = \sup_{G\in \mathcal{F}}\left| \frac{1}{n_1} \sum_{i=1}^{n_1} G(X_{1i}) - \mathbb{E}_1[G(X)]\right|.
\end{align*}
The second term has already been controlled in Lemma \ref{uniform_lemma_synthetic_SMOTE_second_part}. The first term is the focus of the next lemma. 
\begin{lemma}\label{uniform_lemma_synthetic_KDEO_first_part_new}
    Let $\mathcal{D}_n = \{(X_i, Y_i)\}_{i=1}^{n}$ be a set of $n$ i.i.d.~samples drawn from a distribution $\mathbb{P}$, and let $\{X^{*}_{1i}\}_{i=1}^{m}$ be $m$ i.i.d.~samples generated according to KDE-based oversampling~\eqref{smoteKDE-equation}. Let $\mathcal{F}$ be a function class satisfying Assumptions~\ref{function_class} and~\ref{assump:rad-free}. Fix $\delta\in (0,1/3)$.  Then, on the event $\{ n_1>0 \}$, we have 
\[
\mathbb{P}\left( Z^*_1(\mathcal F) \leq 4\mathcal{R}_m(\mathcal{F})+\sqrt{\frac{2\hat{\sigma}_1^2(\mathcal{F}) \log(1/\delta)}{m}} + 5Lh M_1(K)+ \frac{9B}{m} \log(1/\delta) \, \middle| \, Y_{1:n} \right) \geq 1-3\delta.
\]
%
\end{lemma}

\begin{proof} 
    Analogously to the proof of Lemma \ref{uniform_lemma_synthetic_SMOTE_first_part_new}, we work on the event $\{ n_1>0 \}$ to write, as in~\eqref{decomposition_proof_theorem1}: 
    \[
    Z^*_1(\mathcal F) \leq \sup_{G\in \mathcal{F}}\left|\frac{1}{m}  \sum_{i=1}^m \left\{  G(X_{1i}^{*}) - G(\tilde{X}_{1i}) \right\} \right| + \sup_{G\in \mathcal{F}}\left|\frac{1}{m} \sum_{i=1}^m  G(\tilde{X}_{1i})  - \frac{1}{n_1} \sum_{i=1}^{n_1} G(X_{1i}) \right|.
    \]
    For the second term in the right-hand side, we obtained in the proof of Lemma \ref{uniform_lemma_synthetic_SMOTE_first_part_new} that 
    \begin{align}
    \nonumber
    \mathbb{P}\left(
    \sup_{G\in \mathcal{F}}\left|
    \frac{1}{m}\sum_{i=1}^{m} G(\tilde{X}_{1i})
    - \frac{1}{n_1}\sum_{i=1}^{n_1} G(X_{1i})
    \right|
    \le 4\mathcal{R}_m(\mathcal{F}) \right.
    &+ \sqrt{\frac{2\hat{\sigma}_1^2(\mathcal{F})\log(1/\delta)}{m}} \\
    \label{kdeo_bound_1}
    &+ \left. \frac{8B}{3m}\log(1/\delta)
    \,\Bigm|\, Y_{1:n}
    \right)
    \ge 1 - 2\delta
    \end{align}
    %
    %
    on the event $\{ n_1>0 \}$. The first term needs a particular treatment because, by contrast with \textsc{Smote}, the KDE-based perturbation $hW_i$ used to generate $X_{1i}^{*}$ from $\tilde X_{1i}$ is not bounded anymore. We further decompose this first term. 
    Note that conditionally on the initial sample \(\mathcal D_n\) and $\{ n_1>0 \}$, the $\Delta^*_i=\Delta^*_i(G)= G(X_{1i}^{*}) - G(\tilde X_{1i})$ are i.i.d.~and that, for any $G\in \mathcal{F}$, 
    \begin{equation}
    \label{delta1star}
    \mathbb E[ |\Delta_1^{*}| |\mathcal D_n ] = \frac{1}{n_1} \sum_{i=1}^{n_1} \int |G(X_{1i} +hw ) - G(X_{1i}) | K(w) dw \leq Lh \int \|w\|_2 K(w)dw = Lh M_1(K).
    \end{equation}
    We then have, when $n_1>0$, 
    \begin{align*}
    \sup_{G\in \mathcal{F}}\left| \frac{1}{m} \sum_{i=1}^{m} (G(X_{1i}^{*}) - G(\tilde X_{1i})) \right| 
    & \leq  \sup_{G\in \mathcal{F}}\left|\frac{1}{m}\sum_{i=1}^m  (\Delta_i ^{*}  -  \mathbb E[ \Delta_i^{*} |\mathcal D_n ] )\right|+ LhM_1(K).
    \end{align*}
    Since $|\Delta _i^*|\leq 2B$, we have $|\Delta _i^* - \mathbb E [ \Delta^*_i  |\mathcal D_n ] |\leq 4B$ and, due to the above bound on $\mathbb E[ |\Delta_1^{*}| |\mathcal D_n ]$, 
    it holds that 
    \[
    \textrm {Var} (\Delta^*_1 |\mathcal D_n)\leq\mathbb E [\Delta^{*2}_1 | \mathcal D_n  ]\leq 2B \mathbb E [|\Delta^{*}_1| | \mathcal D_n  ] \leq 2B LhM_1(K) .
    \]
    Applying Lemma \ref{simplified_term_uniform_probability} (with absolute value), we find that when $n_1>0$,
    \begin{align*}
    \mathbb{P}\left(
\sup_{G \in \mathcal{F}}\left|
\frac{1}{m}\sum_{i=1}^m \big(\Delta_i^{*} - \mathbb{E}[\Delta_i^{*}\mid\mathcal{D}_n]\big)
\right| \right.
    &\le 2\,\mathbb{E}\!\left[\sup_{G\in\mathcal{F}}\!\left|\frac{1}{m}\sum_{i=1}^m \!\big(\Delta_i^{*} - \mathbb{E}[\Delta_i^{*}\mid\mathcal{D}_n]\big)\right| \Bigm| \mathcal{D}_n \right]\\[3pt]
    & \hspace*{-1cm} + \left. \sqrt{\frac{4BLhM_1(K)\log(1/\delta)}{m}}
    + \frac{16B}{3m}\log(1/\delta)
    \,\middle|\, \mathcal{D}_n
    \right)
    \ge 1 - \delta .
    \end{align*}
    %
    %
    Using~\eqref{delta1star} and the inequality $2\sqrt{ab}\leq a+b$ for any $a,b\geq 0$, we obtain that 
    \[
    \mathbb{P} \left( \sup_{G \in \mathcal{F}}
    \left|
    \frac{1}{m} \sum_{i=1}^{m} (G(X_{1i}^{*}) - G(\tilde X_{1i}))
    \right| \leq 5Lh M_1(K)  + \frac{19B}{3m} \log(1/\delta) \,\middle|\, \mathcal D_n \right) \geq 1-\delta.
    \]
    %
    Then, when $n_1>0$,
    \[
    \mathbb{P}\left(\sup_{G \in \mathcal{F}}
    \left|
    \frac{1}{m} \sum_{i=1}^{m} (G(X_{1i}^{*}) - G(\tilde X_{1i}))
    \right| \leq 5Lh M_1(K)  + \frac{19B}{3m} \log(1/\delta) \,\middle|\, Y_{1:n} \right) \geq 1 - \delta
    \]
    by integrating out the conditional probability with respect to $X_1,\ldots ,X_n$. The result follows by recalling \eqref{kdeo_bound_1}.
\end{proof}

\textit{End of the proof of Theorem \ref{th:smote_expectation}.} The proof is similar to that of Theorem \ref{theorem-smote-concentration_uniform}. Putting together the conclusion of Lemma \ref{uniform_lemma_synthetic_SMOTE_second_part} and that of Lemma \ref{uniform_lemma_synthetic_KDEO_first_part_new}, we obtain that whenever $n_1>0$,
\begin{align}\label{kdeo_important_step}
\mathbb P \left(\{ Z_{\mathrm{KDEO}}^*(\mathcal F) > t\} \, |\,Y_{1:n} \right)  \leq 5\delta,
\end{align}
where 
\begin{align*}
t &=4\mathcal{R}_{n_1}(\mathcal{F}) + 4\mathcal{R}_m(\mathcal{F}) + 5Lh M_1(K)\\
&+\left(\sqrt{\frac{\sigma_1^2(\mathcal{F}) }{n_1}}+ \sqrt{\frac{\hat{\sigma}_1^2(\mathcal{F}) }{m}}\right)  \sqrt{2\log\left(\frac{1}{\delta}\right)}  
 + \frac{B}{3} \left( \frac{27}{m} + \frac{8}{n_1} \right) \log\left(\frac{1}{\delta}\right).
\end{align*}
The result follows because under $n_1= 0$, it holds that $\mathbb P (\{ Z_{\mathrm{KDEO}}^*(\mathcal F) > +\infty\} \, |\,Y_{1:n} )= 0$. \qed

\section{Proof of Corollaries \ref{cor:smote} and \ref{cor:kde}}
\label{app:proofcoro}

We start by proving Corollary \ref{cor:smote} and then we will adapt the proof to obtain Corollary \ref{cor:kde}. Let
\[
\hat R_{1/2} (g) := \frac1 {m+n_0 } \left(\mathrm 1_{ \{ n_1>0 \} } \sum_{i=1}^m \ell(g(X_{1i}^*))  +\mathrm 1_{ \{ n_0>0 \} } \sum_{i=1}^{n_0} \ell(-g(X_{0i})) \right). 
\]
The proof starts with the standard decomposition 
\begin{align*}
    R_{1/2} (\hat g ^*_{\mathcal G}) - \inf_{g\in \mathcal G} R_{1/2} (g) & \leq 2\sup_ {g\in \mathcal G} | \hat R_{1/2} (g) -  R_{1/2} (g)  |.
\end{align*} 
This follows from, first, the inequalities
\begin{align*}
    R_{1/2} (\hat g ^*_{\mathcal G}) - \inf_{g\in \mathcal G} R_{1/2} (g) &= R_{1/2} (\hat g ^*_{\mathcal G}) - \hat R_{1/2} (\hat g ^*_{\mathcal G}) + \hat R_{1/2} (\hat g ^*_{\mathcal G}) - \inf_{g\in \mathcal G} R_{1/2} (g) \\ 
    &\leq \sup_ {g\in \mathcal G} | \hat R_{1/2} (g) -  R_{1/2} (g)  | + \inf_{g\in \mathcal G} \hat R_{1/2} (g) - \inf_{g\in \mathcal G} R_{1/2} (g)
\end{align*} 
and, second, from the inequalities 
\begin{align*}
    \inf_{g\in \mathcal G} \hat R_{1/2} (g) - \inf_{g\in \mathcal G} R_{1/2} (g) &\leq \hat R_{1/2} (g') - \inf_{g\in \mathcal G} R_{1/2} (g) \\ 
    &= \hat R_{1/2} (g') - R_{1/2}(g') + R_{1/2}(g') - \inf_{g\in \mathcal G} R_{1/2} (g) \\
    &\leq \sup_ {g\in \mathcal G} | \hat R_{1/2} (g) -  R_{1/2} (g)  | + R_{1/2}(g') - \inf_{g\in \mathcal G} R_{1/2} (g)
\end{align*} 
valid for any $g'\in \mathcal G$. Then, by definition of $\hat R_{1/2}(g)$ and $R_{1/2} (g)$ and using $m = n_0$, we find, when $n_0,n_1>0$,
\begin{align*}
&    R_{1/2} (\hat g ^*_{\mathcal G}) - \inf_{g\in \mathcal G} R_{1/2} (g)  \\
    &\leq 2\sup_ {g \in \mathcal G} \,  \frac1 {m+n_0 } \left| \sum_{i=1}^m \{ \ell(g(X_{1i}^*)) - \mathbb E _1[ \ell (g(X) ) ]\}  +\sum_{i=1}^{n_0} \{ \ell(-g(X_{0i})) - \mathbb E _0[ \ell ( {-g(X)} ) ]\} \right|   \\
    & \leq   \sup_ {g \in \mathcal G} \left| \frac 1 {m}\sum_{i=1}^m \ell(g(X_{1i}^*))  - \mathbb E _1[ \ell (g(X) ) ] \right|   +  \sup_ {g \in \mathcal G} \left| \frac 1 {n_0}  \sum_{i=1}^{n_0} \ell(-g(X_{0i})) - \mathbb E _0[ \ell ( - g(X) ) ] \right |\\
    & = \sup_ {G \in \ell(\mathcal G)} \left| \frac 1 {m}\sum_{i=1}^m  G(X_{1i}^*)   - \mathbb E _1[   G(X)  ] \right|   +  \sup_ {G \in \ell( - \mathcal G) } \left| \frac 1 {n_0}  \sum_{i=1}^{n_0} G(X_{0i})- \mathbb E _0[  G(X)  ] \right |\\
    & \leq  \sup_ {G \in  \mathcal F} \left| \frac 1 {m}\sum_{i=1}^m G(X_{1i}^*)   - \mathbb E _1[   G(X)  ] \right|   +  \sup_ {G \in \mathcal F} \left| \frac 1 {n_0}  \sum_{i=1}^{n_0} G(X_{0i})- \mathbb E _0[  G(X)  ] \right | \\
    &= Z^*(\mathcal F) + Z_0(\mathcal F)
\end{align*}
where $Z^*(\mathcal F)$ is introduced in the proof of Theorem \ref{theorem-smote-concentration_uniform} and $Z_0(\mathcal F) $ is the rightmost term. Regarding $Z_0(\mathcal F)$, we can use Lemma \ref{uniform_lemma_synthetic_SMOTE_second_part} (with $\mathbb P_0$ instead of $\mathbb P_1$) to get that, when $n_0>0$, it holds that
\[
\mathbb P \left( Z_0(\mathcal F) 
> t_0 \,  \middle|\,Y_{1:n} 
\right) \leq 2\delta,
\]
with 
\[
t_0 = 4\mathcal{R}_{n_0}( \mathcal F ) + \sqrt{\frac{2\sigma_0^2( \mathcal F ) \log(1/\delta)}{n_0}} + \frac{8B}{3n_0} \log(1/\delta).
\]
About the first term, $Z^*(\mathcal F)$, it is shown in \eqref{important_step} that, whenever $n_1>0$,
\[
\mathbb P \left( Z^*(\mathcal F) > t \, |\,Y_{1:n} \right) \leq 5\delta,
\]
with $t$ defined just below \eqref{important_step}. The union bound gives that, whenever $n_0,n_1>0$,
\[
\mathbb P \left(   R_{1/2} (\hat g ^*_{\mathcal G}) - \inf_{g\in \mathcal G} R_{1/2} (g) > t+t_0 \, |\,Y_{1:n} \right) \leq 7\delta .
\]
Then, setting $t + t_0 =+\infty$ in case $n_0n_1 = 0$, and considering the cases $n_0 n_1>0$ and $n_0 n_1=0$ separately, we find
\begin{align*}
    &\mathbb P \left(   R_{1/2} (\hat g ^*_{\mathcal G}) - \inf_{g\in \mathcal G} R_{1/2} (g) > t+t_0 \, |\,Y_{1:n} \right) \\
    &=\mathbb P \left(   R_{1/2} (\hat g ^*_{\mathcal G}) - \inf_{g\in \mathcal G} R_{1/2} (g) > t+t_0 \, |\,Y_{1:n} \right) \mathrm{1}_{\{n_0 n_1>0\}} \leq 7\delta. 
\end{align*}
%
%
Hence taking the expectation gives
\begin{align*}
     \mathbb P \left(   R_{1/2} (\hat g ^*_{\mathcal G}) - \inf_{g\in \mathcal G} R_{1/2} (g) > t+t_0  \right)
     &\leq 7\delta. 
\end{align*}
%
We obtain the statement of Corollary \ref{cor:smote} by using that $m=n_0$ when rearranging the terms in $t+t_0$.

Concerning Corollary \ref{cor:kde}, the proof proceeds in a similar fashion, up to the application of \eqref{important_step}. At this point, we instead use \eqref{kdeo_important_step} to complete the argument.
\qed

\section{Proof of Theorem~\ref{th:risk_bound_1}}
\label{app:riskbound}

We have 
\begin{align*}
2 {R_{1/2}}(\hat g) &=  \mathbb{E}_1 ( \mathrm{1}_{\{\hat g (X) \neq 1\}} ) +  \mathbb{E}_0 ( \mathrm{1}_{\{\hat g (X) \neq 0\}} ) \\
&= (1-p_1)^{-1} \mathbb{E} ( \mathrm{1}_{\{\hat g (X) =1\}} (1-\eta(X)) )+p_1^{-1} \mathbb{E} ( \mathrm{1}_{\{\hat g (X) =0\}}  \eta(X) )
\end{align*}
by taking conditional expectations with respect to $X$. Besides, $\mathbb{E} ( \eta(X) ) = p_1$ and then
\begin{align*}
2 R_{1/2}(\hat g) &= (1-p_1)^{-1} \mathbb{E} ( \mathrm{1}_{\{\hat g (X) =1\}} (1-\eta(X)) ) - p_1^{-1} \mathbb{E} ( \mathrm{1}_{\{\hat g (X) =1\}} \eta(X) ) + 1 \\
& =  (1-p_1)^{-1} p_1^{-1} \mathbb{E} ( \mathrm{1}_{\{\hat g (X) =1\}} ( p_1-\eta(X)) ) + 1.
\end{align*}
It follows that 
\[
2 (R_{1/2}(\hat g) - R_{1/2}(g)) =  (1-p_1)^{-1} p_1^{-1} \mathbb{E} ( (\mathrm{1}_{\{\hat g (X) =1\}}  - \mathrm{1}_{\{ g (X) =1 \}}  ) ( p_1-\eta(X)) ).
\]
Observe that 
$ 
\mathrm{1}_{\{\hat g (X) =1\}} - \mathrm{1}_{\{ g (X) =1\}}
= \mathrm{1}_{\{\hat g (X) \neq g (X) \}} ( \mathrm{1}_{\{ g (X) = 0 \}} - \mathrm{1}_{\{ g (X) = 1 \}}) 
$ 
and that, by definition of $g$, $g(X)=0$ if and only if $\eta(X)\leq p_1$. As such
\[
( \mathrm{1}_{\{ g (X) = 0 \}} - \mathrm{1}_{\{ g (X) = 1 \}} )(p_1-\eta(X)) =| p_1-\eta(X) | 
\]
and then
\begin{align*}
2 (R_{1/2}(\hat g) - R_{1/2}( g)) &= (1-p_1)^{-1} p_1^{-1} \mathbb{E} ( \mathrm{1}_{\{\hat g (X) \neq g(X) \}} |p_1 - \eta(X)| ) \\
&= (1-p_1)^{-1} \int_{S} \mathrm{1}_{\{\hat g (x) \neq g(x) \}} | f(x) - f_1(x)  | dx\\
&= \int_{S} \mathrm{1}_{\{\hat g (x) \neq g(x) \}} |  f_1(x)   -  f_0(x)  | dx\\
&\leq \int_{S} \mathrm{1}_{\{\hat g (x) \neq g(x) \}} |  f_1(x)   -  f_0(x)  - (\hat f_1  (x)  -  \hat f_{0}  (x)  ) | dx
\end{align*}
where to obtain the last inequality, we have used that $\hat g $ and $g$ disagree if and only if $\hat f_1 - \hat f_0  $ and $f_1 - f_0$ have different signs. Conclude using the triangle inequality.\qed
\section{Proof of Theorem~\ref{th:risk_bound_smote_1}}
\label{app:riskboundsmote}

The proof consists in applying Theorem~\ref{th:risk_bound_1} after obtaining two bounds on the $L^1-$errors of $\hat f_{0s}$ and $\hat f_{1s}^{*}$. This leads to proving two preliminary results which are of interest in their own right. {Recall that $\hat {f}_{yh} (x)=0$ if $n_y=0$.}


\begin{proposition}[$L^1-$bound on class-specific kernel density estimators] 
\label{th:kde_L1_rate2}
Let $\delta\in (0,1)$ and $y\in \{ 0,1 \}$. Then, on the event $\{ n_y>0 \}$, we have
\begin{align*}
\mathbb{P}\left( \int  |\hat {f}_{yh} (x)  - f_y(x) | dx \right. &\leq {c_{y,d,\varepsilon}} (1 + h^{(d+\varepsilon)/2})  \sqrt{\frac{M_0(K^2)}{n_y h^d}} + \sqrt{\frac{2 \log(1/\delta)}{n_y}} \\
    &+ \left. \left\{ \! \! \begin{array}{l} \phi_y h^2 \ \mbox{ under Assumptions } \ref{cond:kernel} \mbox{ and } \ref{cond:P_X_density2} \\ \psi_y h \ \mbox{ if } h\leq r_0 \mbox{ under Assumptions } \ref{cond:kernel_compact} \mbox{ and } \ref{cond:P_X_density1_compact} \end{array} \right. \, \middle| \, Y_{1:n} \right) \geq 1-\delta
\end{align*}
%
where $c_{y,d,\varepsilon}$ (resp.~$\phi_y$, $\psi_y$) is defined below Lemma~\ref{lemma:conv_kernel} (resp.~below Lemma \ref{lemma:bias_regular_density}, below Lemma~\ref{lemma:technical_bias_general}) and depends only on $(K,f_y,d,\epsilon)$ (resp.~$(K,f_y)$).
\end{proposition} 
%
%
\begin{proof} Denote throughout by $Y_{1:n}$ the random vector $(Y_1,\ldots, Y_n)$. We start by applying McDiarmid's inequality (Lemma~\ref{lemma:mcdiarmid}), with respect to the conditional probability given $Y_{1:n}$, to 
\[
{T_{Y_{1:n}}(X_1,\ldots, X_n) = \int |\hat f_{yh} (x) - f_y(x)| dx.} 
\]
Note that for any $i\in \{1,\ldots,n\}$, whenever $n_y>0$,
\begin{align*}
    &{|T_{Y_{1:n}}(X_1,\ldots, X_{i-1}, X_i, X_{i+1},\ldots, X_n)  - T_{Y_{1:n}}(X_1,\ldots, X_{i-1}, X_i',X_{i+1},\ldots, X_n)|} \\ 
    &\leq \frac{\mathrm{1}_{\{ Y_i = y \}}}{n_y} \int | K_h(x- X_i) -   K_h(x- X_i') | dx \leq 2 \frac{\mathrm{1}_{\{ Y_i = y \}}}{n_y} =: C_i .
\end{align*}
Remarking that $\sum_{i=1} ^n C_i^2 = (4/n_y) \mathrm{1}_{\{ n_y>0\}}$, {and that the assumption of Lemma~\ref{lemma:mcdiarmid} is trivially correct for $n_y=0$,} we obtain by the McDiarmid inequality that with probability (conditionally on $Y_{1:n}$) at least $1-\delta$,
\[
\int |\hat f_{yh} (x) - f_y(x)| dx \leq \mathbb E \left( \int |\hat f_{yh} (x) - f_y(x)| dx \, \middle| \, Y_{1:n} \right) + \sqrt{ \frac{2 \log(1/\delta)}{n_y} } \mathrm 1_{\{ n_y>0\}}.
\]
Swapping integral and conditional expectation and using the triangle and (conditional) Jensen inequalities, we get 
\begin{align}
\nonumber
\int |\hat f_{yh} (x) - f_y(x)| dx & \leq \int \sqrt{\mathbb E \left. \left[ \left(\hat {f}_{yh} (x) - \mathbb E (\hat{f}_{yh}  (x) | Y_{1:n}) \right)^2 \right| Y_{1:n} \right]} dx \\
\label{bound_intermediate_conditional}
    &+ \int \left|\mathbb E (\hat{f}_{yh}  (x)| Y_{1:n}) - f_y(x) \right| dx + \sqrt{ \frac{2 \log(1/\delta)}{n_y} } \mathrm 1_{\{ n_y>0\}}.
\end{align}
The first two terms depend on $Y_{1:n}$ and are further investigated {on the event $\{ n_y>0 \}$.} For the first term above, which is the variance term, we have, by independence of the random variables $(Y_1,X_1) ,\ldots,(Y_n,X_n)  $, 
\begin{align*}
\mathbb E \left. \left[ \left(\hat {f}_{yh} (x) - \mathbb E (\hat{f}_{yh}  (x) | Y_{1:n}) \right)^2 \right| Y_{1:n} \right] &= \frac{1}{n_y^2} \sum_{i=1}^n \operatorname{Var} ( K_h(x-X_i) | Y_i ) \mathrm{1}_{\{ Y_i = y \}} \\
    & \leq \frac{1}{n_y^2} \sum_{i=1}^n \mathbb E ( K_h^2(x-X_i) | Y_i ) \mathrm{1}_{\{ Y_i = y \}} \\
    & = \frac{1}{n_y} \mathbb E ( K_h^2(x-X) | Y = y ) \\
    & = \frac{M_0(K^2)}{n_y h^d} f_y \ast \tilde K_h (x)
\end{align*}
with $\tilde K = K^2 / M_0(K^2)$. By Lemma \ref{lemma:conv_kernel} applied to $f_y$ and $\tilde K$, 
%
\begin{equation}
\label{bound_variance_conditional}
\int \sqrt {\mathbb E \left. \left[ \left(\hat {f}_{yh} (x) - \mathbb E (\hat{f}_{yh}  (x) | Y_{1:n}) \right)^2 \right| Y_{1:n} \right]} dx \leq {c_{y,d,\varepsilon}}( 1 +  h^{(d+\varepsilon)/2} ) \sqrt { \frac{{M_0(K^2)}}{n_y h^{d}} }.
\end{equation}
For the bias term, noting that 
\[
\mathbb E (\hat{f}_{yh}  (x)| Y_{1:n}) =  \frac{1}{n_y} \sum_{i=1} ^n {\mathbb E ( K_h(x-X_i ) | Y_i )} \mathrm{1}_{\{Y_i = y\}} = \mathbb E( K_h(x-X) |Y = y ) = f_y \ast K_h (x),
\]
we apply Lemma \ref{lemma:bias_regular_density} to obtain that 
\begin{equation}
\label{bound_bias2_conditional}
\int |\mathbb E (\hat{f}_{yh}  (x)| Y_{1:n}) - f_y(x) | dx  \leq \phi_y h^2  
\end{equation}
under Assumptions~\ref{cond:kernel}-\ref{cond:P_X_density2}, and we use Lemma \ref{lemma:technical_bias_general} to find 
\begin{equation}
\label{bound_bias_compact_conditional}
\int |\mathbb E (\hat{f}_{yh}  (x)| Y_{1:n}) - f_y(x) | dx \leq \psi_y h 
\end{equation}
under Assumptions~\ref{cond:kernel_compact}-\ref{cond:P_X_density1_compact} when $h\leq r_0$. Combining~\eqref{bound_intermediate_conditional},~\eqref{bound_variance_conditional},~\eqref{bound_bias2_conditional} and~\eqref{bound_bias_compact_conditional}, we have therefore shown that the event 
\begin{align*}
    E = \left\{ \int |\hat f_{yh} (x) - f_y(x)| dx \right. &\leq {c_{y,d,\varepsilon}} (1 + h^{(d+\varepsilon)/2})  \sqrt{\frac{M_0(K^2)}{n_y h^d}} + \sqrt{\frac{2 \log(1/\delta)}{n_y}} \\
    &\left. + \left| \! \! \begin{array}{l} \phi_y h^2 \ \mbox{ under Assumptions } \ref{cond:kernel} \mbox{ and } \ref{cond:P_X_density2} \\ \psi_y h \ \mbox{ if } h\leq r_0 \ \mbox{ under Assumptions } \ref{cond:kernel_compact} \mbox{ and } \ref{cond:P_X_density1_compact} \end{array} \right. \right\}
\end{align*}
has probability at least $1-\delta$ conditionally on $Y_{1:n}$ and when $n_y>0$, which is the result. 
\end{proof}

\begin{proposition}[$L^1-$bound on kernel density estimators based on \textsc{Kdeo}] 
\label{th:smote_kernel_smooth}
Let $\delta\in (0,1/2)$. Then, on the event $\{ n_1>0 \}$, we have 
\begin{align*}
    \mathbb{P}\left( \int | \hat f^*_{1s} (x) - f_1(x) | dx \right. &\leq {\hat{c}_{1,d,\varepsilon}} ( 1 + s^{(d+\varepsilon)/2}) \sqrt{ \frac{M_0(K^2)}{m s^d} } \\
    &+ {c_{1,d,\varepsilon}} (1 + h^{(d+\varepsilon)/2}) \sqrt{\frac{M_0(K^2)}{n_1 h^d}} \\
    &+ \sqrt{2 \log(1/\delta)} \left( \frac{1}{\sqrt{m}} + \frac{1}{\sqrt{n_1}} \right) \\
    &+ \left. \left\{ \! \! \begin{array}{l} \phi_1 (h^2 + s^2) \ \mbox{ under Assumptions } \ref{cond:kernel} \mbox{ and } \ref{cond:P_X_density2} \\ \psi_1 (h + s) \ \mbox{ if } h,s\leq r_0 \mbox{ under Assumptions } \ref{cond:kernel_compact} \mbox{ and } \ref{cond:P_X_density1_compact} \end{array} \right. \, \middle| \, Y_{1:n} \right) \\
    &\geq 1-2\delta
\end{align*}
where $\hat{c}_{1,d,\varepsilon} =  C_{d,\varepsilon}  ( 1 +\sqrt{M_{d+\varepsilon}(\hat{f}_{1h}+ K)})$ and with the notation of Proposition~\ref{th:kde_L1_rate2}. 
\end{proposition} 
\begin{proof} {
Write the quantity of interest as
\begin{equation}
\label{eqn:L1_smote_conditioning}
\int | \hat f^*_{1s} (x) - f_1(x) | dx \leq \int | \hat f^*_{1s} (x) - \mathbb E ( \hat f^*_{1s} (x) | \mathcal{D}_n ) | dx + \int | \mathbb E ( \hat f^*_{1s} (x) | \mathcal{D}_n ) - f_1(x) | dx. 
\end{equation}
We control the two terms on the right-hand side separately. Define 
\[
T_{\mathcal{D}_n}(X_1^*,\ldots, X_m^*) =\int | \hat f^*_{1s} (x) - \mathbb E ( \hat f^*_{1s} (x) | \mathcal{D}_n ) | dx \, {\mathrm 1_{\{n_1>0\}}}
\]
and note that the function $T_{\mathcal{D}_n}$ satisfies the assumption of Lemma~\ref{lemma:mcdiarmid} with $C_i=(2/m){\mathrm 1_{\{n_1>0\}}}$. Since, given $\mathcal{D}_n$ and $\{ n_1>0 \}$, the $X_{1i}^*$ are i.i.d.~generated according to $\hat f_{1h}$, we can therefore apply the McDiarmid inequality, conditionally on $\mathcal{D}_n$, to obtain, with probability at least $1-\delta$ (conditionally on $\mathcal{D}_n$),
\begin{align*}
     \int | \hat f^*_{1s} (x) - \mathbb E ( \hat f^*_{1s} (x) | \mathcal{D}_n ) | dx &\leq \! \int \! \mathbb E \left[ \left|  \hat f^*_{1s} (x)  - \mathbb E ( \hat f^*_{1s} (x) | \mathcal{D}_n ) \right|  | \mathcal{D}_n  \right] \! dx+ \sqrt{\frac{2\log( 1/\delta) }{m}} {\mathrm 1_{\{n_1>0\}}}.
\end{align*}  
Then the Jensen inequality gives that, with probability at least $1-\delta$ (conditionally on $\mathcal{D}_n$),
\begin{align*}
     \int | \hat f^*_{1s} (x) - \mathbb E ( \hat f^*_{1s} (x) | \mathcal{D}_n ) | dx 
    &\leq \int \sqrt{ \mathbb E \left[ \left( \hat f^*_{1s} (x) - \mathbb E ( \hat f^*_{1s} (x) | \mathcal{D}_n ) \right)^2  | \mathcal{D}_n\right] } dx\\ &+ \sqrt{\frac{2\log( 1/\delta) }{m}} {\mathrm 1_{\{n_1>0\}}}.
\end{align*}  
{Now we investigate the behavior of the first term in the right-hand side on the event $\{ n_1>0 \}$.} We have 
%
\begin{align*}
    \mathbb E \left[ \left( \hat f^*_{1s} (x)  - \mathbb E ( \hat f^*_{1s} (x)   | \mathcal{D}_n ) \right)^2  | \mathcal{D}_n\right] = \frac{1}{m} \operatorname{Var} ( K_{s}  (x - X_{11}^*) | \mathcal{F}_n) &\leq \frac{1}{m} \mathbb E ( K_{s}^2  (x - X_{11}^*) | \mathcal{F}_n) \\
    &= \frac{M_0(K^2)}{m s^d} \hat{f}_{1h} \ast \tilde{K}_s (x) 
\end{align*}
with $\tilde K = K^2 / M_0(K^2)$. Using Lemma \ref{lemma:conv_kernel} and the previous probability bound, we get that the event 
\[
E_1 = \left\{ \int | \hat f^*_{1s} (x) - \mathbb E ( \hat f^*_{1s} (x) | \mathcal{D}_n ) | dx \leq {\hat{c}_{1,d,\varepsilon}} ( 1 + s^{(d+\varepsilon)/2}) \sqrt{ \frac{M_0(K^2)}{m s^d} } + \sqrt{\frac{2\log( 1/\delta)}{m}} \right\}
\]
satisfies $\mathbb P (E_1^c| \mathcal{D}_n) \mathrm 1_{\{ n_1>0\}} \leq \delta$, and then, integrating out the conditional expectation with respect to $X_1,\ldots,X_n$, 
\begin{equation}
\label{eqn:L1_smote_L1error}
\mathbb P (E_1^c| Y_{1:n}) \mathrm 1_{\{ n_1>0\}} \leq \delta. 
\end{equation}  
%
To control the second term in~\eqref{eqn:L1_smote_conditioning}, write 
\begin{align*}
\int | \mathbb E ( \hat f^*_{1s} (x) | \mathcal{D}_n ) - f_1(x) | dx &= \int | \hat{f}_{1h} \ast K_s (x) - f_1(x) | dx \\
    &\leq \int | (\hat{f}_{1h} - f_1) \ast K_s (x) | dx + \int | f_1 \ast K_s (x) - f_1(x) | dx \\
    &\leq \int | \hat{f}_{1h}(x) - f_1(x) | dx + \int | f_1 \ast K_s (x) - f_1(x) | dx. 
\end{align*}
It was shown at the end of Proposition~\ref{th:kde_L1_rate2} that the event 
\begin{align*}
    \left\{ \int |\hat f_{1h} (x) - f_1(x)| dx \right. &\leq {c_{1,d,\varepsilon}} (1 + h^{(d+\varepsilon)/2}) \sqrt{\frac{M_0(K^2)}{n_1 h^d}} + \sqrt{\frac{2 \log(1/\delta)}{n_1}} \\
    &\left. + \left| \! \! \begin{array}{l} \phi_1 h^2 \ \mbox{ under Assumptions } \ref{cond:kernel} \mbox{ and } \ref{cond:P_X_density2} \\ \psi_1 h \ \mbox{ if } h\leq r_0 \ \mbox{ under Assumptions } \ref{cond:kernel_compact} \mbox{ and } \ref{cond:P_X_density1_compact} \end{array} \right. \right\}
\end{align*}
has probability at least $1-\delta$ conditionally on $Y_{1:n}$ and when $n_1>0$. 
This statement with Lemma~\ref{lemma:bias_regular_density} under Assumptions~\ref{cond:kernel}-\ref{cond:P_X_density2}, or Lemma \ref{lemma:technical_bias_general} under Assumptions~\ref{cond:kernel_compact}-\ref{cond:P_X_density1_compact} when $s\leq r_0$, allows to obtain that the event 
\begin{align*}
\nonumber
    E_2 = \left\{ \int | \mathbb E ( \hat f^*_{1s} (x) | \mathcal{D}_n ) - f_1(x) | dx \right. &\leq {c_{1,d,\varepsilon}} (1 + h^{(d+\varepsilon)/2})  \sqrt{\frac{M_0(K^2)}{n_1 h^d}} + \sqrt{\frac{2 \log(1/\delta)}{n_1}} \\
    &+ \left. \left| \! \! \begin{array}{l} \phi_1 (h^2 + s^2) \ \mbox{ under Assumptions } \ref{cond:kernel} \mbox{ and } \ref{cond:P_X_density2} \\ \psi_1 (h + s) \ \mbox{ if } h,s\leq r_0 \mbox{ under Assumptions } \ref{cond:kernel_compact} \mbox{ and } \ref{cond:P_X_density1_compact} \end{array} \right. \right\}
\end{align*}
satisfies
\begin{equation}
\label{eqn:L1_smote_bias}
\mathbb P (E_2^c| Y_{1:n}) \mathrm 1_{\{ n_1>0\}} \leq \delta. 
\end{equation}  
A combination of~\eqref{eqn:L1_smote_conditioning},~\eqref{eqn:L1_smote_L1error} and~\eqref{eqn:L1_smote_bias} yields 
\[
\mathbb P ( E_1^c \cup E_2^c | Y_{1:n}) \mathrm 1_{\{ n_1>0\}} \leq \mathbb P (E_1^c| Y_{1:n}) \mathrm 1_{\{ n_1>0\}} + \mathbb P (E_2^c| Y_{1:n}) \mathrm 1_{\{ n_1>0\}} \leq 2\delta.
\]
The result follows immediately.} 
%
\end{proof}

\noindent \textit{End of the proof of Theorem~\ref{th:risk_bound_smote_1}.} 
%
Since $m=n_0\mathrm 1_{\{n_1>0\}}$, we clearly have $\hat \eta^*(x) >1/2$ if and only if $\hat f_{1s}^{*} (x) > \hat f_{0s} (x)$.
Hence we can apply Theorem \ref{th:risk_bound_1} to obtain
\[
R_{1/2} (\hat g^*) - R_{1/2} (g) \leq \frac 1 2 \int | \hat f_{0s} (x)  - f_0(x) | dx + \frac 1 2 \int | \hat f_{1s}^{*} (x)  - f_1(x) | dx.
\]
Using finally Proposition~\ref{th:kde_L1_rate2} to control the first term and Proposition~\ref{th:smote_kernel_smooth} to control the second one, we conclude that 
\begin{align*}
    \mathbb{P}\Bigg( & R_{1/2} (\hat g^*) - R_{1/2} (g) \\
    &\! \! \! \! \leq {\frac{{c_{0,d,\varepsilon}}}{2}} (1 + s^{(d+\varepsilon)/2}) \sqrt{\frac{M_0(K^2)}{n_0 s^d}} + {\frac{{c_{1,d,\varepsilon}}}{2}} (1 + h^{(d+\varepsilon)/2}) \sqrt{\frac{M_0(K^2)}{n_1 h^d}} \\
    &\! \! \! \! + {\frac{{\hat{c}_{1,d,\varepsilon}}}{2}} (1 + s^{(d+\varepsilon)/2}) \sqrt{\frac{M_0(K^2)}{n_0 s^d}} + \sqrt{2 \log(1/\delta)} \left( {\frac{1}{\sqrt{n_0}} + \frac{1}{2\sqrt{n_1}}} \right) \\
    &\! \! \! \! + {\frac{1}{2}} \left\{ \! \! \begin{array}{l} \phi_0 s^2 + \phi_1 (h^2 + s^2) \ \mbox{ under Assumptions } \ref{cond:kernel} \mbox{ and } \ref{cond:P_X_density2} \\ \psi_0 s + \psi_1 (h + s) \ \mbox{ if } h,s\leq r_0 \mbox{ under Assumptions } \ref{cond:kernel_compact} \mbox{ and } \ref{cond:P_X_density1_compact} \end{array} \right. \Bigg| \, Y_{1:n} \Bigg) \mathrm{1}_{\{n_0n_1>0\}}\geq 1-3\delta.
\end{align*}
Rearrange the above upper bound (taken to be infinite on the event $\{ n_0n_1=0 \}$, and hence also valid on this event) and integrate out the conditional expectation given $Y_{1:n}$ to complete the proof.
\qed
\section{Analysis of the kernel smoothing plug-in rule}
\label{app:remarksmoothing}

Let us finally highlight that we get, as a direct byproduct of Theorem~\ref{th:risk_bound_1} and Proposition~\ref{th:kde_L1_rate2}, the following bound on the risk of the kernel discrimination rule 
\[
\hat{g}(x) = \hat g_h(x) = \mathrm{1}_{ \{  \hat f_{1h}(x)  >  \hat f_{0h}(x)   \} }
\]
based on the initial data only. This will be used for comparison purposes with Theorem~\ref{th:risk_bound_smote_1} in Remark~\ref{rmk:comparison}. 
\begin{proposition} 
\label{th:kbc_L1}
Let $\delta\in (0,1/2)$. 
Then, with probability at least $1-2\delta$,
\begin{align*}
R_{1/2}(\hat g) - R_{1/2}(g) &\leq \frac{\sqrt{M_0(K^2)}}{2} \left( \frac{c_{0,d,\varepsilon}}{\sqrt{n_0 h^d}} + \frac{c_{1,d,\varepsilon}}{\sqrt{n_1 h^d}} \right) (1 + h^{(d+\varepsilon)/2}) \\
    &+ \sqrt{\frac{\log(1/\delta)}{2}} \left( \frac{1}{\sqrt{n_0}} + \frac{1}{\sqrt{n_1}} \right) \\
    &+ \frac{1}{2} \left\{ \! \! \begin{array}{l} (\phi_0+\phi_1) h^2 \ \mbox{ under Assumptions } \ref{cond:kernel} \mbox{ and } \ref{cond:P_X_density2} \\ (\psi_0+\psi_1) h \ \mbox{ if } h\leq r_0 \mbox{ under Assumptions } \ref{cond:kernel_compact} \mbox{ and } \ref{cond:P_X_density1_compact} \end{array} \right.
\end{align*}
with the notation of Proposition~\ref{th:kde_L1_rate2}, where the upper bound is taken to be infinite when $n_0n_1=0$.
\end{proposition} 
%
%
%

\section{Additional numerical results}
\label{app:numerical}

\subsection{Algorithms for oversampling methods}
Algorithm~\ref{alg:SMOTE} explains the different steps of the \textsc{Smote} algorithm for generating synthetic samples while highlighting the different hyperparameters involved. 

  \begin{algorithm}
\begin{algorithmic}[1] 
\Statex{\textbf{Input:} Samples $\{X_{11}, \ldots, X_{1n_1}\}\subset \mathbb{R}^{d}$, number of nearest neighbors $k\in\{0, 1,\ldots, n_1-1\}$, and number of synthetic samples $m$}
   \For{each $i= 1,\ldots , m$} 
\Statex \hspace{\algorithmicindent}{Generate $\tilde{X}_{1i}$ uniformly among  \( \{X_{1i}\}_{1\leq i\leq n_1} \).} 
   \Statex \hspace{\algorithmicindent}{If $n_1>1$ and $k>0$, generate $\overline{X}_{1i}$ uniformly among the $k$ nearest neighbors to $\tilde{X}_{1i}$ in the minority class deprived of $\tilde{X}_{1i}$}. Else, do $\overline{X}_{1i} = \tilde{X}_{1i}$.
\Statex \hspace{\algorithmicindent}{$X_{1i}^{*} = (1-\lambda) \tilde{X}_{1i} + \lambda \overline X_{1i}$ ,
where \( \lambda \sim \mathcal{U}[0,1] \).}
\EndFor
\State{Return $  (X_{1i}^{*},\ldots, X_{1m}^{*})$}
\end{algorithmic}
\caption{\textsc{Smote}}
\label{alg:SMOTE}
\end{algorithm}

\begin{algorithm}
\caption{\textsc{Kdeo}}
\label{alg:KDE_SMOTE}
\begin{algorithmic}[1]
\Statex \textbf{Input:} Samples $\{X_{11}, \ldots, X_{1{n_1} } \}\subset \mathbb{R}^{d}$, number of synthetic samples $m$
\State Compute the (empirical) covariance matrix $S$ of minority samples and set the bandwidth matrix $H_1$ to satisfy $H_1^2 = n_{1}^{-2/(d+4)} S$ (Scott's rule of thumb). 
\For{each $i= 1,\ldots , m$} 
    \Statex 
    \hspace{\algorithmicindent}{Generate $\tilde{X}_{1i}$ uniformly among  \( \{X_{1i}\}_{1\leq i\leq n_1} \).}
    \Statex \hspace{\algorithmicindent}{Generate $W_i \sim \mathcal{N}(0, I_p)$.}
    \Statex \hspace{\algorithmicindent}{Generate $X_{1i}^{*} = \tilde{X}_{1i} + H_1 W_i$.}
\EndFor
 \State{Return $  (X_{1i}^{*},\ldots, X_{1m}^{*})$}
\end{algorithmic}
\end{algorithm}

\begin{figure}[h]
	\begin{subfigure}{.5\textwidth}
		\centering
\includegraphics[width=1\linewidth, height=0.20\textheight]{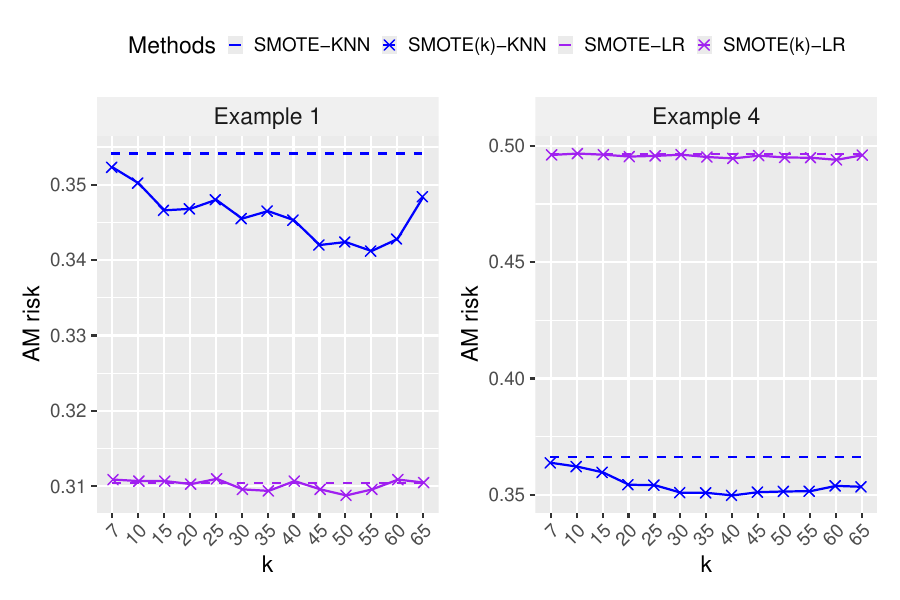}  
	\end{subfigure}
	\begin{subfigure}{.5\textwidth}
		\centering
		\includegraphics[width=1\linewidth, height=0.20\textheight]{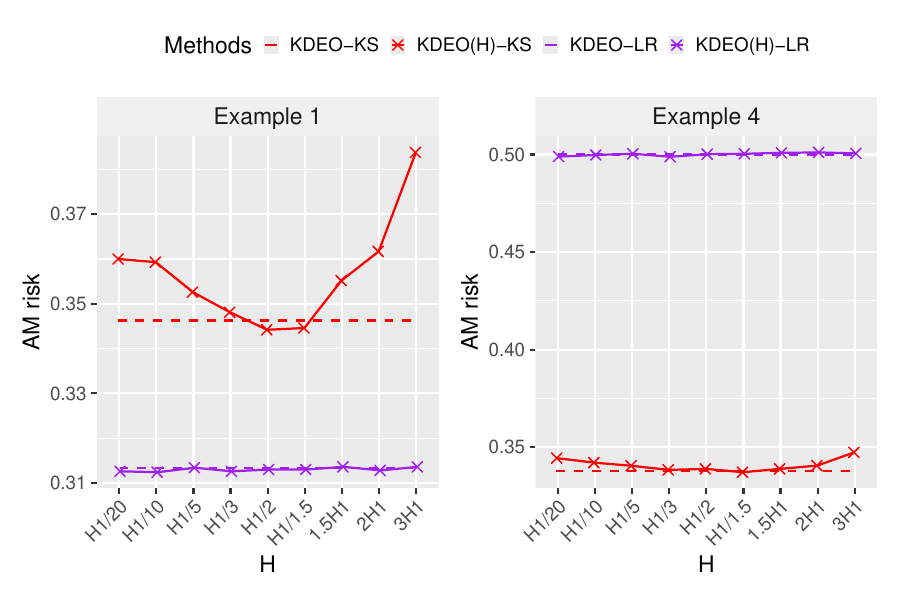}  
	\end{subfigure}
	\caption{ Average AM-risk of KNN, KS, and LR classifiers on balanced data over 50 replications. \textit{Left:} using \textsc{Smote} and \textsc{Smote($k$)} with $k \in (7, 65)$. 
\textit{Right:} using \textsc{Kdeo} and  \textsc{Kdeo}\((H)\), with \( H=cH_1 \) and $c$ ranging in \((1/20, 3)\) where $H_1$ follows from Scott’s rule.
 }
	\label{fig:bandwidth-check1-Sup}
\end{figure}

\begin{figure}[t]
	\begin{subfigure}{.5\textwidth}
		\centering
\includegraphics[width=1\linewidth, height=0.25\textheight]{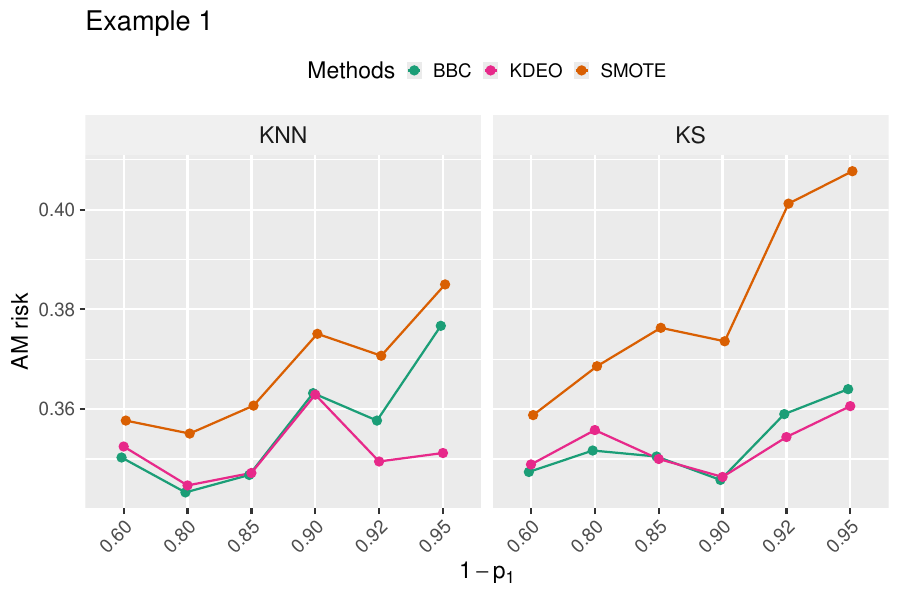}  
	\end{subfigure}
    \begin{subfigure}{.5\textwidth}
		\centering
\includegraphics[width=1\linewidth, height=0.25\textheight]{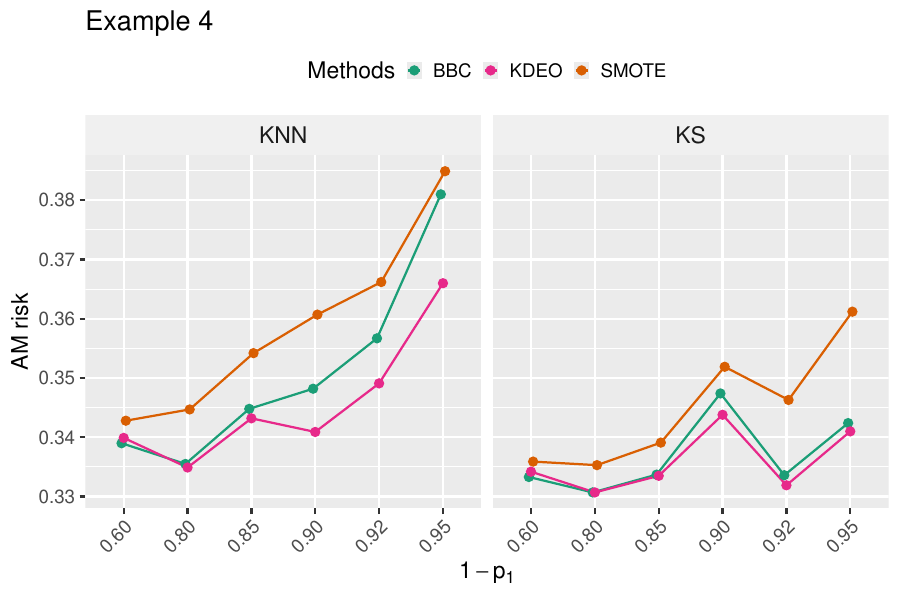}  
	\end{subfigure}
     \begin{subfigure}{.5\textwidth}
		\centering
\includegraphics[width=1\linewidth, height=0.25\textheight]{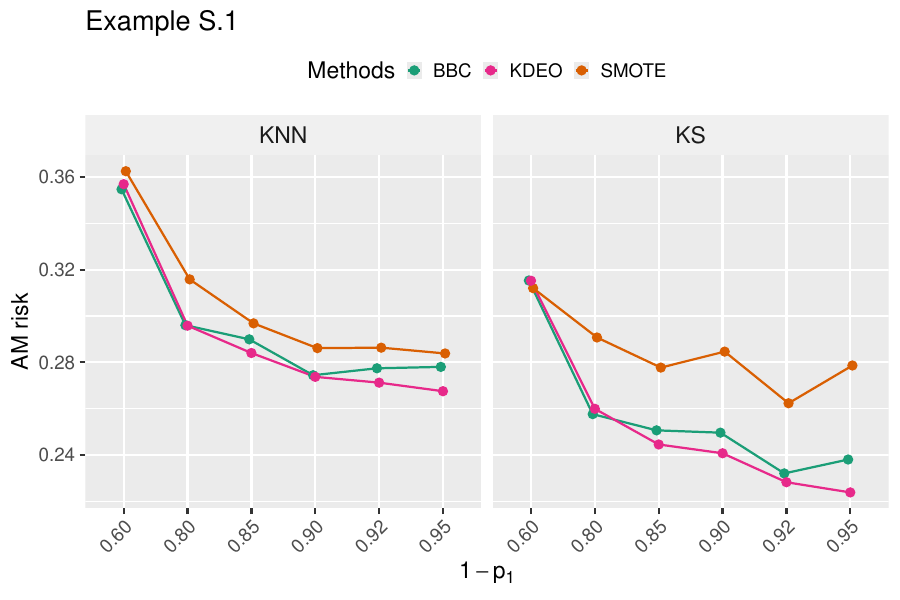}  
	\end{subfigure}

	\caption{Average AM-risk across different data imbalance regimes for the KS~(described in Section~\ref{sec:main:excessrisk}) and $K$NN classification rules computed over 50 replications.
    }
	\label{fig:sim-models-suppp1}
\end{figure}

\begin{figure}[t]
 	\begin{subfigure}{.5\textwidth}
 		\centering
 \includegraphics[width=1\linewidth, height=0.30\textheight]{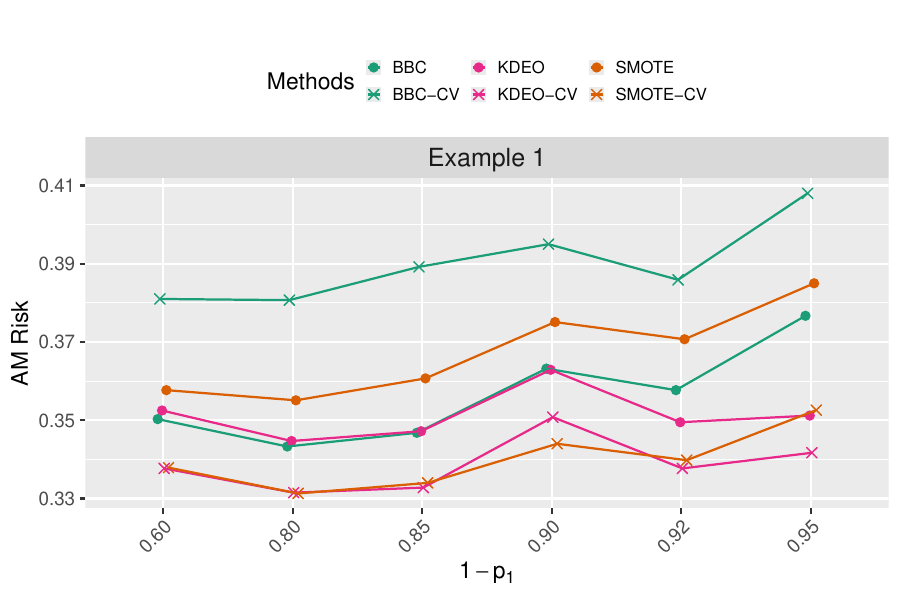}  
	\end{subfigure}
	 \begin{subfigure}{.5\textwidth}
	 	\centering
	 	\includegraphics[width=1\linewidth, height=0.30\textheight]{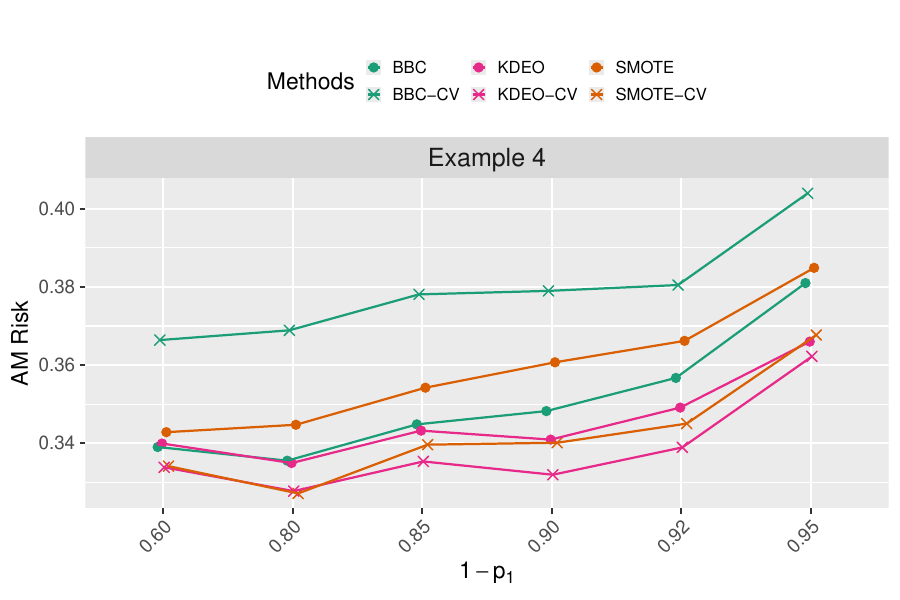}  
 \end{subfigure}

  \begin{subfigure}{.5\textwidth}
	 	\centering
	 	\includegraphics[width=1\linewidth, height=0.30\textheight]{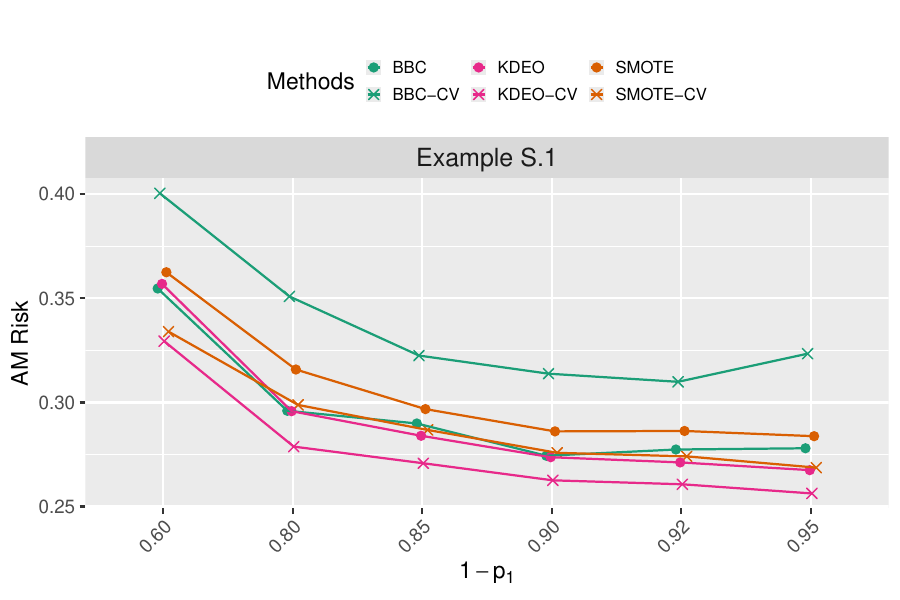}  
 \end{subfigure}
	\caption{Average AM-risk across different data imbalance regimes for the $K$NN methods described in Section~\ref{Methods+variants}, computed over 50 replications.
    }
	\label{fig:sim-models-supp}
\end{figure}




KDE-based oversampling (\textsc{Kdeo}) is another option to tackle the imbalanced data in a similar spirit to \textsc{Smote}. 
In contrast with \textsc{Smote}, the \textsc{Kdeo} algorithm relies on a kernel function $K$ and a bandwidth matrix $H$ that governs the smoothness of the estimated density function
and the tradeoff between bias and variance. Since the parameter $H$ has a significant impact on the accuracy of KDE, 
several methods have been developed to obtain an appropriate choice of $H$. 
Algorithm~\ref{alg:KDE_SMOTE} explains the steps \textsc{Kdeo} follows for oversampling in compliance with the bandwidth selected through the multivariate version of Scott’s rule of thumb \citep{scott2015multivariate}. 

\subsection{Simulated data}

We provide results linked to the analysis of Examples 1 and 4 in the main paper. Before that, we also introduce a further example, where again, the $\{(X_i, Y_i)\}_{1\leq i\leq n}$ are $n=1000$ i.i.d.~samples from the distribution of $(X, Y)$ and $\boldsymbol{e}_i$ is the $i$th vector in the canonical basis of $\mathbb{R}^d$.

\textbf{Example S.1:} Let $Z = \mathcal{B} Y_1 + (1 - \mathcal{B}) Y_2$, with $\mathcal{B} \sim \text{Bernoulli}(0.5)$, $Y_1$ is an extended generalized Pareto random variable, i.e.~$Y_1 \sim \text{EGPD}(\kappa(X), \sigma(X), 0.5)$ and $Y_2 \sim \text{Exp}(10\gamma(X))$, where $\kappa(X) = \exp(X^\top \boldsymbol{e}_1)$, $\sigma(X) = \exp(X^\top \boldsymbol{e}_2)$ and $\gamma(X) = \exp(X^\top \boldsymbol{e}_3)$. Define $Y = \mathrm{1}_{ \{ Z>t \} }$ with $t$ tuning class imbalance.

\begin{figure}[ht]
	\begin{subfigure}{.5\textwidth}
		\centering
\includegraphics[width=1\linewidth, height=0.20\textheight]{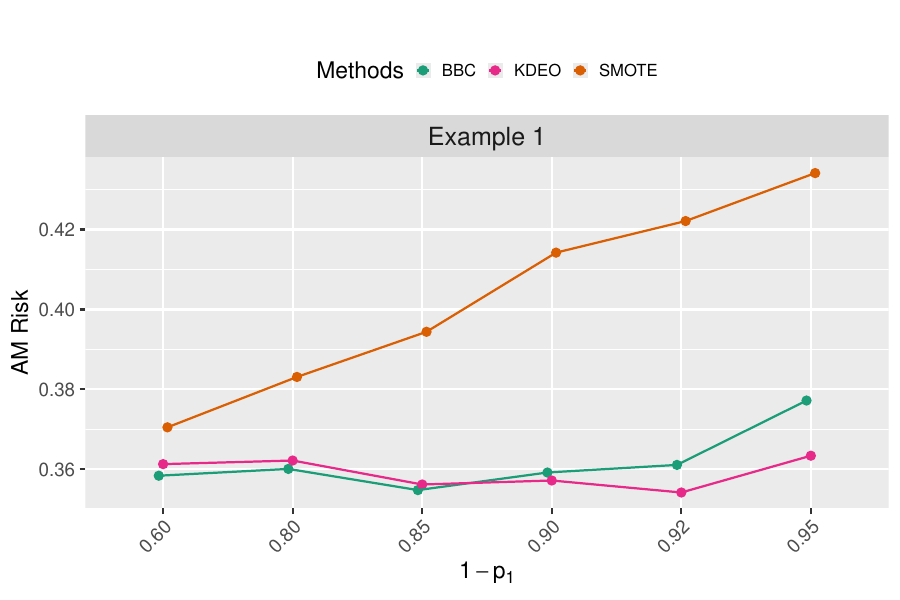}  
	\end{subfigure}
	\begin{subfigure}{.5\textwidth}
		\centering
		\includegraphics[width=1\linewidth, height=0.20\textheight]{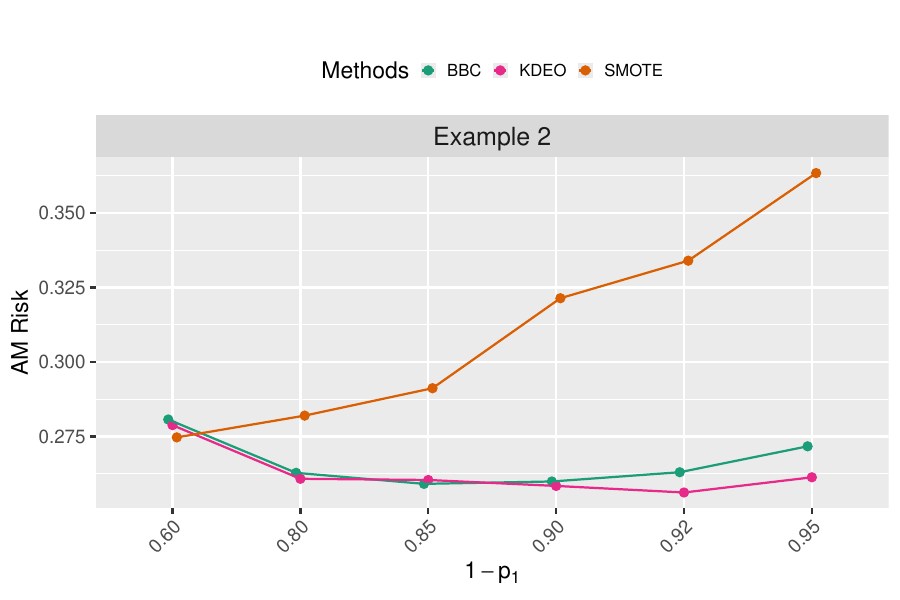}  
	\end{subfigure}
        \newline
       \begin{subfigure}{.5\textwidth}
		\centering
\includegraphics[width=1\linewidth, height=0.20\textheight]{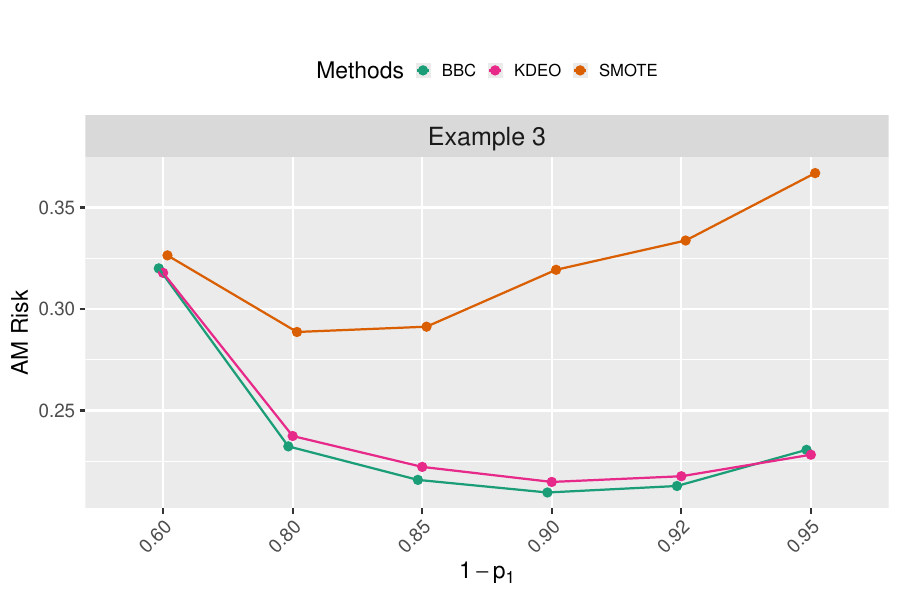}  
	\end{subfigure}
	\begin{subfigure}{.5\textwidth}
		\centering
		\includegraphics[width=1\linewidth, height=0.20\textheight]{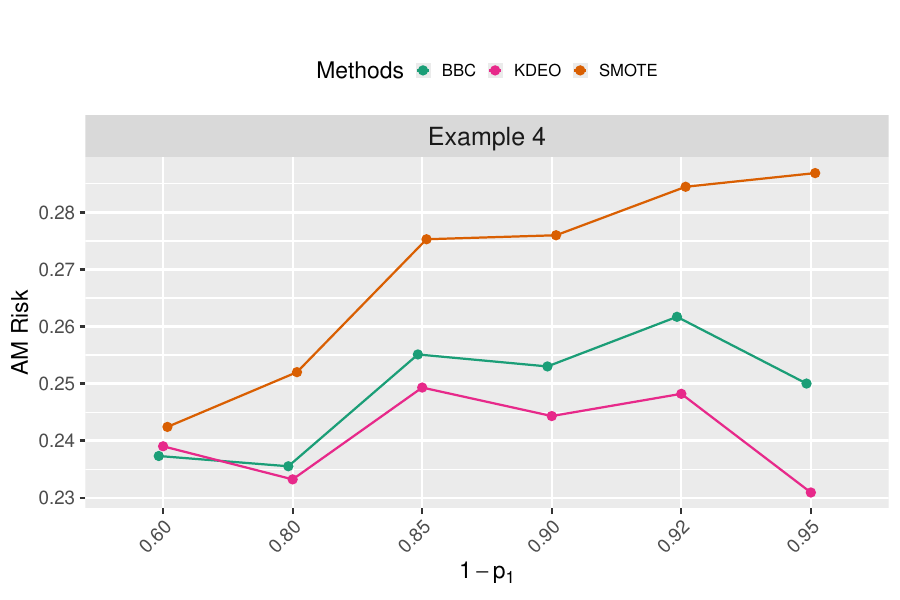}  
	\end{subfigure}
    \begin{subfigure}{.5\textwidth}
		\centering
\includegraphics[width=1\linewidth, height=0.20\textheight]{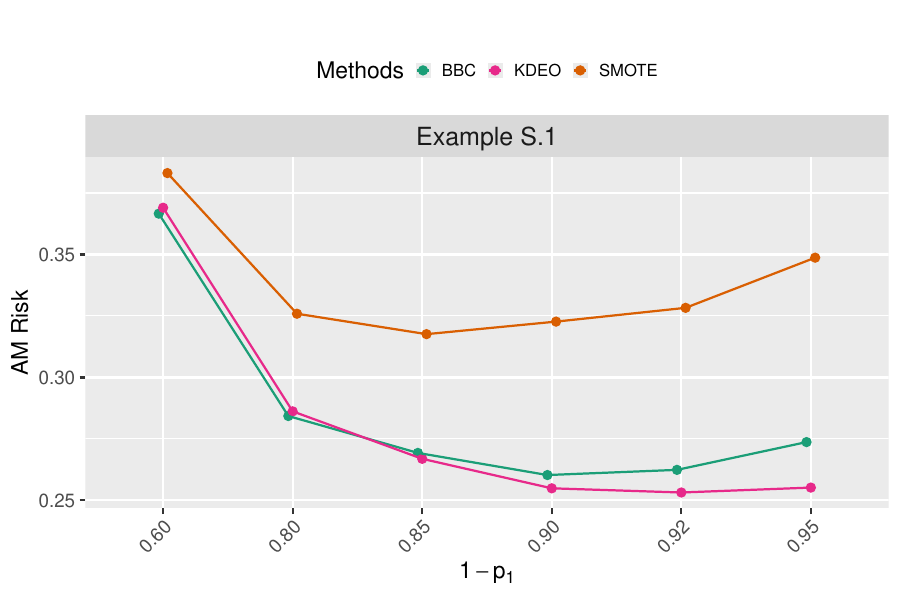}  
	\end{subfigure}
	\caption{Average AM-risk across different data imbalance regimes for the Random Forest classifier computed over 50 replications. }
	\label{fig:sim-models-RF}
\end{figure}

In Example S.1, the parameter \( t \) is tuned to achieve various probability levels (\( 1-p_1 = 0.60, 0.80, 0.85, 0.90, 0.92, 0.95 \)). The considered methods are the ones introduced in Sections~\ref{sec:main:excessrisk} and~\ref{Methods+variants} of the main document. 

The AM-risk performances of the KS and $K$NN classifiers for Examples 1, 4, and S.1 are reported in Figure~\ref{fig:sim-models-suppp1}, where both methods exhibit identical patterns and achieve superior performance with \textsc{Kdeo} and BBC compared to \textsc{Smote}. Moreover, the $K$NN classifier can further be improved by tuning the hyperparameter $K$ via CV, as shown in Figure~\ref{fig:sim-models-supp}; 
A noticeable performance improvement can be observed in the \textsc{Kdeo}-CV, \textsc{Smote}-CV results, but 
CV on $K$ yields poorer performance for the BBC. 
Moreover, \textsc{Smote-CV}, \textsc{Kdeo} and BBC exhibit very similar performance, indicating that these methods enhance minority class representation comparably when used with $K$NN, which is consistent with the theoretical justifications presented in the main document. 

We also consider other classifiers, i.e.~Random Forest and Logistic Regression (LR), compared to the 
the main document. 
Figure~\ref{fig:sim-models-RF} shows the AM-risk results for the Random Forest classifier. The \textsc{Kdeo} and BBC perform better than \textsc{Smote} in terms of AM-risk. Further notice that \textsc{Kdeo} still performs better than BBC even when the probability of the minority class $p_1$ is small. In Example 4, all methods perform similarly. 
We did not use CV to select the number of features at each split, as the low dimensionality makes CV not beneficial; 
without CV, BBC performs comparably to the oversampling methods across all examples. 

In contrast, the LR classifier performs uniformly across all resampling methods (see Figure~\ref{fig:sim-models-logis}). Notice that the LR classifier performs worse when applied to more complex classification structures. This overall uniformity in results is in line with the results given in Theorems \ref{theorem-smote-concentration_uniform} and \ref{th:smote_expectation} as they suggest that estimating the risk with \textsc{Smote} or \textsc{Kdeo} gives better results when $K$ and $H$ are small, underlining that oversampling may not be critical for such type of (parametric) classifier.

\begin{figure}[ht]
	\begin{subfigure}{.5\textwidth}
		\centering
\includegraphics[width=1\linewidth, height=0.20\textheight]{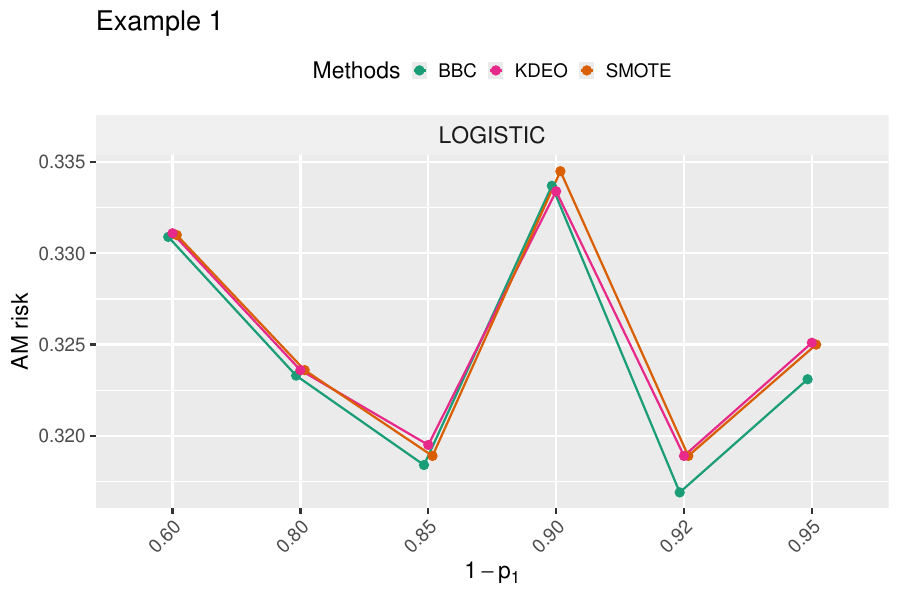}  
	\end{subfigure}
	\begin{subfigure}{.5\textwidth}
		\centering
		\includegraphics[width=1\linewidth, height=0.20\textheight]{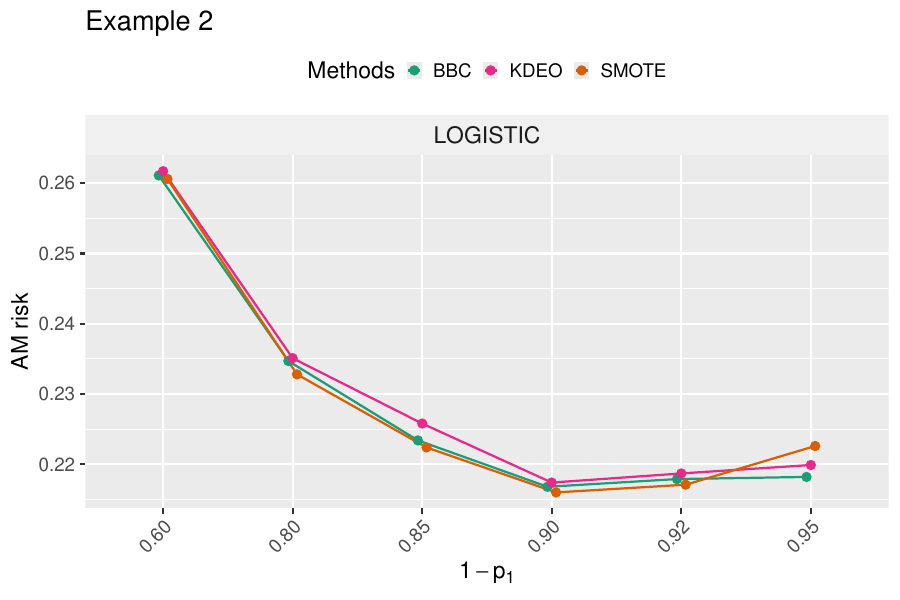}  
	\end{subfigure}
        \newline
       \begin{subfigure}{.5\textwidth}
		\centering
\includegraphics[width=1\linewidth, height=0.20\textheight]{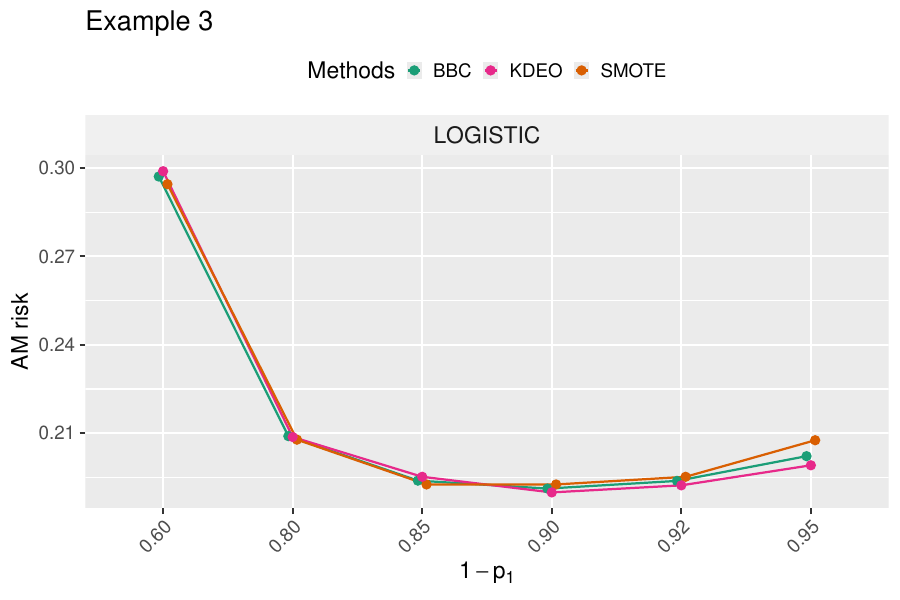}  
	\end{subfigure}
	\begin{subfigure}{.5\textwidth}
		\centering
		\includegraphics[width=1\linewidth, height=0.20\textheight]{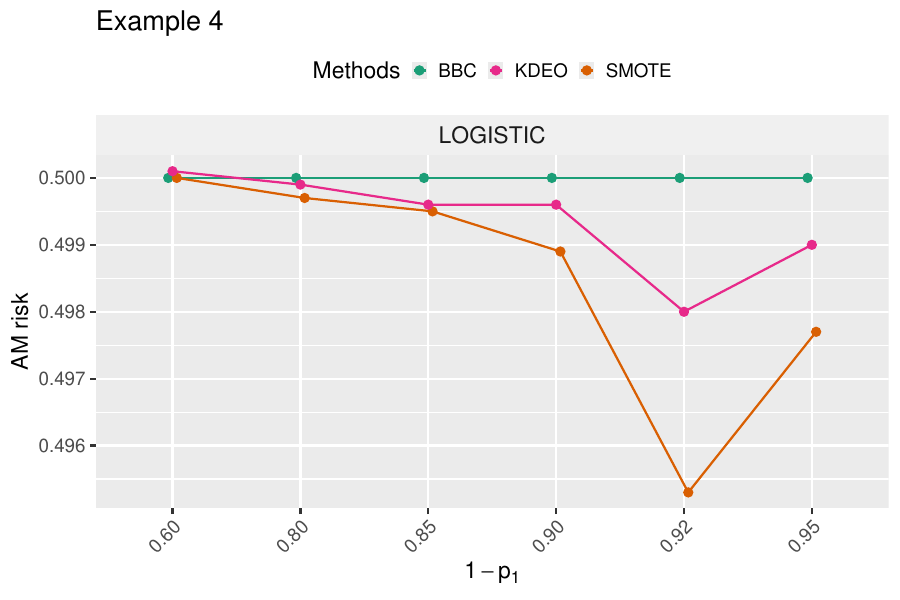}  
	\end{subfigure}
    \begin{subfigure}{.5\textwidth}
		\centering
\includegraphics[width=1\linewidth, height=0.20\textheight]{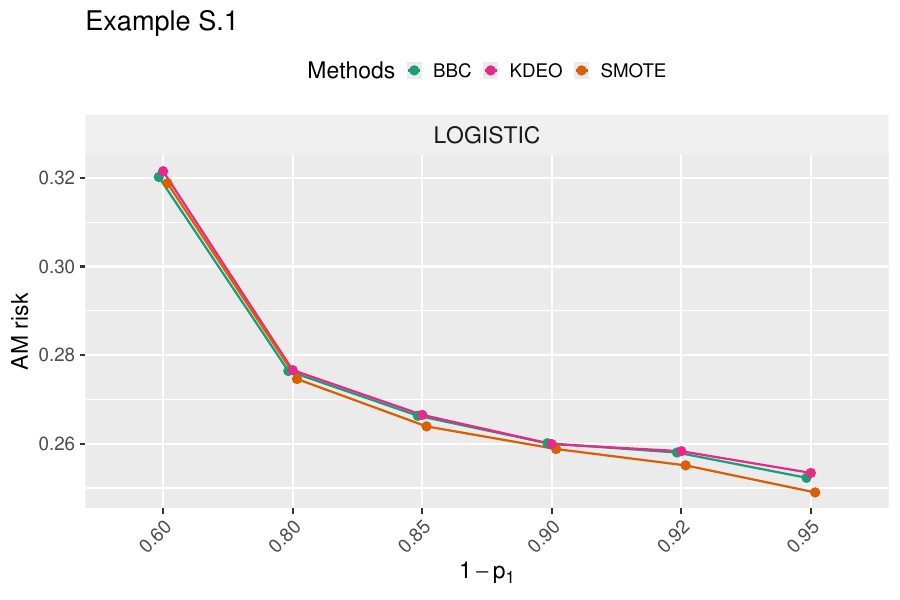}  
	\end{subfigure}
	\caption{Average AM-risk across different data imbalance regimes for the Logistic Regression classifier computed over 50 replications. }
	\label{fig:sim-models-logis}
\end{figure}

\subsection{Real data analysis: Abalone, California, MagicTel, Phoneme, and House\_16H datasets}

We apply the Random Forest, KS, LR, and LR-Lasso classifiers to the real datasets considered in the main paper. 
The AM-risk results for the KS are given in Figure~\ref{fig:realdata-sup-KSM}, those for Random Forest in Figure~\ref{fig:realdata-sup-RF}, while results for the LR and LR-Lasso classifiers are presented in Figure~\ref{fig:realdata-sup-logis}. The KS classifier shows comparable performance to \textsc{Smote} and \textsc{Kdeo} across all datasets. For instance, in the MagicTel dataset, \textsc{Kdeo} performs slightly better, while California and Abalone \textsc{Smote} stand superior. 
For the Random Forest classifier, \textsc{Smote} and \textsc{Kdeo}-based oversampling exhibit comparable performance across all datasets. Similarly, both oversampling methods yield nearly identical results in the LR and LR-Lasso classifiers, except for the Abalone and House\_16H datasets. A key finding is that using cross-validation to select the number of features at each split in Random Forest and the regularization parameter $\lambda$ in LR-Lasso consistently enhances performance across all scenarios, with the exception of Abalone and House\_16H. The low dimensionality of the Abalone and House\_16H datasets may explain the lack of performance improvement when CV is used for parameter tuning. 

\begin{figure}[t]
		\centering
\includegraphics[width=1\linewidth, height=0.4\textheight]{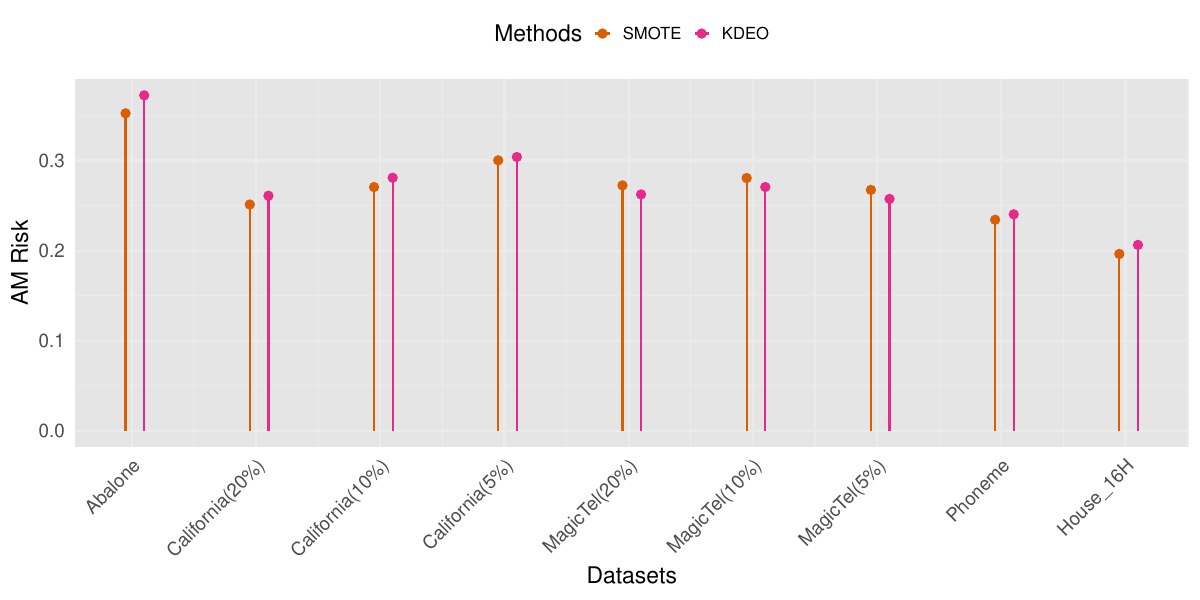}
	\caption{AM-risk corresponding to different rebalancing methods and datasets when using the KS classifier. 
    }
	\label{fig:realdata-sup-KSM}
\end{figure}

\begin{figure}[t]
		\centering
\includegraphics[width=1\linewidth, height=0.4\textheight]{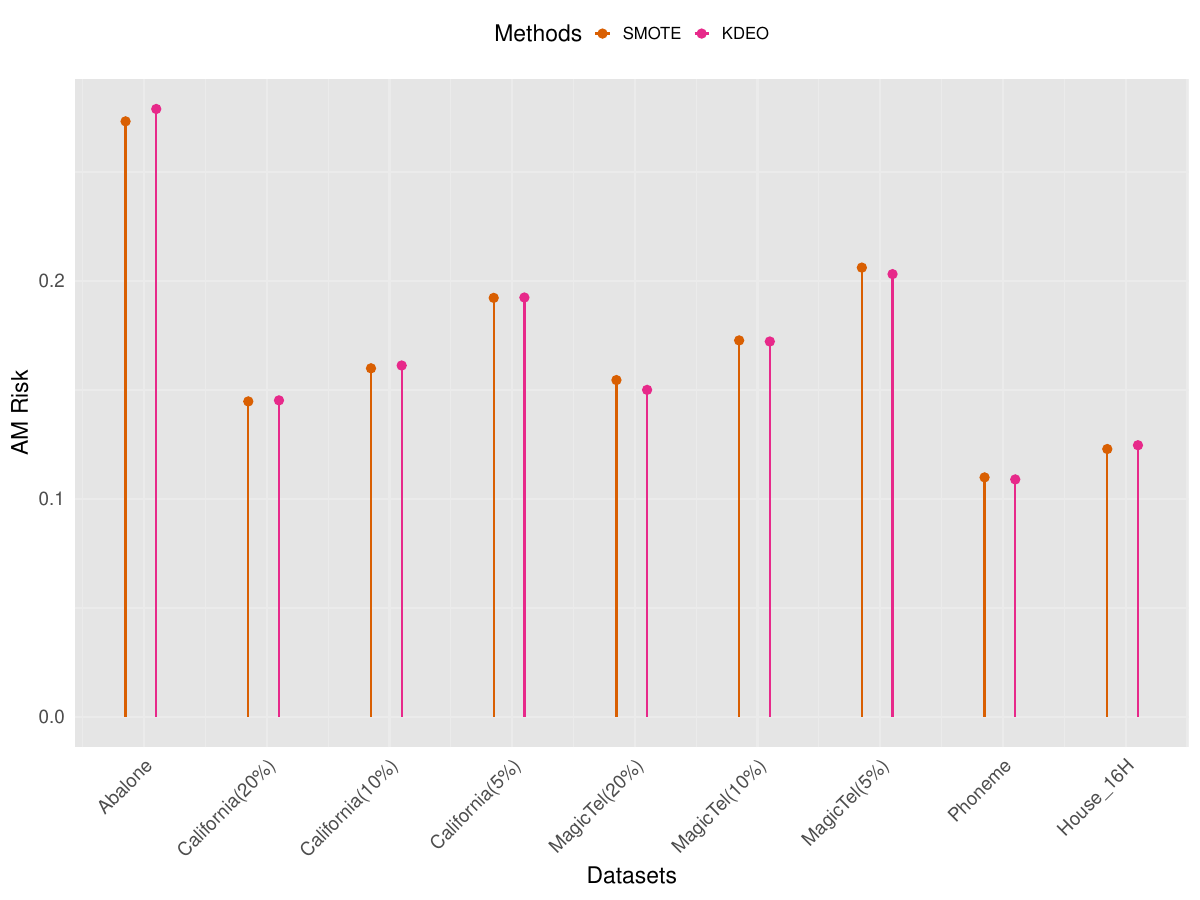}
	\caption{AM-risk corresponding to different rebalancing methods and datasets when using the Random Forest classifier.
    }
	\label{fig:realdata-sup-RF}
\end{figure}

\begin{figure}[t]
		\centering
\includegraphics[width=1\linewidth, height=0.4\textheight]{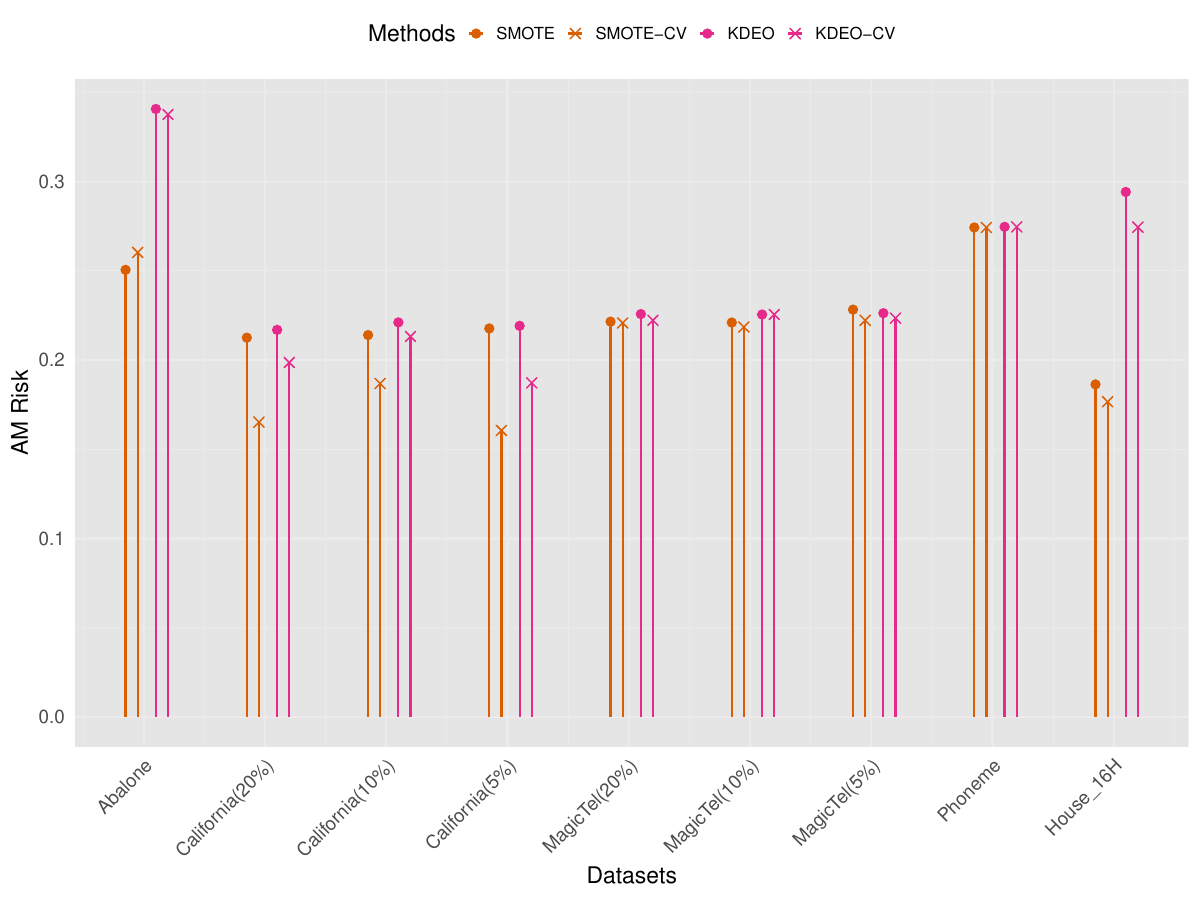}
	\caption{AM-risk corresponding to different rebalancing methods and datasets when using the LR and LR-Lasso classifier. Methods containing the ``CV'' characters stand for the LR-Lasso classifier with a regularization parameter tuned through CV.}
	\label{fig:realdata-sup-logis}
\end{figure}

\end{document}